\documentclass{article}

\usepackage{arXiv}

\usepackage[utf8]{inputenc} % allow utf-8 input
\usepackage[T1]{fontenc}    % use 8-bit T1 fonts
\usepackage{hyperref}       % hyperlinks
\usepackage{url}            % simple URL typesetting
\usepackage{booktabs}       % professional-quality tables
\usepackage{amsfonts}       % blackboard math symbols
\usepackage{nicefrac}       % compact symbols for 1/2, etc.
\usepackage{microtype}      % microtypography
\usepackage{xcolor}
\usepackage{amsthm}
\usepackage{graphicx}
\usepackage{subfigure}
\usepackage{wrapfig}

\usepackage{amsfonts,amssymb,graphics,epsfig,verbatim,bm,bbm,latexsym,amsmath,url,amsbsy,multicol,multirow}

\newtheorem{theorem}{Theorem}
\newtheorem{assumption}{Assumption}
\newtheorem{proposition}{Proposition}

\newtheorem{remark}{Remark}

\newcommand{\iid}{ \stackrel{i.i.d.}{\sim}}
\def\Holder{{H\"{o}lder}}
\def\Cramer{Cram\'{e}r}

\DeclareMathOperator*{\argmin}{arg\,min}

\def\Holder{{H\"{o}lder}}

\def\Cramer{Cram\'{e}r}

\newcommand{\op}{\mathrm{o}_{p}}

\newcommand{\pa}{\mathrm{\pa}}

\newcommand{\epol}{\pi^\mathrm{e}}

\usepackage{natbib}
\usepackage{braket}

\usepackage{authblk}

\begin{document}

\title{The Adaptive Doubly Robust Estimator\\
for Policy Evaluation in Adaptive Experiments\\
and a Paradox Concerning Logging Policy}

\author[1]{Masahiro Kato\thanks{Corresponding author: masahiro\_kato@cyberagent.co.jp}\ \ }
\author[1]{Shota Yasui}
\author[2]{Kenichiro McAlinn}

\affil[1]{CyberAgent Inc.}
\affil[2]{Temple University}

\maketitle

\begin{abstract}
The \emph{doubly robust} (DR) estimator, which consists of two nuisance parameters, the conditional mean outcome and the logging policy (the probability of choosing an action), is crucial in causal inference. This paper proposes a DR estimator for \emph{dependent samples} obtained from adaptive experiments. To obtain an asymptotically normal semiparametric estimator from dependent samples with non-Donsker nuisance estimators, we propose \emph{adaptive-fitting} as a variant of sample-splitting. We also report an empirical \emph{paradox} that our proposed DR estimator tends to show better performances compared to other estimators utilizing the true logging policy. While a similar phenomenon is known for estimators with i.i.d. samples, traditional explanations based on asymptotic efficiency cannot elucidate our case with dependent samples. We confirm this hypothesis through simulation studies\footnote{A code of the conducted experiments is available at
\url{https://github.com/MasaKat0/adr}.}.
\end{abstract}

% A doubly robust (DR) estimator is crucial in causal inference, which consists of two nuisance parameters: the conditional mean outcome and logging policy (probability of choosing an action). This paper provides a DR estimator for dependent samples obtained in adaptive experiments with introducing two related topics: double/debiased machine learning (DML) for dependent samples and empirical paradox on logging policy. There are two problems for $\sqrt{T}$ consistent DR estimator in our setting, where where $T$ is the sample size. First, samples are dependent because an adaptive experiment allows the logging policy to be sequentially updated based on past observations. Second, we need to control complexities of the nuisance estimators. In theoretical analysis, we propose adaptive-fitting as a variant of DML for solving these two problems. Then, we show that the proposed DR estimator constructed via adaptive-fitting is asymptotically normal under some time-series assumptions and standard convergence rate conditions on the nuisance estimators. In empirical analysis, we point out that the DR estimator shows better performances than other estimators using the true logging policy instead of its estimated estimator. While a similar phenomenon is also reported for i.i.d. samples, we hypothesize that traditional explanations based on asymptotic efficiency are not applicable to a case with dependent samples. We confirm our hypothesizes on this phenomenon through simulation studies.

\section{Introduction}
\emph{Adaptive experiments}, including efficient treatment effect estimation \citep{Laan2008TheCA,Hahn2011}, treatment regimes \citep{Zhang2012,zhao2012est_ind}, and multi-armed bandit (MAB) problems \citep{Gittins89,lattimore2020bandit}, are widely accepted and used in real-world applications, for example, in social experiments \citep{Hahn2011}, online advertisement \citep{zhang2012jointoptimization}, clinical trials \citep{ChowChang201112,Villar2018}, website optimization \citep{white2012bandit}, and recommendation systems \citep{Lihong2010contextual}. In those various applications, there is a significant interest in off-line counterfactual inference using samples obtained from past trials generated by a logging policy  \citep[the probability of choosing an action:][]{ Lihong2010contextual,Li2011}. Among several off-line evaluation criteria, this paper focuses on \emph{off-policy value estimation}  \citep[OPVE:][]{Li2015}. The goal of OPVE is to estimate the expected value of the weighted outcome, which includes average treatment effect (ATE) estimation as a special case \citep{hirano2003,bang2005drestimation}. 

In OPVE, when the true logging policy is known, there are mainly three types of estimators: inverse probability weighting \citep[IPW:][]{Horvitz1952}, direct method (DM), and augmented IPW \citep[AIPW][]{bang2005drestimation} estimators. IPW-type estimators refer to a sample average of the outcomes weighted by the true logging policy. DM-type estimators refer to a sample average of an estimated conditional outcome. AIPW-type estimators consist of two components: the IPW part and the DM part. More details are given in Section~\ref{sec:prelim}. In addition, when the true logging policy is unknown, we call IPW-type and AIPW-type estimators that use the estimated logging policy EIPW-type and DR-type estimators, respectively. In particular, and the focus of this paper, DR-type estimators are crucial in OPVE, since in most applications of interest the true logging policy is either unknown, difficult, or costly to obtain, making IPW and AIPW-type estimators often infeasible. %For instance, the DR estimator is feasible even if we do not know the true logging policy, in contrast to  IPW and AIPW-type estimators, which rely on such knowledge.
%not available in practice. 

While existing OPVE methods typically assume that the samples are \emph{independent and identically distributed}  \citep[i.i.d.:][]{hirano2003,dudik2011doubly}, adaptive experiments usually allow logging policies to be updated based on past observations, under which the generated samples are non-i.i.d. In this case, theoretical results shown under i.i.d. assumptions, such as consistency and asymptotic normality, are not guaranteed. In particular, the asymptotic normality is critical, as we can obtain $\sqrt{T}$-rate confidence intervals, where $T$ is the sample size. For this problem, \citet{Laan2008TheCA} proposed an asymptotically normal AIPW estimator under dependent samples, with several follow-up studies, including \citet{hadad2019}, \citet{Kato2020adaptive}, and \citet{kelly2020batched}.
%extended the AIPW estimator by adding the adding evaluation weights to mitigate the influence of the time-varying logging policy. %\citet{Kato2020adaptive} derived a non-asymptotic confidence interval of the AIPW estimator. 
However, a DR-type estimator for dependent samples has never been formally proposed, despite its importance. 

%Several studies, such as \citet{Laan2008TheCA}, \citet{hadad2019}, \citet{Kato2020adaptive}, and  \citet{kelly2020batched}, tackled this problem mainly from using the martingale properteies obtained from the true logging policy.
%Several approaches have been proposed, and we categorize them into three categories: the first is to use the martingale central limit theorem (CLT) under the assumption that the logging policy converges to a time-invariant function in probability \citep{Laan2008TheCA,Laan2016onlinetml,hadad2019,Kato2020adaptive}, the second is to assume batch update, where sample sizes under fixed policies are sufficiently large \citep{Hahn2011,Laan2016onlinetml,kelly2020batched}, and the third is to standardize the score functions for each period \citep{Luedtke2016}. 

%Existing studies extend these methods for non-i.i.d. samples obtained in adaptive experiments. \citet{Laan2008TheCA} proposed an asymptotically normal AIPW estimator for dependent samples. \citet{hadad2019} proposed adding adaptive weights to the AIPW estimator of \citet{Laan2008TheCA} to mitigate the influence of the time-varying logging policy. \citet{Kato2020adaptive} derived a non-asymptotic confidence interval of the AIPW estimator. However, a DR-type estimator for dependent samples has rarely been discussed despite its importance. 

This paper proposes an asymptotically normal DR estimator for dependent samples. We call this DR estimator the \emph{adaptive DR} (ADR) estimator. There are mainly two difficulties concerning this estimator. First, we cannot use the CLT for i.i.d. or martingale difference sequences, as done in \citet{Laan2008TheCA}, \citet{hadad2019}, and \citet{Kato2020adaptive}. Second, when nuisance parameters are estimated with complicated models, the Donsker condition does not hold in general, which is required to show the asymptotic normality of semiparametric estimators, such as our proposed ADR estimator. In this paper, to solve these problems, we propose adaptive-fitting, inspired by the double/debiased machine learning for i.i.d. samples \citep{ChernozhukovVictor2018Dmlf}. We also find that by using the ADR estimator, not only can OPVE be done without knowing the true logging policy, the ADR estimator paradoxically often outperforms the performance of AIPW estimators that use the true logging policy. We list our contributions as follows. 

{\bf (I) Asymptotic normality of the ADR estimator.}
We show the asymptotic normality of our proposed ADR estimator. To the best of our knowledge, a DR-type estimator for dependent samples obtained in adaptive experiments has not been formally proposed.

{\bf (II) Adaptive fitting.}
While the asymptotic normalities of IPW and AIPW estimators are shown by the martingale CLT, we cannot directly apply this technique to our proposed ADR estimator. This is because the martingale condition does not hold when the true logging policy is unknown. To solve this problem, we utilize a novel sample-splitting method called adaptive-fitting to show the asymptotic normality of the ADR estimator. We generalize this technique as a variant of sample-splitting \citep{ChernozhukovVictor2018Dmlf} for semiparametric inference with dependent samples without the Donsker condition.  Unlike existing sample-splitting methods, which assume that samples are i.i.d., the proposed adaptive-fitting is applicable to non-i.i.d. samples. Note that this adaptive-fitting is different from martingale based estimators, such as \citet{Laan2008TheCA} and \citet{hadad2019}, which requires the true logging policy.
%Although methods similar to adaptive-fitting have been used in the past \citep{Laan2008TheCA}, the contribution of this paper is to summarize it as a variant of double/debiased machine learning \citep{ChernozhukovVictor2018Dmlf}. Note that adaptive-fitting is different to martingale based estimator, such as \citet{Laan2008TheCA} and \citet{hadad2019}.

{\bf  (III) Empirical paradox on using estimated logging policy.}
Through experimental studies, we investigate the empirical performance of our proposed ADR estimator. We find that the ADR estimator often exhibits improved performances over other estimators using the true logging policy, such as the AIPW estimator. Estimators requiring the true logging policy tend to empirically be unstable due to the instability of the nuisance parameter; the instability being caused by the logging policy being near zero before convergence. Our finding implies that the ADR estimator mitigates this instability. We call this phenomenon a \emph{paradox concerning logging policy} because using more information (the true logging policy) does not improve empirical performance. A similar paradox is known for IPW-type estimators for i.i.d. samples \citep{hahn1998role,Henmi2004paradox}. However, we cannot explain our paradox from a conventional semiparametric efficiency perspective \citep{hahn1998role,hirano2003}, as the asymptotic variances are the same between the ADR estimator (with an estimated logging policy) and the AIPW estimator (with the true logging policy), unlike IPW-type estimators. 

%To support this argument, we conduct another experiment to compare the performance of the ADR and APIW for i.i.d samples where those estimators have the same asymptotic variance. As a result, we observe the similar performance.
%To support this argument, we confirm that, while ADR and AIPW estimators show similar performances for i.i.d. samples, the ADR estimator is more stable for dependent samples than the AIPW estimator in our experiments. In both cases, the asymptotic variances are the same. 
%There can be several explanations for this paradox, and we consider this paradox as an open problem. 

{\bf Organization of this paper.} In Section~\ref{sec:prb}, we formulate our problem. In Section~\ref{sec:prelim}, we introduce existing OPVE estimators for dependent samples. In Section~\ref{sec:main}, we propose the ADR estimator, which is our main contribution. We call the method used for the ADR estimator adaptive-fitting, and generalize it in Section~\ref{sec:relation_ddm}. In Section~\ref{sec:paradox}, we report a paradox through simulation studies and explain the cause of this phenomenon. In Section~\ref{sec:exp}, we conduct experiments using benchmark datasets.

\vspace{-0.1cm}
\section{Problem setting}
\label{sec:prb}
%In this section, we define our problem with the parameter we want to estimate.
Suppose that there is a time series $1,2,\dots, T$, and denote the set as $[T] = \{1,\dots,T\}$. For $t\in[T]$, let $A_t$ be an \emph{action} in $\mathcal{A}=\{1,2,\dots,K\}$, $X_t$ be a \emph{covariate} observed by a decision maker when choosing an action, and $\mathcal{X}$ be its space. Following the Neyman–Rubin causal model \citep{rubin1974}, let a reward at period $t$ be $Y_t=\sum^K_{a=1}\mathbbm{1}[A_t = a]Y_t(a)$, where $Y_t(a):\mathcal{A}\to\mathbb{R}$ is a potential (random) outcome. We have a dataset $\big\{(X_t, A_t, Y_t)\big\}^{T}_{t=1}$ with the following data-generating process (DGP):
\begin{align}
\label{eq:DGP}
(X_t, A_t, Y_t(A_t))\sim p(x)p_t(a|x)p(y_{a}|x),
\end{align}
where $p(x)$ denotes the density of the covariate $X_t$, $p_t(a| x)$ denotes the probability of choosing an action $a$ conditional on a covariate $x$ at period $t$, and $p(y_a| x)$ denotes the density of a reward $Y_t(a)$ conditional on a covariate $x$. We assume that $p(x)$ and $p(y_a| x)$ are invariant across periods; that is, $\{(X_t, Y_t(1),\dots, Y_t(K))\}^T_{t=1}$ is i.i.d., but $p_t(a| x)$ can take different values across periods based on past observations. In this case, the samples $\big\{(X_t, A_t, Y_t)\big\}^{T}_{t=1}$ are correlated over time, that is, the samples are not i.i.d. Let $\Omega_{t-1}=\{X_{t-1}, A_{t-1}, Y_{t-1}, \dots, X_{1}, A_1, Y_{1}\}$ be the history and $\mathcal{M}_{t-1}$ be a set of possible histories until the $t$-th period. The probability $p_t(a| x)$ is determined by a \emph{logging policy} $\pi_t:\mathcal{A}\times\mathcal{X}\times\mathcal{M}_{t-1}\to(0,1)$, such that $\sum^K_{a=1}\pi_t(a| x, \Omega_{t-1}) = 1$, which is a function of a covariate $X_t$, an action $A_t$, and history $\Omega_{t-1}$. We also assume that $\pi_t$ is conditionally independent of $Y_t(a)$ to satisfy unconfoundedness (Remark~\ref{rem:unconfoundedness}).

\begin{remark}[Stable unit treatment value assumption (SUTVA)]
\label{rem:sutva}
The DGP implies SUTVA; that is, $p(y_a| x)$ is invariant for any $p(a| x) $\citep{Rubi:86}.
\end{remark}

\begin{remark}[Unconfoundedness]
\label{rem:unconfoundedness}
In this paper, unconfoundedness refers to independence between $(Y_t(1),\dots, Y_t(K))$ and $A_t$, conditional on $X_t$ and $\Omega_{t-1}$, which is required for identification. 
\end{remark}

\subsection{OPVE}
\label{sec:OPVE}
The goal of OPVE is to estimate the expected value of the sum of the outcomes $Y_t(a)$ weighted by an evaluation function $\epol:\mathcal{A}\times \mathcal{X} \to\mathbb{R}$; that is, 
\begin{align}
R(\epol) &:=\mathbb{E}_{(X_t, Y_t(1),\dots,Y_t(K))}\left[\sum^K_{a=1}\epol(a | X_t)Y_t(a)\right],
\end{align}
where the expectation $\mathbb{E}_{(X_t, Y_t(1),\dots, Y_t(K))}$ is taken over $(X_t, Y_t(1), \dots, Y_t(K))$. We denote it as $\mathbb{E}$, when there is no ambiguity. As with \citet{dudik2011doubly}, the evaluation function is usually referred to as an evaluation policy, where $\epol(a| x)\in [0,1]$ and the sum is $1$. However, to not restrict it so we can include other forms, such as the ATE, we refer to it differently. The ATE is also a special case of OPVE for $\mathcal{A}=\{1,2\}$, where $\epol(1| x) = 1$ and $\epol(2| x) = -1$. To identify $R(\epol)$, we assume the overlap in the distributions of policies, convergence of $\pi_{t-1}$, and the boundedness of reward. 
\begin{assumption}\label{asm:overlap_pol}
For all $t\in[T]$, $a\in\mathcal{A}$, $x\in\mathcal{X}$, and $\Omega_{t-1}\in\mathcal{M}_{t-1}$, there exists a constant $C_\pi > 0$, such that $\left| \frac{\epol(a| x)}{\pi_t(a| x, \Omega_{t-1})}\right| \leq C_\pi$.
\end{assumption}
Assumption~\ref{asm:overlap_pol} equivalently means that $\pi_{t}(a| x) > 0$ for all $a\in\mathcal{A}$ and $x\in\mathcal{X}$. 
\begin{assumption}
\label{asm:stationarity}
For all $x\in\mathcal{X}$ and $a\in\mathcal{A}$, $\pi_t(a| x, \Omega_{t-1})-\tilde{\pi}(a| x)\xrightarrow{\mathrm{p}}0$, where $\tilde{\pi}: \mathcal{A}\times\mathcal{X}\to(0,1)$.
\end{assumption}
\begin{assumption}\label{asm:bounded_reward}
For all $t\in[T]$ and $a\in\mathcal{A}$, there exists a constant $C_Y > 0$, such that $|Y_t(a)| \leq C_Y$.
\end{assumption}

Although the reader may feel that Assumption~\ref{asm:overlap_pol}--\ref{asm:stationarity} and the SUTVA \ref{rem:sutva} are strong, we adopt it as a simple and basic case for the application, in order to introduce adaptive-fitting and the ADR estimator. We can extend the proposed method for different cases, such as when the data has the structure of batches or when the average of logging policy converges \citep{Kato2021nonstationary}. Note that the convergence assumption (Assumption~\ref{asm:stationarity}) is also explicitly or implicitly required in other studies, such as \citet{Laan2008TheCA} and \citet{hadad2019}. 

\subsection{Notations}
We denote $\mathbb{E}[Y_t(a)| x]$ and $\mathrm{Var}(Y_t(a)| x)$ as $f^*(a, x)$ and $v^*(a, x)$, respectively. Let $\hat{f}_{t}(a, x)$ be an estimator of $f^*(a, x)$ constructed from $\Omega_{t}$. Let $\mathcal{N}(\mu, \mathrm{var})$ be the normal distribution with the mean $\mu$ and the variance $\mathrm{var}$. For a random variable $Z$ with density $p(z)$ and function $\mu$, let $\|\mu(Z)\|_2=\int |\mu(z)|^2 p(z) dz $ be the $L^{2}$-norm.

\section{Preliminaries of OPVE}
\label{sec:prelim}
%We review existing estimators with related theories.

\subsection{Existing estimators}
For estimating $R(\epol)$ from dependent samples, existing studies propose various estimators. An adaptive version of the IPW estimator is defined as $R^{\mathrm{AdaIPW}}_T(\epol)=\frac{1}{T}\sum^T_{t=1}\sum^K_{a = 1}\frac{\epol(A_t| X_t)\mathbbm{1}[A_t=a]Y_t }{\pi_{t-1}(A_t| X_t, \Omega_{t-1})}$ \citep{Laan2008TheCA}. If the model specification is correct, the direct method (DM) estimator $\frac{1}{T}\sum^T_{t=1}\sum^K_{a=1}\epol(a | x)\hat{f}_{T}(a| X_t)$ is known to be consistent to $R(\epol)$. As an adaptive version of the AIPW estimator, \citet{Laan2008TheCA} proposed an estimator $\widehat{R}^{\mathrm{AIPW}}_T(\epol)$ defined as 
\begin{align*}
\frac{1}{T}\sum^T_{t=1}\sum^K_{a=1}\Bigg\{\frac{\epol(a| X_t)\mathbbm{1}[A_t=a]\left(Y_t - \hat{f}_{t-1}(a, X_t)\right) }{\pi_{t-1}(a| X_t, \Omega_{t-1})}+ \epol(a| X_t)\hat{f}_{t-1}(a, X_t)\Bigg\}.
\end{align*}
Using the martingale property, \citet{Laan2008TheCA} showed asymptotic normality under Assumption~\ref{asm:stationarity}. \citet{hadad2019} and \citet{Kato2020adaptive} organized the results ( Proposition~\ref{prp:asymp_dist_a2ipw}).

\subsection{Asymptotic efficiency}
\label{sec:asymp_eff}
We are often interested in the asymptotic efficiency of estimators. The lower bound of the asymptotic variance is defined for an estimator under some posited models of the DGP \eqref{eq:DGP}. As with the \Cramer-Rao lower bound for  the parametric model, we can also define the lower bound for the non- or semiparametric model \citep{bickel98}. 
%parametric model, then the lower bound is equal to the \Cramer-Rao lower bound. When this posited model is a 
%non- or semiparametric model, as well as  the \Cramer-Rao lower bound for parametric models, we can define a lower bound, \citep{bickel98}. 
The semiparametric lower bound of the DGP~\eqref{eq:DGP} under $p_1(a| x)=\cdots=p_T(a| x)=p(a| x)$ is given as follows \citep{hahn1998role,narita2019counterfactual}:
\begin{align*}
\Psi(\epol, p) = \mathbb{E}\Bigg[\sum^{K}_{a=1}\frac{\big(\epol(a| X_t)\big)^2v^*(a, X_t)}{p(a| X_t)}+ \left(\sum^{K}_{a=1}\epol(a| X_t)f^*(a, X_t) - R(\epol)\right)^2\Bigg].
\end{align*}
The asymptotic variance of the asymptotic distribution is also known as the asymptotic mean squared error (MSE); that is, an OPVE estimator achieving the semiparametric lower bound also minimizes the MSE to the true value $R(\epol)$, not just obtaining a tight confidence interval.

\subsection{Related work} 
There are various studies related to OPVE, including ATE estimation, under the assumption that samples are i.i.d. \citep{hahn1998role,hirano2003,dudik2011doubly,wang2017optimal,narita2019counterfactual,Bibaut2019moreffficient,Oberst2019}. There are also several studies extending these methods to OPVE from dependent samples \citep{Laan2008TheCA,Laan2016onlinetml,Luedtke2016,hadad2019,Kato2020adaptive}.

The AIPW estimator for dependent samples are proposed by \citet{Laan2008TheCA}. \citet{hadad2019} proposed an evaluation weight to stabilize the estimator, which shares a similar motivation with weight clipping, or shrinkage, when i.i.d. samples are given \citep{Bembom2008,Bottou2013,Wang2017,Su2019,Su2020}. \citet{Kato2020adaptive} showed a non-asymptotic confidence interval of the AIPW estimator. \citet{Laan2016onlinetml} and \citet{Luedtke2016} proposed an OPVE method without the convergence of the logging policies by using batches and standardization, respectively. Note that \citet{hadad2019} cannot weaken the assumptions regarding the logging policy, unlike \citet{Luedtke2016}. \citet{kelly2020batched} proposed an estimator similar to \citet{Laan2016onlinetml} and applied it to linear regression. Estimators proposed by \citet{Laan2008TheCA}, \citet{Luedtke2016}, \citet{hadad2019}, \citet{Kato2020adaptive}, and \citet{kelly2020batched} require the true logging policy, unlike our ADR estimator.

A semiparametric estimator usually requires the Donsker condition for its asymptotic normality \citep{bickel98}. For semiparametric inference without the Donsker condition, sample-splitting is a typical approach \citep{klaassen1987,ZhengWenjing2011CTME,ChernozhukovVictor2018Dmlf}. \citet{ChernozhukovVictor2018Dmlf} referred to sample-splitting as \emph{cross-fitting} and the semiparametric inference using cross-fitting as \emph{double-debiased machine learning} (DML). For off-policy evaluation of reinforcement learning from dependent samples, \citet{KallusNathan2019EBtC} proposed a mixingale-based sample-splitting. 

\section{The ADR estimator}
\label{sec:main}
For OPVE with dependent samples, this paper proposes the ADR estimator $\widehat{R}^{\mathrm{ADR}}_T(\epol)$ defined as
\begin{align*}
\frac{1}{T}\sum^T_{t=1}\sum^K_{a=1}\Bigg\{\frac{\epol(a| X_t)\mathbbm{1}[A_t=a]\left(Y_t - \hat{f}_{t-1}(a, X_t)\right) }{\hat{g}_{t-1}(a| X_t)}+ \epol(a| X_t)\hat{f}_{t-1}(a, X_t)\Bigg\},
\end{align*}
where $\hat{g}_{t-1}$ is an estimator of $\pi_{t-1}$, constructed only from $\Omega_{t-1}$. We can use standard regression methods for constructing $\hat{f}_{t-1}$ and $\hat{g}_{t-1}$ if they satisfy the following assumptions.
\begin{assumption}
\label{asm:conv_rate1}
For $p, q > 0$ such that $p+q = 1/2$, $\|\hat{g}_{t-1}(a| X_t) - \pi_{t-1}(a| X_t, \Omega_{t-1})\|_{2}=\op(t^{-p})$, and $\|\hat{f}_{t-1}(a,X_t)-f^*(a,X_t)\|_2=\op(t^{-q})$, where the expectation of the norm is taken over $X_t$.
\end{assumption}
\begin{assumption}
\label{asm:bound_nuisance1}
There exists a constant $C_f$ such that $|\hat{f}_{t-1}(a, x)| \leq C_f$  $\forall a\in\mathcal{A},x\in\mathcal{X}$.
\end{assumption}
\begin{assumption}
\label{asm:bound_nuisance2}
There exist a constant $C_g$ such that $0 < \left|\frac{\epol(a| x)}{\hat{g}_{t-1}(a| x)}\right| \leq C_g$ $\forall a\in\mathcal{A},x\in\mathcal{X}$.
\end{assumption}

Assumption~\ref{asm:conv_rate1} requires convergence rates standard in regression estimators. For instance, we can apply nonparametric estimators proposed in MAB problems \citep{yang2002,qian2016kernel}. Under the assumptions, we show the asymptotic normality of the ADR estimator.
\begin{theorem}[Asymptotic distribution of an ADR estimator]
\label{thm:asymp_dist_adr}
Under Assumptions~\ref{asm:overlap_pol}--\ref{asm:bound_nuisance2},
\begin{align*}
\sqrt{T}\left(\widehat{R}^{\mathrm{ADR}}_T(\epol)-R(\epol)\right)\xrightarrow{d}\mathcal{N}\left(0, \Psi(\epol, \tilde{\pi})\right).
\end{align*}
%where recall that $\Psi(\epol, \tilde{\pi})$ is defined in Section~\ref{sec:asymp_eff}.
\end{theorem}
Let us put forward the following assumption.
\begin{assumption}
\label{asm:pointwise_f}
For all $x\in\mathcal{X}$ and $a\in\mathcal{A}$, $\hat{f}_{t-1}(a, x)-f^*(a, x)\xrightarrow{\mathrm{p}}0$.
\end{assumption}

The proof is shown in Appendix~\ref{appdx:proof_main}, which uses the following proposition from \citet{Kato2020adaptive}.

\begin{proposition}[Asymptotic distribution of an AIPW estimator (Corollary~1, \citet{Kato2020adaptive}).]
\label{prp:asymp_dist_a2ipw}
Under Assumptions~\ref{asm:overlap_pol}--\ref{asm:bounded_reward}, \ref{asm:bound_nuisance1}, and \ref{asm:pointwise_f}, $\sqrt{T}\left(\widehat{R}^{\mathrm{AIPW}}_T(\epol)-R(\epol)\right)\xrightarrow{d}\mathcal{N}\left(0, \Psi(\epol, \tilde{\pi})\right)$.
\end{proposition}

\begin{remark}[Consistency and double robustness]
The ADR estimator has double robustness, as with standard DR estimators; that is, if either $\hat{f}$ or $\hat{g}$ is consistent, the ADR estimator is also consistent.
\end{remark}
%with the following theorem. 
\begin{theorem}[Consistency]
Under Assumptions~\ref{asm:overlap_pol}--\ref{asm:bounded_reward} and \ref{asm:bound_nuisance1}--\ref{asm:bound_nuisance2}, if either $\hat{f}_{t-1}$ or $\hat{g}_{t-1}$ is consistent, $\widehat{R}^{\mathrm{ADR}}_T\xrightarrow{\mathrm{p}} R(\epol)$.
\end{theorem}
We can prove this by the law of large numbers for martingales (Proposition~\ref{prp:mrtgl_WLLN} in Appendix~\ref{appdx:prelim}).
%and the boundedness of  random variables. 

\vspace{-0.2cm}
\paragraph{Donsker condition.} The main reason for using step-wise estimators $\{\hat{f}_{t-1}\}^T_{t=1}$ and $\{\hat{g}_{t-1}\}^T_{t=1}$ is to regard them as constants in the expectation conditioned on $\Omega_{t-1}$. The motivation is shared with DML. For asymptotic normality shown in Theorem~\ref{thm:asymp_dist_adr}, we do not impose the Donsker condition on the nuisance estimators, $\hat{f}_{t-1}$ and $\hat{g}_{t-1}$, but only require the convergence rate conditions. We call this sample-splitting method, using $\hat{f}_{t-1}$ and $\hat{g}_{t-1}$, \emph{adaptive-fitting} and discuss again in Section~\ref{sec:relation_ddm}. In the MAB problem, convergence rate conditions in nonparametric regression, such as a nearest-neighbor regression \citep{yang2002}, Nadaraya-Watson regression \citep{qian2016kernel}, random forest \citep{Feraud2016}, kernelized linear models \citep{Chowdhury2017ker}, and neural networks \citep{Zhou2020}, have been shown. Note that we need to slightly modify the results for each situation because they are influenced by time-series behavior of logging probabilities.

\vspace{-0.2cm}
\paragraph{Convergence rate of the logging policy.} In the main theorem, we do not explicitly describe the convergence rate of the logging policy. However, from  Assumption~\ref{asm:conv_rate1}, which requires $\|\hat{g}_{t-1}(a| X_t) - \pi_{t-1}(a| X_t, \Omega_{t-1})\|_{2}=\op(t^{-p})$, $\|\tilde{\pi}(a| x) - \pi_{t-1}(a| X_t, \Omega_{t-1})\|_{2}=\op(t^{-p})$ is also required.
%because $\|\hat{g}_{t-1}(a| X_t) - \pi_{t-1}(a| X_t, \Omega_{t-1})\|_{2}$ is decomposed into $\|\hat{g}_{t-1}(a| X_t) - \tilde{\pi}(a| x)\|_{2}$ and $\|\tilde{\pi}(a| x) - \pi_{t-1}(a| X_t, \Omega_{t-1})\|_{2}$.

\vspace{-0.2cm}
\paragraph{Theoretical comparison between AIPW and ADR estimators.} There are two major differences between AIPW and ADR estimators. First, the AIPW estimator requires {\it a priori} knowledge of the true logging policy, but the ADR estimator does not. Another main difference between the two is the convergence rate of nuisance estimators $\hat{g}_{t-1}$ and $\hat{f}_{t-1}$. For the AIPW estimator, only the uniform convergence in probability is required, but the ADR estimator requires specific convergence rates on $\hat{g}_{t-1}$ and $\hat{f}_{t-1}$. This difference comes from unbiasedness. The AIPW estimator is unbiased; therefore, the convergence of the asymptotic variance is essential, where it converges with $\op(1)$ if $\hat{f}$ and $\pi_t$ is $\op(1)$. Thus, the AIPW estimator does not require specific convergence rates. On the other hand, the ADR estimator requires the asymptotic bias term to vanish in a specific order. For this purpose, it imposes specific convergence rates on the nuisance estimators. From another perspective, a standard DML for i.i.d. samples and Theorem~\ref{thm:asymp_dist_adr} require $\|\hat{g}_{t-1}(a| X_t) - \pi_{t-1}(a| X_t, \Omega_{t-1})\|_{2}=\op(t^{-p})$, $\|\hat{f}_{t-1}(a,X_t)-f^*(a,X_t)\|_2=\op(t^{-q})$, and $p+q=1/2$. %Since the AIPW estimator uses the true logging policy, the conditions hold for any consistent $\hat{f}_{t-1}$.

\vspace{-0.3cm}
\paragraph{Asymptotic efficiencies.}
As shown in Theorem~\ref{thm:asymp_dist_adr} and Proposition~\ref{prp:asymp_dist_a2ipw}, ADR and AIPW estimators achieve the semiparametric lower bound. On the other hand, the asymptotic variance of the IPW estimator using the true logging policy $\pi_t$ is larger than the lower bound \citep{hirano2003,Laan2008TheCA,Kato2020adaptive}. Although it is known that the IPW estimator using the estimated logging policy can achieve the lower bound under some conditions with i.i.d. samples \citep{hirano2003}, the asymptotic property under dependent samples is still unknown.

\section{Adaptive-fitting: DML for dependent samples}
\label{sec:relation_ddm}
We generalize the method used to derive the asymptotic normality of the ADR estimator as a variant of DML. Let us define the parameter of interest $\theta_0$ that  satisfies $\mathbb{E}[\psi(W_t; \theta_0, \eta_0)] = 0$, where $\{W_t\}^T_{t=1}$ are observations, $\eta_0$ is a nuisance parameter, and $\psi$ is a score function. We consider obtaining an asymptotic normal estimator of $\theta_0$ when using complex and data-adaptive regression methods, such as random forests, neural networks, and Lasso, to estimate the nuisance parameter $\eta_0$. Under such a situation, the Donsker condition does not hold, in general. Sample-splitting is a typical approach to control the complexities of semiparametric inference without the Donsker condition \citep{klaassen1987,ZhengWenjing2011CTME,ChernozhukovVictor2018Dmlf}. In DML of \citet{ChernozhukovVictor2018Dmlf}, the dataset with i.i.d. samples $\{W_t\}^T_{t=1}$ is separated into several subgroups. Then, a semiparametric estimator for each subgroup is constructed, but the nuisance estimators are constructed from the other subgroups. Here, the nuisance estimators are independent in the expectation conditioned on the other subgroups. Thus, the complexities are controlled without the Donsker condition, and standard nonparametric convergence rate conditions on nuisance estimators suffice to show the asymptotic normality of the semiparametric estimator. \citet{ChernozhukovVictor2018Dmlf} called the method cross-fitting.

\begin{figure}[t]
\vspace{-0.5cm}
\begin{center}
 \includegraphics[width=100mm]{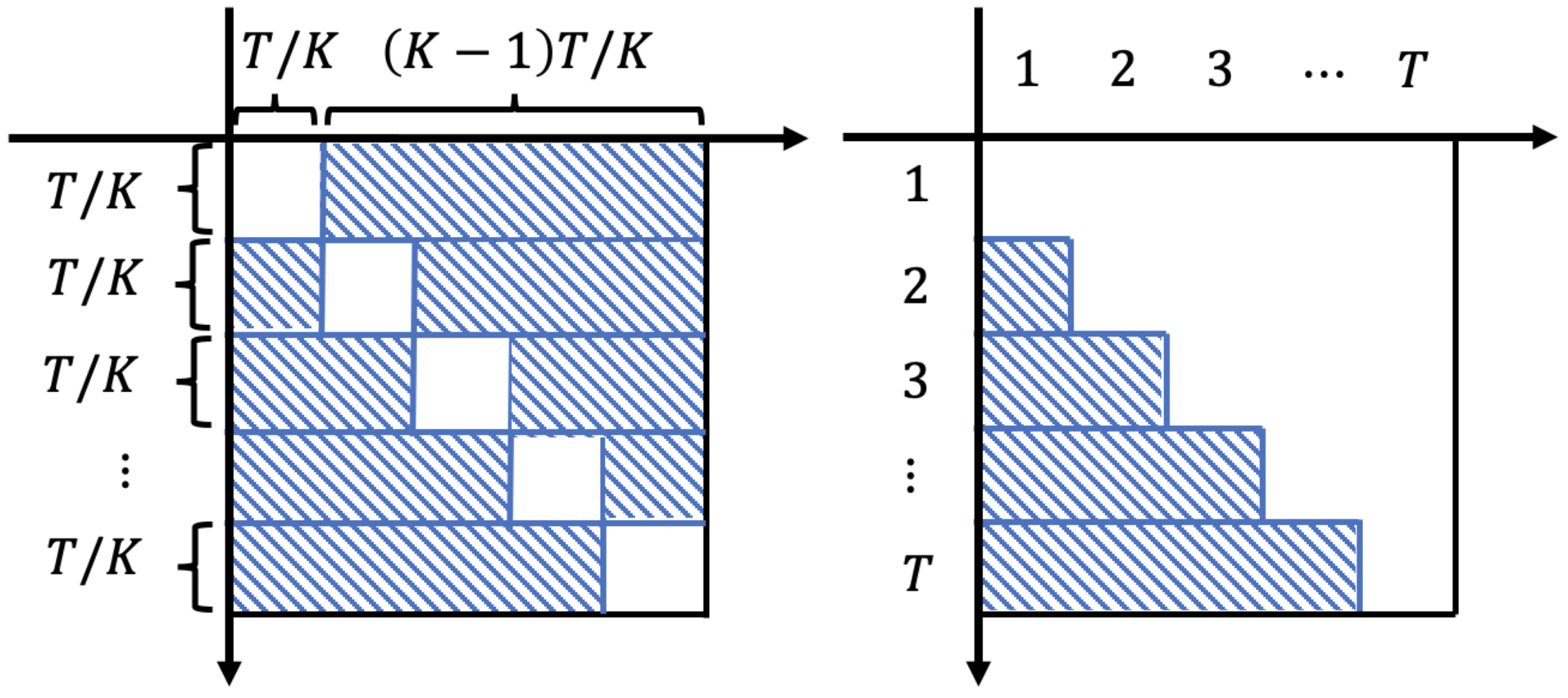}
\end{center}
\vspace{-0.3cm}
\caption{The difference between cross-fitting (left) and adaptive-fitting (right). }
\label{fig:concept}
\vspace{-0.15cm}
\end{figure} 

However, cross-fitting of \citet{ChernozhukovVictor2018Dmlf} cannot be applied when samples $\{W_t\}^T_{t=1}$ are dependent. Therefore, the ADR estimator uses step-wise nuisance estimators based on past observations $\Omega_{t-1}$. This construction is inspired by \citet{Laan2016onlinetml}, which proposed a sample-splitting for batched dependent samples. The nuisance estimators are independent in the expectation over the $t$-th samples conditioned on $\Omega_{t-1}$. Thus, we can regard our sample-splitting as another approach for DML. We call this step-wise construction \emph{adaptive-fitting}, in contrast to cross-fitting. We briefly summarize the procedure as follows: (i) at each step $t=1,2,\dots,T$, we estimate $\eta_0$ only using $\{W_s\}^{t-1}_{s=1}$ and denote the estimaters as $\{\eta_{t-1}\}^T_{t=1}$; (ii) then, we substitute $W_{t}$ and $\eta_{t-1}$ into $\psi$ and obtain an estimator $\hat{\theta}_T$ of $\theta_0$ by solving $\frac{1}{T}\sum^T_{t=1}\psi(W_t, \hat{\theta}_T, \eta_{t-1}) = 0$. Under this construction, if (a) an estimator $\check{\theta}_T$ obtained by solving $\frac{1}{T}\sum^T_{t=1}\psi(W_t, \check{\theta}_T, \eta_{0}) = 0$ has the asymptotically normal distribution and (b) the asymptotic bias decays with the rate faster than $\op(-1/\sqrt{t})$, we can obtain the asymptotically normal estimator of $\theta_0$. In Figure~\ref{fig:concept}, we illustrate the difference between cross-fitting (left) and adaptive-fitting (right). The y-axis represents samples used for constructing the OPVE estimator with nuisance estimators using samples represented by the x-axis. In the left graph, for $T/K$ samples, we calculate the sample average of the component including nuisance estimators based on the other $(K-1)T/K$ samples. In the right graph, at period $t$, we use nuisance estimators constructed from samples at $t=1,2,\dots,t-1$. To satisfy (a), we assumed that the convergence of the logging policy (Assumption~\ref{asm:stationarity}). There are other ways to satisfy this condition. For example, we can assume the existence of batches as \citet{Laan2016onlinetml} and \citet{kelly2020batched}. We briefly introduce the ADR estimator for this case in Appendix~\ref{appdx:batch}.

%Note that \citet{kelly2020batched} uses the true logging policy, but we can replace it with its estimator even when we do not the logging policy. \citet{Kato2021nonstationary} proposed another approach to this problem by assumpting the convergence of the average of the logging policy. The condition (b) is not usually satisfied. However, as \citet{ChernozhukovVictor2018Dmlf}, by using the doubly robust structure, we can satisfy the convergence rate with two nuisance estimator with the converngece rates slower than $\op(-1/\sqrt{t})$. In the ADR estimator, Assumption~\ref{asm:conv_rate1} corresponds to this condition. Note that because cross-fitting of \citet{ChernozhukovVictor2018Dmlf} and adaptive-fitting are a bit conceptual, it is difficult to describe them as algorithms, and their specific use is case-by-case, as is the case with the ADR estimator.

\section{Paradox concerning logging policy}
\label{sec:paradox}
The advantage of the ADR estimator is that, unlike the AIPW estimator, it does not require the true logging policy. However, we find that the ADR estimator often achieves a smaller MSE than the AIPW estimator, which requires the true logging policy. We refer to this phenomenon as a paradox concerning logging policy.

\subsection{Instability of the AIPW estimator}
\label{sec:a2ipw}
\citet{hadad2019} pointed out that the AIPW estimator tends to be unstable when the nuisance parameter $\pi_{t-1}$ can take a value close to zero before converging to $\tilde{\pi}$. To prevent this instability, \citet{hadad2019} proposed the Adaptively Weighted AIPW (AW-AIPW) estimator by adding evaluation weights, which converge to a constant almost surely, to the AIPW estimator proposed in other studies, such as \citet{Laan2008TheCA}. Note that the AW-AIPW estimator cannot mitigate the assumptions on the logging policy $\pi_{t-1}$ required in the AIPW \citep{Laan2008TheCA} and our ADR estimator, namely $\pi_{t-1}$ is non-zero and converges. The aim of adaptive weighting is not in theoretical improvement of the AIPW estimator, but in its empirical stabilization.

\subsection{Sensitivity of the ADR estimator to the logging policy.}
The ADR estimator replaces the true logging policy $\pi_{t-1}$ with its estimator. We find that by replacing the true logging policy with its estimator, we can control the instability caused from $\pi_{t-1}$ by constructing well-formed $\hat{g}_{t-1}$. For example, if we know that the range of $\tilde{\pi}$ is $(\varepsilon, 1-\varepsilon)$ for $0< \varepsilon < 1-\varepsilon < 1$, we can add the clipping technique when estimating $\hat{g}_{t-1}$. For example, in early periods, we can set $\hat{g}_{t-1}$ as $0.5$. Such techniques is greatly beneficial.
%, as we show in the experiments in later sections. 
Note that if the range of $\tilde{\pi}$ is truly $(\varepsilon, 1-\varepsilon)$, the clipping of $\hat{g}_{t-1}$ does not cause clipping bias; that is, the estimator correctly converges to the asymptotic distribution shown in Theorem~\ref{thm:asymp_dist_adr}. 

The main difference between the AW-AIPW estimator and the ADR estimator is that the former requires information of the true logging policy $\pi_{t-1}$, while the latter does not by replacing it with its estimator. Note that the AW-AIPW estimator requires the same assumptions on $\pi_{t-1}$ as the AIPW and ADR estimator and loses unbiasedness due to the evaluation weight. In addition, the construction of the adaptive weight is not obvious when there are covariates $X_t$ because all but one method for weight construction in \citet{hadad2019} does not consider covariates. Our finding also shares motivation with weight clipping for i.i.d. samples, such as \citet{Bottou2013}. We can regard the AW-AIPW estimator as a variant of this method, where the weight clipping can decay as $\pi_t$ converges, and be ignored asymptotically. However, our interest is in asymptotic normality from dependent samples without the true logging policy, which is not discussed in the literature.

\begin{table*}[t]
\vspace{-0.1cm}
\caption{The results of Section~\ref{sec:a2ipw} with sample sizes $T=250, 500,750$. We show the RMSE, SD, and coverage ratio of the confidence interval (CR). We highlight in red bold two estimators with the lowest RMSE.
%and highlight in under line the estimator with the lowest RMSE among estimators with asymptotic normality. 
Estimators with asymptotic normality are marked with $\dagger$, and estimators that do not require the true logging policy are marked with $*$.} 
\vspace{-0.3cm}
\label{tbl:exp_table1}
\begin{center}
\scalebox{0.61}[0.61]{
\begin{tabular}{|l||rrr||rrr|rrr|rrr||rrr|rrr|}
\hline
\multicolumn{19}{|c|}{LinUCB policy}\\
\hline
{$T$} &    \multicolumn{3}{c||}{ADR $\dagger*$} &   \multicolumn{3}{c|}{IPW $\dagger$} &   \multicolumn{3}{c|}{AIPW $\dagger$} &   
\multicolumn{3}{c||}{AW-AIPW $\dagger$} &   \multicolumn{3}{c|}{DM $*$} &    \multicolumn{3}{c|}{EIPW $*$}\\
\hline
 &      RMSE &      SD &      CR &      RMSE &      SD &      CR &      RMSE &      SD &      CR &      RMSE &     SD &     CR &  RMSE &      SD &      CR &  RMSE &      SD &      CR\\
\hline
250 &  \textcolor{red}{\textbf{0.054}} &  0.003 &  0.97 &  0.117 &  0.034 &  0.87 &  0.106 &  0.036 &  0.90 &  0.152 &  0.022 &  0.09 &  0.073 &  0.006 &  0.12 &  0.103 &  0.017 &  0.82 \\
500 &  \textcolor{red}{\textbf{0.040}} &  0.002 &  0.94 &  0.078 &  0.014 &  0.93 &  0.059 &  0.007 &  0.98 &  0.179 &  0.020 &  0.00 &  0.046 &  0.002 &  0.16 &  0.121 &  0.013 &  0.40 \\
750 &  \textcolor{red}{\textbf{0.033}} &  0.001 &  0.97 &  0.064 &  0.008 &  0.90 &  0.059 &  0.006 &  0.97 &  0.176 &  0.015 &  0.00 &  0.040 &  0.002 &  0.15 &  0.122 &  0.012 &  0.22 \\
\hline
\end{tabular}
} 
\end{center}
\vspace{-0.3cm}
\begin{center}
\scalebox{0.61}[0.61]{
\begin{tabular}{|l||rrr||rrr|rrr|rrr||rrr|rrr|}
\hline
\multicolumn{19}{|c|}{LinTS policy}\\
\hline
{$T$} &    \multicolumn{3}{c||}{ADR $\dagger*$} &   \multicolumn{3}{c|}{IPW $\dagger$} &   \multicolumn{3}{c|}{AIPW $\dagger$} &   
\multicolumn{3}{c||}{AW-AIPW $\dagger$} &   \multicolumn{3}{c|}{DM $*$} &    \multicolumn{3}{c|}{EIPW $*$}\\
\hline
 &      RMSE &      SD &      CR &      RMSE &      SD &      CR &      RMSE &      SD &      CR &      RMSE &     SD &     CR &  RMSE &      SD &      CR &  RMSE &      SD &      CR\\
\hline
250 &  \textcolor{red}{\textbf{0.049}} &  0.002 &  0.96 &  0.128 &  0.037 &  0.82 &  0.107 &  0.031 &  0.93 &  0.121 &  0.016 &  0.20 &  0.069 &  0.005 &  0.15 &  0.077 &  0.011 &  0.89 \\
500 &  \textcolor{red}{\textbf{0.035}} &  0.001 &  0.95 &  0.151 &  0.139 &  0.91 &  0.122 &  0.086 &  0.91 &  0.136 &  0.014 &  0.06 &  0.048 &  0.002 &  0.13 &  0.090 &  0.007 &  0.57 \\
750 &  \textcolor{red}{\textbf{0.028}} &  0.001 &  0.97 &  0.069 &  0.010 &  0.90 &  0.053 &  0.005 &  0.91 &  0.154 &  0.013 &  0.00 &  0.036 &  0.001 &  0.14 &  0.113 &  0.009 &  0.23 \\
\hline
\end{tabular}
} 
\end{center}
\vspace{-0.5cm}
\end{table*}

\begin{figure}[t]
\begin{center}
 \includegraphics[width=130mm]{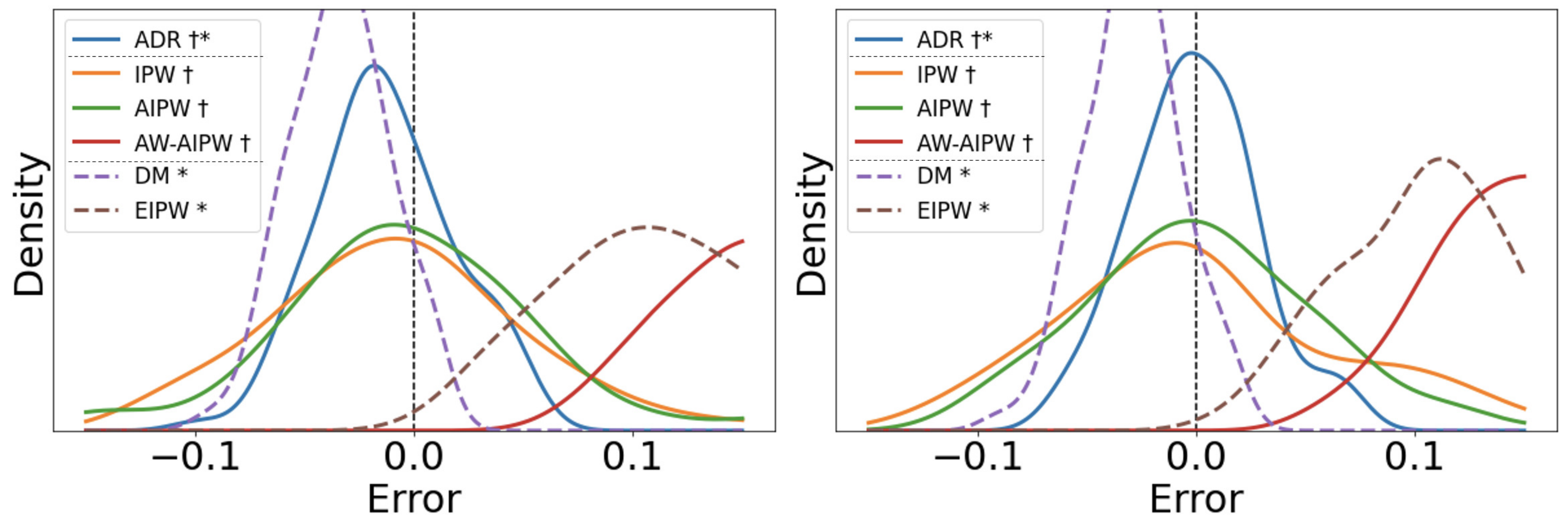}
\end{center}
\vspace{-0.5cm}
\caption{This figure illustrates the error distributions of estimators for OPVE from dependent samples generated with the LinUcB(left) and LinTS(right) with the sample size $750$. We smoothed the error distributions using kernel density estimation. Estimators with asymptotic normality are marked with $\dagger$, and estimators that do not require a true logging policy are marked with $*$.}
\label{fig:bias}
\vspace{-0.5cm}
\end{figure} 

\subsection{Simulation studies for OPVE estimators}
\label{sec:num_exp}
For reasons given in the previous section, the ADR estimator may empirically perform better than the AIPW estimator because the estimator $\hat{g}_{t-1}$ stabilizes the ADR estimator by absorbing the instability of $\pi_{t-1}$. 
%For instance, even if $\pi_{t-1}$ takes an extreme value at that period, the ADR estimator is more robust than the AIPW estimator because the it replaces the true logging policy with its estimator.
We empirically show this paradox using a synthetic dataset. We generate an artificial pair of covariate and potential outcome, $(X_t, Y_t(1), Y_t(2), Y_t(3))$. The covariate $X_t$ is a $10$ dimensional vector generated from the standard normal distribution. For $a\in\{1,2,3\}$, $Y_t(a)=1$ if $a$ is chosen with probability $q(a| x) = \frac{\exp(g(a, x))}{\sum^3_{a'=1}\exp(g(a', x))}$, where $g(1, x) = \sum^{10}_{d=1} X_{t,d}$, $g(2, x) = \sum^{10}_{d=1} W_dX^2_{t,d}$, $g(3, x) = \sum^{10}_{d=1} W_d|X_{t,d}|$, and $W_d$ is uniform randomly chosen from $\{-1, 1\}$. We generate three datasets, $\mathcal{S}^{(1)}_{T_{(1)}}$, $\mathcal{S}^{(2)}_{T_{(2)}}$, and $\mathcal{S}^{(3)}_{T_{(3)}}$, where $\mathcal{S}^{(m)}_{T_{(m)}} = \{(X^{(m)}_t, Y^{(m)}_t(1), Y^{(m)}_t(2), Y^{(m)}_t(3))\}^{T_{(m)}}_{t=1}$. First, we train an evaluation function $\epol$ by solving a prediction problem between $X^{(1)}_t$ and $Y^{(1)}_t(1), Y^{(1)}_t(2), Y^{(1)}_t(3)$ using $\mathcal{S}^{(1)}_{T_{(1)}}$. Then, we regard the evaluation function $\epol$ as a logging policy and apply it on the independent dataset $\mathcal{S}^{(2)}_{T_{(2)}}$ to obtain $\{(X'_t, A'_t, Y'_t)\}^{T_{(2)}}_{t=1}$, where $A'_t$ is an action chosen from the evaluation function and $Y^{(m)}_t = \sum^3_{a=1}\mathbbm{1}[A^{(m)}_t = a]Y^{(m)}_t(a)$. Then, we estimate the true value $R(\epol)$ as $\frac{1}{T_{(2)}}\sum^{T_{(2)}}_{t=1}Y^{(m)}_t$. Next, using the datasets $\mathcal{S}^{(3)}_{T_{(3)}}$ and an MAB algorithm, we generate a bandit dataset $\mathcal{S}=\{(X_t, A_t, Y_t)\}^{T_{(3)}}_{t=1}$. For $\mathcal{S}$, we apply the ADR estimator, IPW estimator with the true logging policy (IPW), AIPW estimator (AIPW), AW-AIPW estimator (AW-AIPW) IPW estimator with estimated logging policy (EIPW), and DM estimator (DM). For the AW-AIPW estimator, we use the evaluation weight $\sqrt{\pi_{a,x}(t)}/T$, proposed in \citet{hadad2019}. The other proposed weights of \citet{hadad2019} seem to be specific to situations without covariate $X_t$. For estimating $f^*$ and the logging policy, we use the kernelized Ridge least squares and logistic regression, respectively. We use a Gaussian kernel, and both hyper-parameters of the Ridge and kernel are chosen from $\{0.01, 0.1, 1\}$. We define the estimation error as $R(\epol) - \widehat{R}(\epol)$. We conduct ten experiments by changing sample sizes and MAB algorithms. For the sample size $T_{(3)}$, we use $100$, $250$, $500$, $750$, and $10,000$ with the LinUCB and LinTS algorithms. For the sample sizes $T_{(1)}$ and $T_{(2)}$, we use $1,000$ and $100,000$ respectively. We conduct $100$ trials to obtain the root MSEs (RMSEs), the standard deviations of MSEs (SDs), and the coverage ratios (CRs) of the $95\%$ confidence interval (percentage that the confidence interval covers the true value). The results with $T=250$, $500$, $750$ are shown in Table~\ref{tbl:exp_table1} and those with $T=750$ are shown in Figure~\ref{fig:bias}. The other results are shown in Appendix~\ref{appdx:det_exp}. In all cases, the ADR estimator performs well compared to other methods. Although the AIPW estimator has asymptotic normality, the performance is not comparable with the ADR estimator. The reason why the AW-AIPW estimator underperforms is that the performance depends on the choice of evaluation weights, and the weight proposed by \citet{hadad2019} is not suitable for our case with covariates. The EIPW estimator suffers from the dependency problem. We postulate that this is owing to the instability of $\pi_t$.

\begin{remark}[Paradox for i.i.d. samples]
This paradox is similar to the well-known property that the IPW estimator using an estimated propensity score shows a smaller asymptotic variance than the IPW using the true one \citep{hirano2003,Henmi2004paradox,Henmi2007imp}. However, as discussed above, we consider this to be a different phenomenon. In these studies, the paradox is mainly explained by differences in the asymptotic variance between IPW estimators with the true and estimated propensity score. On the other hand, for our case, AIPW and ADR estimators have the same asymptotic variance, unlike IPW-type estimators. Therefore, we cannot resolve the paradox by the aforementioned explanations. The instability of the AIPW estimator is consistent with the findings of \citet{hadad2019}, but, in our paper, we suggest using the ADR estimator as another solution.
\end{remark}

\section{Experiments}
\label{sec:exp}
Following \citet{dudik2011doubly}, we evaluate the estimators using classification datasets by transforming them into contextual bandit data. From the LIBSVM repository \citep{CC01a}, we use the {\tt mnist}, {\tt satimage}, {\tt sensorless}, and {\tt connect-4} datasets \footnote{\url{https://www.csie.ntu.edu.tw/~cjlin/libsvmtools/datasets/.}}. The dataset description is in Table~\ref{Dataset} of Appendix~\ref{appdx:det_exp}. We construct a logging policy as $\pi_t(a| x) = \alpha\pi^m_t(a| x)+(1-\alpha)\frac{0.1}{K}$, where $\alpha\in (0,1)$ is a constant, and $\pi^m_t(a| x)$ is a policy such that $\pi^m_t(a| x) = 1$ for an action $a\in\mathcal{A}$ and $\pi^m_t(a'| x) = 0$ for the other actions ($a'\neq a$). The policy $\pi^m_t(a| x)$ is determined by logistic regression and MAB algorithms. For MAB algorithms, we use upper confidence bound and Thompson sampling with a linear model, which are denoted as LinUCB \citep{Wei2011} and LinTS \citep{Agrawal2013}. The evaluation function is fixed at $\epol(a| x) = 0.9\pi^d(a| x)+\frac{0.1}{K}$, where $\pi_d$ is a prediction of a logistic regression. 

We focus on the ADR estimator and compare it to the IPW, EIPW, AIPW, and DM estimators, as done in Section~\ref{sec:num_exp}. We omit the AW-AIPW estimator, as we found that the performance depends on the choice of the evaluation weight (see Section~\ref{sec:paradox}). Additionally, the paper does not consider or discuss situations where there exist covariates. We compare those estimators using the benchmark datasets generated through the process explained above. For $\alpha\in\{0.7, 0.4, 0.1\}$ and the sample sizes $800$, $1,000$, and $1,200$, we calculate the RMSEs and the SDs over $10$ trials. The results of {\tt mnist} and {\tt satimage} with the LinUCB policy and $T=1,000$ are shown in Tables~\ref{tbl:exp_man_table1}. The full results, including the results with the LinTS policy and i.i.d. samples, are shown in Appendix~\ref{appdx:det_exp}. As discussed in Section~\ref{sec:paradox}, the ADR estimator performs well. Although the asymptotic distributions of AIPW and ADR estimators are the same, the AIPW estimator shows poorer performance. As Section~\ref{sec:num_exp}, we consider this to be because the estimator of $\pi_t$ is absorbing the instability of $\pi_t$. 

\begin{table*}[t]
\vspace{-0.2cm}
\caption{The results of benchmark datasets with the LinUCB policy. We highlight in red bold the estimator with the lowest RMSE and highlight in under line the estimator with the lowest RMSE among estimators that do not use the true logging policy. Estimators with asymptotic normality are marked with $\dagger$, and estimators that do not require the true logging policy are marked with $*$.} 
\label{tbl:exp_man_table1}
%\begin{minipage}{5cm}
\begin{center}
\scalebox{0.70}[0.70]{
\begin{tabular}{|l||rr||rr|rr||rr|rr|}
\hline
{\tt mnist}$\ \ \ \ \ \ \ \ $ &    \multicolumn{2}{|c||}{ADR\ $\dagger*$} &   \multicolumn{2}{c|}{IPW $\dagger$} &   \multicolumn{2}{c||}{AIPW $\dagger$} &    \multicolumn{2}{c|}{DM $*$} &    \multicolumn{2}{c|}{EIPW $*$} \\
\hline
$\alpha$ &      RMSE &      SD &      RMSE &      SD &      RMSE &      SD &      RMSE &      SD &      RMSE &      SD \\
\hline
0.7 &  \underline{\textcolor{red}{\textbf{0.046}}} &  0.002 &  0.100 &  0.011 &  0.162 &  0.027 &  0.232 &  0.013 &  0.148 &  0.014 \\
0.4 &  \underline{\textcolor{red}{\textbf{0.028}}} &  0.001 &  0.068 &  0.005 &  0.112 &  0.009 &  0.249 &  0.010 &  0.080 &  0.004 \\
0.1 &  \underline{0.086} &  0.006 &  \textcolor{red}{\textbf{0.078}} &  0.006 &  0.085 &  0.008 &  0.299 &  0.024 &  0.091 &  0.008 \\
\hline
\end{tabular}
}
\end{center}
\vspace{-0.3cm}
%\end{minipage}
%\hfill
%\begin{minipage}{5cm}
\begin{center}
\scalebox{0.70}[0.70]{
\begin{tabular}{|l||rr||rr|rr||rr|rr|}
\hline
{\tt satimage}$\ \ $ &    \multicolumn{2}{|c||}{ADR\ $\dagger*$} &   \multicolumn{2}{c|}{IPW $\dagger$} &   \multicolumn{2}{c||}{AIPW $\dagger$} &    \multicolumn{2}{c|}{DM $*$} &    \multicolumn{2}{c|}{EIPW $*$} \\
\hline
$\alpha$ &      RMSE &      SD &      RMSE &      SD &      RMSE &      SD &      RMSE &      SD &      RMSE &      SD \\
\hline
0.7 &  \underline{\textcolor{red}{\textbf{0.013}}} &  0.000 &  0.098 &  0.011 &  0.060 &  0.004 &  0.037 &  0.001 &  0.056 &  0.002 \\
0.4 &  \underline{\textcolor{red}{\textbf{0.022}}} &  0.000 &  0.078 &  0.008 &  0.019 &  0.000 &  0.043 &  0.001 &  0.060 &  0.002 \\
0.1 &  \underline{\textcolor{red}{\textbf{0.029}}} &  0.001 &  0.078 &  0.005 &  0.061 &  0.008 &  0.041 &  0.002 &  0.041 &  0.002 \\
\hline
\end{tabular}
}
\end{center}
%\end{minipage}
\vspace{-0.65cm}
\end{table*}

Note that the ADR, IPW, and AIPW estimators are asymptotically normal, but the IPW and APIW estimators are not feasible when the true logging policy is not given. To the best of our knowledge, asymptotic normality of the EIPW and DM estimators has not been shown when samples are dependent, and the proof is non-trivial owing to the dependency and the Donsker condition. 

%In experiments, we find that the ADR estimator tends to perform better than the IPW and AIPW estimators with asymptotic normality for dependent samples inspire of their requirements of the true logging policy, unlike the ADR estimator.
%As a surprising discovery, the ADR estimator shows improved performance compared to the AIPW estimator, even though the AIPW estimator uses more information (the true logging policy $\pi_t$) than the ADR estimator, even though their asymptotic properties are the same. As discussed above, we consider that this result is due to the unstable behavior of the logging policy. Even when knowing the true logging policy $\pi_t$, we can stabilize the result by re-estimating the logging policy from $(A_t, X_t)$. 

%When samples are i.i.d., the AIPW estimator often outperforms the ADR estimator. This result also implies that, in our observed paradox for dependent samples, the asymptotic property is not essential. Thus, we consider our reported paradox is different from the paradox reported for i.i.d. samples.

\section{Conclusion}
DR-type estimators are crucial in causal inference because they do not assume {\it a priori} knowledge of the true logging policy and they asymptotically follow a normal distribution under standard convergence rate conditions of nuisance estimators. However, existing studies have rarely discussed DR-type estimators when samples are dependent. We derived the ADR estimator and proposed adaptive-fitting as a variant of DML to obtain asymptotic normality. In experiments, we found a paradox that the ADR estimator tends to be more stable than the AIPW estimator and conjectured that this is because the ADR estimator absorbs the instability of the true logging policy $\pi_t$. 

In OPVE with dependent samples, we need to put some assumptions on the behavior of $\pi_t$ to obtain asymptotic normality (Assumption~\ref{asm:stationarity}). These assumptions can be broken if the time-series is very complicated. However, we note that this is a limitation of the entire field. The asymptotic normality and double robustness of OPVE estimators are necessary theoretical properties to avoid deriving false causality. Since causal inference is often used in applications related closely to public policy, we consider understanding these limitations critical.

\bibliographystyle{icml2021}
\bibliography{arXiv.bbl}

\clearpage
\onecolumn
\appendix

\section{Mathematical preliminaries}
\label{appdx:prelim}

\begin{proposition}[$L^r$ Convergence Theorem, \citet{loeve1977probability}]
\label{prp:lr_conv_theorem}
Let $0<r<\infty$, suppose that $\mathbb{E}\big[|a_n|^r\big] < \infty$ for all $n$ and that $a_n \xrightarrow{\mathrm{p}}a$ as $n\to \infty$. The following are equivalent: 
\begin{description}
\item{(i)} $a_n\to a$ in $L^r$ as $n\to\infty$;
\item{(ii)} $\mathbb{E}\big[|a_n|^r\big]\to \mathbb{E}\big[|a|^r\big] < \infty$ as $n\to\infty$; 
\item{(iii)} $\big\{|a_n|^r, n\geq 1\big\}$ is uniformly integrable.
\end{description}
\end{proposition}

\begin{proposition}
\label{prp:mrtgl_WLLN}[Weak Law of Large Numbers for Martingale, \citet{hall2014martingale}]
Let $\{S_n = \sum^{n}_{i=1} X_i, \mathcal{H}_{t}, t\geq 1\}$ be a martingale and $\{b_n\}$ a sequence of positive constants with $b_n\to\infty$ as $n\to\infty$. Then, writing $X_{ni} = X_i\mathbbm{1}[|X_i|\leq b_n]$, $1\leq i \leq n$, we have that $b^{-1}_n S_n \xrightarrow{\mathrm{p}} 0$ as $n\to \infty$ if 
\begin{description}
\item[(i)] $\sum^n_{i=1}P(|X_i| > b_n)\to 0$;
\item[(ii)] $b^{-1}_n\sum^n_{i=1}\mathbb{E}[X_{ni}| \mathcal{H}_{t-1}] \xrightarrow{\mathrm{p}} 0$, and;
\item[(iii)] $b^{-2}_n \sum^n_{i=1}\big\{\mathbb{E}[X^2_{ni}] - \mathbb{E}\big[\mathbb{E}\big[X_{ni}| \mathcal{H}_{t-1}\big]\big]^2\big\}\to 0$.
\end{description}
\end{proposition}
\begin{remark} The weak law of large numbers for martingale holds when the random variable is bounded by a constant.
\end{remark}

\section{Proof of Theorem~\ref{thm:asymp_dist_adr}}
\label{appdx:proof_main}

\begin{proof}[Proof of Theorem~\ref{thm:asymp_dist_adr}]
We show asymptotic normality of 
\begin{align*}
\widehat{R}^{\mathrm{ADR}}_T(\epol)=\frac{1}{T}\sum^T_{t=1}\left\{\phi_1(X_t, A_t, Y_t; \hat{g}_{t-1}, \hat{f}_{t-1}) + \phi_2(X_t; \hat{f}_{t-1})\right\},
\end{align*}
where
\begin{align*}
&\phi_1(X_t, A_t, Y_t; g, f)=\sum^K_{a=1}\frac{\epol(a| X_t)\mathbbm{1}[A_t=a]\left(Y_t - f(a, X_t)\right) }{g(a| X_t)}\\
&\phi_2(X_t; f)=\sum^K_{a=1}\epol(a| X_t)f(a, X_t).
\end{align*}

Let us define an AIPW estimator with $\hat{f}=f^*$ as
\begin{align*}
\widehat{R}^*(\epol)=\frac{1}{T}\sum^T_{t=1}\left\{\phi_1(X_t, A_t, Y_t; \pi_{t-1}, f^*) + \phi_2(X_t; f^*)\right\}.
\end{align*}
We decompose $\sqrt{T}\left(R^{\mathrm{ADRE}}(\epol) - R(\epol)\right)$.
\begin{align*}
&\sqrt{T}\left(R^{\mathrm{ADRE}}(\epol) - R(\epol)\right)= \sqrt{T}\left(\widehat{R}^{\mathrm{ADR}}_T(\epol) - \widehat{R}^*(\epol) + \widehat{R}^*(\epol) - R(\epol)\right).
\end{align*}
From Proposition~\ref{prp:asymp_dist_a2ipw} of \citet{Kato2020adaptive}, condition~(i) and (ii), and Assumption~\ref{asm:overlap_pol} and \ref{asm:bounded_reward}, because $\sqrt{T}\left(\widehat{R}^*(\epol) - R(\epol)\right)$ follows asymptotic normal distribution, we want to show 
\begin{align*}
\widehat{R}^{\mathrm{ADR}}_T(\epol) - \widehat{R}^*(\epol) = \mathrm{o}_p(1/\sqrt{T}).
\end{align*}
Here, we have 
\begin{align*}
&\widehat{R}^{\mathrm{ADR}}_T(\epol) - \widehat{R}^*(\epol)\\
&=\frac{1}{T}\sum^T_{t=1}\Bigg\{\phi_1(X_t, A_t, Y_t; \hat{g}_{t-1}, \hat{f}_{t-1}) - \phi_1(X_t, A_t, Y_t; \pi_{t-1}, f^*)\\
&\ \ \ -\mathbb{E}\left[\phi_1(X_t, A_t, Y_t; \hat{g}_{t-1}, \hat{f}_{t-1}) - \phi_1(X_t, A_t, Y_t; \pi_{t-1}, f^*)| \Omega_{t-1}\right]\\
&\ \ \ + \phi_2(X_t; \hat{f}_{t-1}) - \phi_2(X_t; f^*) -\mathbb{E}\left[\phi_2(X_t; \hat{f}_{t-1}) - \phi_2(X_t; f^*)| \Omega_{t-1}\right]\Bigg\}\\
&\ \ \ + \frac{1}{T}\sum^T_{t=1}\mathbb{E}\left[\phi_1(X_t, A_t, Y_t; \hat{g}_{t-1}, \hat{f}_{t-1})| \Omega_{t-1}\right] + \frac{1}{T}\sum^T_{t=1}\mathbb{E}\left[\phi_2(X_t; \hat{f}_{t-1})| \Omega_{t-1}\right]\\
&\ \ \ - \frac{1}{T}\sum^T_{t=1}\mathbb{E}\left[\phi_1(X_t, A_t, Y_t; \pi_{t-1}, f^*)| \Omega_{t-1}\right] - \frac{1}{T}\sum^T_{t=1}\mathbb{E}\left[\phi_2(X_t; f^*)| \Omega_{t-1}\right].
\end{align*}
In the following parts, we separately show that
\begin{align}
\label{eq:part1}
&\sqrt{T}\frac{1}{T}\sum^T_{t=1}\Bigg\{\phi_1(X_t, A_t, Y_t; \hat{g}_{t-1}, \hat{f}_{t-1}) - \phi_1(X_t, A_t, Y_t; \pi_{t-1}, f^*)\\
&\ \ \ -\mathbb{E}\left[\phi_1(X_t, A_t, Y_t; \hat{g}_{t-1}, \hat{f}_{t-1}) - \phi_1(X_t, A_t, Y_t; \pi_{t-1}, f^*)| \Omega_{t-1}\right]\nonumber\\
&\ \ \ + \phi_2(X_t; \hat{f}_{t-1}) - \phi_2(X_t; f^*) -\mathbb{E}\left[\phi_2(X_t; \hat{f}_{t-1}) - \phi_2(X_t; f^*)| \Omega_{t-1}\right]\Bigg\}\nonumber\\
&= \mathrm{o}_p(1);\nonumber
\end{align}
and 
\begin{align}
\label{eq:part2}
&\frac{1}{T}\sum^T_{t=1}\mathbb{E}\left[\phi_1(X_t, A_t, Y_t; \hat{g}_{t-1}, \hat{f}_{t-1})| \Omega_{t-1}\right] + \frac{1}{T}\sum^T_{t=1}\mathbb{E}\left[\phi_2(X_t; \hat{f}_{t-1})| \Omega_{t-1}\right]\\
&- \frac{1}{T}\sum^T_{t=1}\mathbb{E}\left[\phi_1(X_t, A_t, Y_t; \pi_{t-1}, f^*)| \Omega_{t-1}\right] - \frac{1}{T}\sum^T_{t=1}\mathbb{E}\left[\phi_2(X_t; f^*)| \Omega_{t-1}\right] = \mathrm{o}_p(1/\sqrt{T}).\nonumber
\end{align}

\paragraph{Proof of \eqref{eq:part1}.}
For any $\varepsilon > 0$, to show that 
\begin{align*}
&\mathbb{P}\Bigg(\Bigg|\sqrt{T}\frac{1}{T}\sum^T_{t=1}\Bigg\{\phi_1(X_t, A_t, Y_t; \hat{g}_{t-1}, \hat{f}_{t-1}) - \phi_1(X_t, A_t, Y_t; \pi_{t-1}, f^*)\\
&\ \ \ -\mathbb{E}\left[\phi_1(X_t, A_t, Y_t; \hat{g}_{t-1}, \hat{f}_{t-1}) - \phi_1(X_t, A_t, Y_t; \pi_{t-1}, f^*)| \Omega_{t-1}\right]\\
&\ \ \ + \phi_2(X_t; \hat{f}_{t-1}) - \phi_2(X_tt; f^*) -\mathbb{E}\left[\phi_2(X_t; \hat{f}_{t-1}) - \phi_2(X_t; f^*)| \Omega_{t-1}\right]\Bigg\}\Bigg| > \varepsilon \Bigg)\\
& \to 0,
\end{align*}
we show that the mean is $0$ and the variance of the component converges to $0$. Then, from the Chebyshev's inequality, this result yields the statement.

The mean is calculated as 
\begin{align*}
&\sqrt{T}\frac{1}{T}\sum^T_{t=1}\mathbb{E}\Bigg[\Bigg\{\phi_1(X_t, A_t, Y_t; \hat{g}_{t-1}, \hat{f}_{t-1}) - \phi_1(X_t, A_t, Y_t; \pi_{t-1}, f^*)\\
&\ \ \ \ \ \ \ -\mathbb{E}\left[\phi_1(X_t, A_t, Y_t; \hat{g}_{t-1}, \hat{f}_{t-1}) - \phi_1(X_t, A_t, Y_t; \pi_{t-1}, f^*)| \Omega_{t-1}\right]\\
&\ \ \ \ \ \ \ + \phi_2(X_t; \hat{f}_{t-1}) - \phi_2(X_t; f^*) -\mathbb{E}\left[\phi_2(X_t; \hat{f}_{t-1}) - \phi_2(X_t; f^*)| \Omega_{t-1}\right]\Bigg\}\Bigg]\\
&=\sqrt{T}\frac{1}{T}\sum^T_{t=1}\mathbb{E}\Bigg[\mathbb{E}\Bigg[\Bigg\{\phi_1(X_t, A_t, Y_t; \hat{g}_{t-1}, \hat{f}_{t-1}) - \phi_1(X_t, A_t, Y_t; \pi_{t-1}, f^*)\\
&\ \ \ \ \ \ \ -\mathbb{E}\left[\phi_1(X_t, A_t, Y_t; \hat{g}_{t-1}, \hat{f}_{t-1}) - \phi_1(X_t, A_t, Y_t; \pi_{t-1}, f^*)| \Omega_{t-1}\right]\\
&\ \ \ \ \ \ \ + \phi_2(X_t; \hat{f}_{t-1}) - \phi_2(X_t; f^*) -\mathbb{E}\left[\phi_2(X_t; \hat{f}_{t-1}) - \phi_2(X_t; f^*)| \Omega_{t-1}\right]\Bigg\}| \Omega_{t-1}\Bigg]\Bigg]\\
& = 0
\end{align*}

Because the mean is $0$, the variance is
\begin{align*}
&\mathrm{Var}\Bigg(\sqrt{T}\frac{1}{T}\sum^T_{t=1}\Bigg\{\phi_1(X_t, A_t, Y_t; \hat{g}_{t-1}, \hat{f}_{t-1}) - \phi_1(X_t, A_t, Y_t; \pi_{t-1}, f^*)\\
&\ \ \ \ \ \ \ -\mathbb{E}\left[\phi_1(X_t, A_t, Y_t; \hat{g}_{t-1}, \hat{f}_{t-1}) - \phi_1(X_t, A_t, Y_t; \pi_{t-1}, f^*)| \Omega_{t-1}\right]\\
&\ \ \ \ \ \ \ + \phi_2(X_t; \hat{f}_{t-1}) - \phi_2(X_t; f^*) -\mathbb{E}\left[\phi_2(X_t; \hat{f}_{t-1}) - \phi_2(X_t; f^*)| \Omega_{t-1}\right]\Bigg\}\Bigg)\\
&=\mathbb{E}\Bigg[\Bigg(\sqrt{T}\frac{1}{T}\sum^T_{t=1}\Bigg\{\phi_1(X_t, A_t, Y_t; \hat{g}_{t-1}, \hat{f}_{t-1}) - \phi_1(X_t, A_t, Y_t; \pi_{t-1}, f^*)\\
&\ \ \ \ \ \ \ -\mathbb{E}\left[\phi_1(X_t, A_t, Y_t; \hat{g}_{t-1}, \hat{f}_{t-1}) - \phi_1(X_t, A_t, Y_t; \pi_{t-1}, f^*)| \Omega_{t-1}\right]\\
&\ \ \ \ \ \ \ + \phi_2(X_t; \hat{f}_{t-1}) - \phi_2(X_t; f^*) -\mathbb{E}\left[\phi_2(X_t; \hat{f}_{t-1}) - \phi_2(X_t; f^*)| \Omega_{t-1}\right]\Bigg\}\Bigg)^2\Bigg]\\
&=\frac{1}{T}\mathbb{E}\Bigg[\Bigg(\sum^T_{t=1}\Bigg\{\phi_1(X_t, A_t, Y_t; \hat{g}_{t-1}, \hat{f}_{t-1}) - \phi_1(X_t, A_t, Y_t; \pi_{t-1}, f^*)\\
&\ \ \ \ \ \ \ -\mathbb{E}\left[\phi_1(X_t, A_t, Y_t; \hat{g}_{t-1}, \hat{f}_{t-1}) - \phi_1(X_t, A_t, Y_t; \pi_{t-1}, f^*)| \Omega_{t-1}\right]\\
&\ \ \ \ \ \ \ + \phi_2(X_t; \hat{f}_{t-1}) - \phi_2(X_t; f^*) -\mathbb{E}\left[\phi_2(X_t; \hat{f}_{t-1}) - \phi_2(X_t; f^*)| \Omega_{t-1}\right]\Bigg\}\Bigg)^2\Bigg].
\end{align*}
Therefore, we have
\begin{align*}
&=\frac{1}{T}\sum^T_{t=1}\mathbb{E}\Bigg[\Bigg(\phi_1(X_t, A_t, Y_t; \hat{g}_{t-1}, \hat{f}_{t-1}) - \phi_1(X_t, A_t, Y_t; \pi_{t-1}, f^*)\\
&\ \ \ \ \ \ \ -\mathbb{E}\left[\phi_1(X_t, A_t, Y_t; \hat{g}_{t-1}, \hat{f}_{t-1}) - \phi_1(X_t, A_t, Y_t; \pi_{t-1}, f^*)| \Omega_{t-1}\right]\\
&\ \ \ \ \ \ \ + \phi_2(X_t; \hat{f}_{t-1}) - \phi_2(X_t; f^*) -\mathbb{E}\left[\phi_2(X_t; \hat{f}_{t-1}) - \phi_2(X_t; f^*)| \Omega_{t-1}\right]\Bigg)^2\Bigg]\\
&\ \ \ +\frac{2}{T}\sum^{T-1}_{t=1}\sum^T_{s=t+1}\mathbb{E}\Bigg[\Bigg(\phi_1(X_t, A_t, Y_t; \hat{g}_{t-1}, \hat{f}_{t-1}) - \phi_1(X_t, A_t, Y_t; \pi_{t-1}, f^*)\\
&\ \ \ \ \ \ \ -\mathbb{E}\left[\phi_1(X_t, A_t, Y_t; \hat{g}_{t-1}, \hat{f}_{t-1}) - \phi_1(X_t, A_t, Y_t; \pi_{t-1}, f^*)| \Omega_{t-1}\right]\\
&\ \ \ \ \ \ \ + \phi_2(X_t; \hat{f}_{t-1}) - \phi_2(X_t; f^*) -\mathbb{E}\left[\phi_2(X_t; \hat{f}_{t-1}) - \phi_2(X_t; f^*)| \Omega_{t-1}\right]\Bigg)\\
&\ \ \ \ \ \ \ \times \Bigg(\phi_1(X_s, A_s, Y_s; \hat{g}_{s-1}, \hat{f}_{s-1}) - \phi_1(X_s, A_s, Y_s; \pi_{s-1}, f^*)\\
&\ \ \ \ \ \ \ -\mathbb{E}\left[\phi_1(X_s, A_s, Y_s; \hat{g}_{s-1}, \hat{f}_{s-1}) - \phi_1(X_s, A_s, Y_s; \pi_{s-1}, f^*)| \Omega_{s-1}\right]\\
&\ \ \ \ \ \ \ + \phi_2(X_s; \hat{f}_{s-1}) - \phi_2(X_s; f^*) -\mathbb{E}\left[\phi_2(X_s; \hat{f}_{s-1}) - \phi_2(X_s; f^*)| \Omega_{s-1}\right]\Bigg)\Bigg].
\end{align*}

For $s > t$, we can vanish the covariance terms as 
\begin{align*}
&\mathbb{E}\Bigg[\Bigg(\phi_1(X_t, A_t, Y_t; \hat{g}_{t-1}, \hat{f}_{t-1}) - \phi_1(X_t, A_t, Y_t; \pi_{t-1}, f^*)\\
&\ \ \ \ \ \ \ -\mathbb{E}\left[\phi_1(X_t, A_t, Y_t; \hat{g}_{t-1}, \hat{f}_{t-1}) - \phi_1(X_t, A_t, Y_t; \pi_{t-1}, f^*)| \Omega_{t-1}\right]\\
&\ \ \ \ \ \ \ + \phi_2(X_t; \hat{f}_{t-1}) - \phi_2(X_t; f^*) -\mathbb{E}\left[\phi_2(X_t; \hat{f}_{t-1}) - \phi_2(X_t; f^*)| \Omega_{t-1}\right]\Bigg)\\
&\ \ \ \ \ \ \ \times \Bigg(\phi_1(X_s, A_s, Y_s; \hat{g}_{s-1}, \hat{f}_{s-1}) - \phi_1(X_s, A_s, Y_s; \pi_{s-1}, f^*)\\
&\ \ \ \ \ \ \ -\mathbb{E}\left[\phi_1(X_s, A_s, Y_s; \hat{g}_{s-1}, \hat{f}_{s-1}) - \phi_1(X_s, A_s, Y_s; \pi_{s-1}, f^*)| \Omega_{s-1}\right]\\
&\ \ \ \ \ \ \ + \phi_2(X_s; \hat{f}_{s-1}) - \phi_2(X_s; f^*) -\mathbb{E}\left[\phi_2(X_s; \hat{f}_{s-1}) - \phi_2(X_s; f^*)| \Omega_{s-1}\right]\Bigg)\Bigg]\\
&=\mathbb{E}\Bigg[U\mathbb{E}\Bigg[\Bigg(\phi_1(X_s, A_s, Y_s; \hat{g}_{s-1}, \hat{f}_{s-1}) - \phi_1(X_s, A_s, Y_s; \pi_{s-1}, f^*)\\
&\ \ \ \ \ \ \ -\mathbb{E}\left[\phi_1(X_s, A_s, Y_s; \hat{g}_{s-1}, \hat{f}_{s-1}) - \phi_1(X_s, A_s, Y_s; \pi_{s-1}, f^*)| \Omega_{s-1}\right]\\
&\ \ \ \ \ \ \ + \phi_2(X_s; \hat{f}_{s-1}) - \phi_2(X_s; f^*) -\mathbb{E}\left[\phi_2(X_s; \hat{f}_{s-1}) - \phi_2(X_s; f^*)| \Omega_{s-1}\right]\Bigg)| \Omega_{s-1}\Bigg]\Bigg]\\
&=0,
\end{align*}
where 
\begin{align*}
U=\Bigg(&\phi_1(X_t, A_t, Y_t; \hat{g}_{t-1}, \hat{f}_{t-1}) - \phi_1(X_t, A_t, Y_t; \pi_{t-1}, f^*)\\
&-\mathbb{E}\left[\phi_1(X_t, A_t, Y_t; \hat{g}_{t-1}, \hat{f}_{t-1}) - \phi_1(X_t, A_t, Y_t; \pi_{t-1}, f^*)| \Omega_{t-1}\right]\\
& + \phi_2(X_t; \hat{f}_{t-1}) - \phi_2(X_t; f^*)\\
& -\mathbb{E}\left[\phi_2(X_t; \hat{f}_{t-1}) - \phi_2(X_t; f^*)| \Omega_{t-1}\right]\Bigg).
\end{align*}
Therefore, the variance is calculated as 
\begin{align*}
&\mathrm{Var}\Bigg(\sqrt{T}\frac{1}{T}\sum^T_{t=1}\Bigg\{\phi_1(X_t, A_t, Y_t; \hat{g}_{t-1}, \hat{f}_{t-1}) - \phi_1(X_t, A_t, Y_t; \pi_{t-1}, f^*)\\
&\ \ \ \ \ \ \ -\mathbb{E}\left[\phi_1(X_t, A_t, Y_t; \hat{g}_{t-1}, \hat{f}_{t-1}) - \phi_1(X_t, A_t, Y_t; \pi_{t-1}, f^*)| \Omega_{t-1}\right]\\
&\ \ \ \ \ \ \ + \phi_2(X_t; \hat{f}_{t-1}) - \phi_2(X_t; f^*)\\
&\ \ \ \ \ \ \ -\mathbb{E}\left[\phi_2(X_t; \hat{f}_{t-1}) - \phi_2(X_t; f^*)| \Omega_{t-1}\right]\Bigg\}\Bigg)\\
&=\frac{1}{T}\sum^T_{t=1}\mathbb{E}\Bigg[\Bigg(\phi_1(X_t, A_t, Y_t; \hat{g}_{t-1}, \hat{f}_{t-1}) - \phi_1(X_t, A_t, Y_t; \pi_{t-1}, f^*)\\
&\ \ \ \ \ \ \ -\mathbb{E}\left[\phi_1(X_t, A_t, Y_t; \hat{g}_{t-1}, \hat{f}_{t-1}) - \phi_1(X_t, A_t, Y_t; \pi_{t-1}, f^*)| \Omega_{t-1}\right]\\
&\ \ \ \ \ \ \ + \phi_2(X_t; \hat{f}_{t-1}) - \phi_2(X_t; f^*)\\
&\ \ \ \ \ \ \  -\mathbb{E}\left[\phi_2(X_t; \hat{f}_{t-1}) - \phi_2(X_t; f^*)| \Omega_{t-1}\right]\Bigg)^2\Bigg]\\
&=\frac{1}{T}\sum^T_{t=1}\mathbb{E}\Bigg[\mathbb{E}\Bigg[\Bigg(\phi_1(X_t, A_t, Y_t; \hat{g}_{t-1}, \hat{f}_{t-1}) - \phi_1(X_t, A_t, Y_t; \pi_{t-1}, f^*)\\
&\ \ \ \ \ \ \ -\mathbb{E}\left[\phi_1(X_t, A_t, Y_t; \hat{g}_{t-1}, \hat{f}_{t-1}) - \phi_1(X_t, A_t, Y_t; \pi_{t-1}, f^*)| \Omega_{t-1}\right]\\
&\ \ \ \ \ \ \ + \phi_2(X_t; \hat{f}_{t-1}) - \phi_2(X_t; f^*)\\
&\ \ \ \ \ \ \  -\mathbb{E}\left[\phi_2(X_t; \hat{f}_{t-1}) - \phi_2(X_t; f^*)| \Omega_{t-1}\right]\Bigg)^2| \Omega_{t-1}\Bigg]\Bigg]\\
&=\frac{1}{T}\sum^T_{t=1}\mathbb{E}\Bigg[\mathrm{Var}\Bigg(\phi_1(X_t, A_t, Y_t; \hat{g}_{t-1}, \hat{f}_{t-1}) - \phi_1(X_t, A_t, Y_t; \pi_{t-1}, f^*) + \phi_2(X_t; \hat{f}_{t-1}) - \phi_2(X_t; f^*) | \Omega_{t-1}\Bigg)\Bigg]\\
&=\frac{1}{T}\sum^T_{t=1}\mathbb{E}\Bigg[\mathrm{Var}\Bigg(\phi_1(X_t, A_t, Y_t; \hat{g}_{t-1}, \hat{f}_{t-1}) - \phi_1(X_t, A_t, Y_t; \pi_{t-1}, f^*) | \Omega_{t-1}\Bigg)\Bigg]\\
&\ \ \ +\frac{1}{T}\sum^T_{t=1}\mathbb{E}\Bigg[\mathrm{Var}\Bigg(\phi_2(X_t; \hat{f}_{t-1}) - \phi_2(X_t; f^*) | \Omega_{t-1}\Bigg)\Bigg]\\
&\ \ \ +\frac{2}{T}\sum^T_{t=1}\mathbb{E}\Bigg[\mathrm{Cov}\Bigg(\phi_1(X_t, A_t, Y_t; \hat{g}_{t-1}, \hat{f}_{t-1}) - \phi_1(X_t, A_t, Y_t; \pi_{t-1}, f^*), \phi_2(X_t; \hat{f}_{t-1}) - \phi_2(X_t; f^*) | \Omega_{t-1}\Bigg)\Bigg].
\end{align*}

Then, we want to show that
\begin{align}
\label{eq:first_e0}
&\frac{1}{T}\sum^T_{t=1}\mathbb{E}\Bigg[\mathrm{Var}\Bigg(\phi_1(X_t, A_t, Y_t; \hat{g}_{t-1}, \hat{f}_{t-1}) - \phi_1(X_t, A_t, Y_t; \pi_{t-1}, f^*) | \Omega_{t-1}\Bigg)\Bigg] \to 0,\\
\label{eq:second_e0}
&\frac{1}{T}\sum^T_{t=1}\mathbb{E}\Bigg[\mathrm{Var}\Bigg(\phi_2(X_t; \hat{f}_{t-1}) - \phi_2(X_t; f^*) | \Omega_{t-1}\Bigg)\Bigg]  \to 0,\\
\label{eq:third_e0}
&\frac{2}{T}\sum^T_{t=1}\mathbb{E}\Bigg[\mathrm{Cov}\Bigg(\phi_1(X_t, A_t, Y_t; \hat{g}_{t-1}, \hat{f}_{t-1}) - \phi_1(X_t, A_t, Y_t; \pi_{t-1}, f^*), \phi_2(X_t; \hat{f}_{t-1}) - \phi_2(X_t; f^*) | \Omega_{t-1}\Bigg)\Bigg] \to 0
\end{align}

For showing \eqref{eq:first_e0}--\eqref{eq:third_e0}, we consider showing
\begin{align}
\label{eq:first_e}
&\mathrm{Var}\Bigg(\phi_1(X_t, A_t, Y_t; \hat{g}_{t-1}, \hat{f}_{t-1}) - \phi_1(X_t, A_t, Y_t; \pi_{t-1}, f^*) | \Omega_{t-1}\Bigg) = \op(1),\\
\label{eq:second_e}
&\mathrm{Var}\Bigg(\phi_2(X_t; \hat{f}_{t-1}) - \phi_2(X_t; f^*) | \Omega_{t-1}\Bigg) = \op(1)\\
\label{eq:third_e}
&\mathrm{Cov}\Bigg(\phi_1(X_t, A_t, Y_t; \hat{g}_{t-1}, \hat{f}_{t-1}) - \phi_1(X_t, A_t, Y_t; \pi_{t-1}, f^*), \phi_2(X_t; \hat{f}_{t-1}) - \phi_2(X_t; f^*) | \Omega_{t-1}\Bigg) = \op(1),
\end{align}

The first equation \eqref{eq:first_e} is shown as 
\begin{align*}
&\mathrm{Var}\Bigg(\phi_1(X_t, A_t, Y_t; \hat{g}_{t-1}, \hat{f}_{t-1}) - \phi_1(X_t, A_t, Y_t; \pi_{t-1}, f^*) | \Omega_{t-1}\Bigg)\\
&\leq \mathbb{E}\left[\left\{\sum^K_{a=1}\frac{\epol(a| X_t)\mathbbm{1}[A_t=a]\left(Y_t - \hat{f}_{t-1}(a, X_t)\right) }{\hat{g}_{t-1}(a| X_t)} - \sum^K_{a=1}\frac{\epol(a| X_t)\mathbbm{1}[A_t=a]\left(Y_t - f^*(a, X_t)\right) }{\pi_{t-1}(a| X_t)}\right\}^2| \Omega_{t-1}\right]\\
&=\mathbb{E}\Bigg[\Bigg\{\sum^K_{a=1}\frac{\epol(a| X_t)\mathbbm{1}[A_t=a]\left(Y_t - \hat{f}_{t-1}(a, X_t)\right) }{\hat{g}_{t-1}(a| X_t)} - \sum^K_{a=1}\frac{\epol(a| X_t)\mathbbm{1}[A_t=a]\left(Y_t - f^*(a, X_t)\right) }{\hat{g}_{t-1}(a| X_t)}\\
&\ \ \ + \sum^K_{a=1}\frac{\epol(a| X_t)\mathbbm{1}[A_t=a]\left(Y_t - f^*(a, X_t)\right) }{\hat{g}_{t-1}(a| X_t)}- \sum^K_{a=1}\frac{\epol(a| X_t)\mathbbm{1}[A_t=a]\left(Y_t - f^*(a, X_t)\right) }{\pi_{t-1}(a| X_t)}\Bigg\}^2| \Omega_{t-1}\Bigg]\\
&\leq 2\mathbb{E}\Bigg[\Bigg\{\sum^K_{a=1}\frac{\epol(a| X_t)\mathbbm{1}[A_t=a]\left(Y_t - \hat{f}_{t-1}(a, X_t)\right) }{\hat{g}_{t-1}(a| X_t)}- \sum^K_{a=1}\frac{\epol(a| X_t)\mathbbm{1}[A_t=a]\left(Y_t - f^*(a, X_t)\right) }{\hat{g}_{t-1}(a| X_t)}| \Omega_{t-1}\Bigg]\\
&\ \ \ + 2\mathbb{E}\Bigg[\Bigg\{\sum^K_{a=1}\frac{\epol(a| X_t)\mathbbm{1}[A_t=a]\left(Y_t - f^*(a, X_t)\right) }{\hat{g}_{t-1}(a| X_t)} - \sum^K_{a=1}\frac{\epol(a| X_t)\mathbbm{1}[A_t=a]\left(Y_t - f^*(a, X_t)\right) }{\pi_{t-1}(a| X_t)}\Bigg\}^2| \Omega_{t-1}\Bigg]\\
&\leq 2C\|f^* - \hat{f}_{t-1}\|^2_2 + 2\times 4C\|\hat{g}_{t-1} - \pi_{t-1}\|^2_{2} = \op(1),
\end{align*}
where $C>0$ is a constant. Here, we have used a parallelogram law from the second line to the third line. We have used $|\hat{f}_{t-1}| < C_3$, and $0<\frac{\epol}{\pi_{t-1}} < C_4$, convergence of $\pi_{t-1}$ and convergence rate conditions from the third line to the fourth line. Then, from the $L^r$ convergence theorem (Proposition~\ref{prp:lr_conv_theorem}) and the boundedness of the random variables, we can show that as $t\to \infty$,
\begin{align*}
&\mathbb{E}\Bigg[\mathrm{Var}\Bigg(\phi_1(X_t, A_t, Y_t; \hat{g}_{t-1}, \hat{f}_{t-1}) - \phi_1(X_t, A_t, Y_t; \pi_{t-1}, f^*) | \Omega_{t-1}\Bigg)\Bigg]\\
&\leq \mathbb{E}\Bigg[\left|\mathrm{Var}\Bigg(\phi_1(X_t, A_t, Y_t; \hat{g}_{t-1}, \hat{f}_{t-1}) - \phi_1(X_t, A_t, Y_t; \pi_{t-1}, f^*)| \Omega_{t-1}\Bigg) \right|\Bigg] \to 0.
\end{align*}
Therefore, for any $\epsilon > 0$, there exists a constant $ C > 0$ such that 
\begin{align*}
&\frac{1}{T} \sum^{T}_{t=1}\mathbb{E}\Bigg[\mathrm{Var}\Bigg(\phi_1(X_t, A_t, Y_t; \hat{g}_{t-1}, \hat{f}_{t-1}) - \phi_1(X_t, A_t, Y_t; \pi_{t-1}, f^*) | \Omega_{t-1}\Bigg)\Bigg]\leq  C/T + \epsilon.
\end{align*}

The second equation \eqref{eq:second_e} is derived by Jensen's inequality, and we show \eqref{eq:second_e0} as well as \eqref{eq:first_e0} by using  $L^r$ convergence theorem.

Next, we show the third equation \eqref{eq:third_e} as 
\begin{align*}
&\mathrm{Cov}\Bigg(\phi_1(X_t, A_t, Y_t; \hat{g}_{t-1}, \hat{f}_{t-1}) - \phi_1(X_t, A_t, Y_t; \pi_{t-1}, f^*), \phi_2(X_t; \hat{f}_{t-1}) - \phi_2(X_t; f^*) | \Omega_{t-1}\Bigg)\\
&\leq \Bigg|\mathbb{E}\Bigg[ \left(\phi_1(X_t, A_t, Y_t; \hat{g}_{t-1}, \hat{f}_{t-1}) - \phi_1(X_t, A_t, Y_t; \pi_{t-1}, f^*) - \mathbb{E}\left[\phi_1(X_t, A_t, Y_t; \hat{g}_{t-1}, \hat{f}_{t-1}) - \phi_1(X_t, A_t, Y_t; \pi_{t-1}, f^*) | \Omega_{t-1}\right]\right)\\
&\ \ \ \ \ \ \ \ \ \ \ \ \ \ \ \ \ \ \ \ \ \ \ \ \ \ \ \ \ \ \ \ \ \ \ \ \ \ \ \ \ \ \ \ \ \ \ \ \ \ \ \ \ \ \ \ \ \ \times \left(\phi_2(X_t; \hat{f}_{t-1}) - \phi_2(X_t; f^*) - \mathbb{E}\left[\phi_2(X_t; \hat{f}_{t-1}) - \phi_2(X_t; f^*)\right]\right) | \Omega_{t-1}\Bigg]\Bigg|\\
&\leq \mathbb{E}\Bigg[\Bigg|\left(\phi_1(X_t, A_t, Y_t; \hat{g}_{t-1}, \hat{f}_{t-1}) - \phi_1(X_t, A_t, Y_t; \pi_{t-1}, f^*) - \mathbb{E}\left[\phi_1(X_t, A_t, Y_t; \hat{g}_{t-1}, \hat{f}_{t-1}) - \phi_1(X_t, A_t, Y_t; \pi_{t-1}, f^*) | \Omega_{t-1}\right]\right)\\
&\ \ \ \ \ \ \ \ \ \ \ \ \ \ \ \ \ \ \ \ \ \ \ \ \ \ \ \ \ \ \ \ \ \ \ \ \ \ \ \ \ \ \ \ \ \ \ \ \ \ \ \ \ \ \ \ \ \ \times \left(\phi_2(X_t; \hat{f}_{t-1}) - \phi_2(X_t; f^*) - \mathbb{E}\left[\phi_2(X_t; \hat{f}_{t-1}) - \phi_2(X_t; f^*)\right]\right)\Bigg| | \Omega_{t-1}\Bigg]\\
&\leq C\mathbb{E}\Bigg[ \Bigg| \phi_1(X_t, A_t, Y_t; \hat{g}_{t-1}, \hat{f}_{t-1}) - \phi_1(X_t, A_t, Y_t; \pi_{t-1}, f^*)\\
&\ \ \ \ \ \ \ \ \ \ \ \ \ \ \ \ \ \ \ \ \ \ \ \ \ \ \ \ \ \ \ \ \ \ \ \ \ \ \ \ \ \ \ \ \ \ \ \ \ \ \ \ \ \ \ \ \ \ - \mathbb{E}\left[\phi_1(X_t, A_t, Y_t; \hat{g}_{t-1}, \hat{f}_{t-1}) - \phi_1(X_t, A_t, Y_t; \pi_{t-1}, f^*) | \Omega_{t-1}\right] \Bigg| | \Omega_{t-1}\Bigg]\\
&=\op(1),
\end{align*}
where $C>0$ is a constant. From the second to third line, we used Jensen's inequality. From the fourth to fifth line, we used consistencies of $\hat{f}_{t-1}$ and $\hat{g}_{t-1}$, which imply that for all $X_t\in\mathcal{X}$,
\begin{align*}
&\phi_1(X_t, A_t, Y_t; \hat{g}_{t-1}, \hat{f}_{t-1}) - \phi_1(X_t, A_t, Y_t; \pi_{t-1}, f^*)\\
&=\sum^K_{a=1}\left(\frac{\epol(a| X_t)\mathbbm{1}[A_t=a]\left(Y_t - \hat{f}_{t-1}(a, X_t)\right) }{\hat{g}_{t-1}(a| X_t)} - \frac{\epol(a| X_t)\mathbbm{1}[A_t=a]\left(Y_t - f^*(a, X_t)\right) }{\pi_{t-1}(a| X_t)}\right)\\
&\leq \sum^K_{a=1}\left|\frac{\epol(a| X_t)\mathbbm{1}[A_t=a]\left(Y_t - \hat{f}_{t-1}(a, X_t)\right) }{\hat{g}_{t-1}(a| X_t)} - \frac{\epol(a| X_t)\mathbbm{1}[A_t=a]\left(Y_t - f^*(a, X_t)\right) }{\pi_{t-1}(a| X_t)}\right|\\
&\leq C\sum^K_{a=1}\left|\pi_{t-1}(a| X_t)\left(Y_t - \hat{f}_{t-1}(a, X_t)\right) - \hat{g}_{t-1}(a| X_t)\left(Y_t - f^*(a, X_t)\right)\right|\\
&\leq C\sum^K_{a=1}\Big|\pi_{t-1}(a| X_t) - \hat{g}_{t-1}(a| X_t)\Big|\\
&\ \ \  - C\sum^K_{a=1}\Big|\pi_{t-1}(a| X_t)\hat{f}_{t-1}(a, X_t) - \hat{g}_{t-1}(a| X_t)\hat{f}_{t-1}(a, X_t) + \hat{g}_{t-1}(a| X_t)\hat{f}_{t-1}(a, X_t) - \hat{g}_{t-1}(a| X_t)f^*(a, X_t)\Big|\\
&\leq C\sum^K_{a=1}\Big|\pi_{t-1}(a| X_t) - \hat{g}_{t-1}(a| X_t)\Big| - C\sum^K_{a=1}\Big|\hat{f}_{t-1}(a, X_t) - f^*(a, X_t)\Big| = \op(1),
\end{align*}
where $C > 0$ is a constant.

Thus, from \eqref{eq:first_e0}--\eqref{eq:third_e0}, the variance of the bias term converges to $0$. Then, from Chebyshev's inequality,
\begin{align*}
&\mathbb{P}\Bigg(\Bigg|\sqrt{T}\frac{1}{T}\sum^T_{t=1}\Bigg\{\phi_1(X_t, A_t, Y_t; \hat{g}_{t-1}, \hat{f}_{t-1}) - \phi_1(X_t, A_t, Y_t; \pi_{t-1}, f^*)\\
&\ \ \ -\mathbb{E}\left[\phi_1(X_t, A_t, Y_t; \hat{g}_{t-1}, \hat{f}_{t-1}) - \phi_1(X_t, A_t, Y_t; \pi_{t-1}, f^*)| \Omega_{t-1}\right]\\
&\ \ \ + \phi_2(X_t; \hat{f}_{t-1}) - \phi_2(X_t; f^*) -\mathbb{E}\left[\phi_2(X_t; \hat{f}_{t-1}) - \phi_2(X_t; f^*)| \Omega_{t-1}\right]\Bigg\}\Bigg| > \varepsilon\Bigg)\\
&\leq \mathrm{Var}\Bigg(\sqrt{T}\frac{1}{T}\sum^T_{t=1}\Bigg\{\phi_1(X_t, A_t, Y_t; \hat{g}_{t-1}, \hat{f}_{t-1}) - \phi_1(X_t, A_t, Y_t; \pi_{t-1}, f^*)\\
&\ \ \ -\mathbb{E}\left[\phi_1(X_t, A_t, Y_t; \hat{g}_{t-1}, \hat{f}_{t-1}) - \phi_1(X_t, A_t, Y_t; \pi_{t-1}, f^*)| \Omega_{t-1}\right]\\
&\ \ \ + \phi_2(X_t; \hat{f}_{t-1}) - \phi_2(X_t; f^*) -\mathbb{E}\left[\phi_2(X_t; \hat{f}_{t-1}) - \phi_2(X_t; f^*)| \Omega_{t-1}\right]\Bigg\}\Bigg)/\varepsilon^2\\
&\to 0.
\end{align*}

\paragraph{Proof of \eqref{eq:part2}.}
\begin{align}
&\frac{1}{T}\sum^T_{t=1}\mathbb{E}\left[\phi_1(X_t, A_t, Y_t; \hat{g}_{t-1}, \hat{f}_{t-1})| \Omega_{t-1}\right] + \frac{1}{T}\sum^T_{t=1}\mathbb{E}\left[\phi_2(X_t; \hat{f}_{t-1})| \Omega_{t-1}\right]\nonumber\\
&\ \ \ - \frac{1}{T}\sum^T_{t=1}\mathbb{E}\left[\phi_1(X_t, A_t, Y_t; \pi_{t-1}, f^*)| \Omega_{t-1}\right] - \frac{1}{T}\sum^T_{t=1}\mathbb{E}\left[\phi_2(X_t; f^*)| \Omega_{t-1}\right]\nonumber\\
& = \frac{1}{T}\sum^T_{t=1} \mathbb{E}\left[\sum^K_{a=1}\frac{\epol(a| X_t)\mathbbm{1}[A_t=a]\left(Y_t - \hat{f}_{t-1}(a, X_t)\right) }{\hat{g}_{t-1}(a| X_t)}| \Omega_{t-1}\right]\nonumber\\
&\ \ \ + \frac{1}{T}\sum^T_{t=1}\mathbb{E}\left[\sum^K_{a=1}\epol(a, X_t)\hat{f}_{t-1}(a, X_t)| \Omega_{t-1}\right]\nonumber\\
\label{eq:vanish1}
&\ \ \ - \frac{1}{T}\sum^T_{t=1}\mathbb{E}\left[\sum^K_{a=1}\frac{\epol(a| X_t)\mathbbm{1}[A_t=a]\left(Y_t - f^*(a, X_t)\right) }{\pi_{t-1}(a| X_t, \Omega_{t-1})}| \Omega_{t-1}\right]\\
&\ \ \ - \frac{1}{T}\sum^T_{t=1}\mathbb{E}\left[\sum^K_{a=1}\epol(a, X_t)f^*(a, X_t)| \Omega_{t-1}\right]\nonumber.
\end{align}
Because \eqref{eq:vanish1} is $0$,
\begin{align*}
& = \frac{1}{T}\sum^T_{t=1} \mathbb{E}\left[\sum^K_{a=1}\frac{\epol(a| X_t)\mathbbm{1}[A_t=a]\left(Y_t - \hat{f}_{t-1}(a, X_t)\right) }{\hat{g}_{t-1}(a| X_t)}| \Omega_{t-1}\right]\\
&\ \ \ + \frac{1}{T}\sum^T_{t=1}\mathbb{E}\left[\sum^K_{a=1}\epol(a, X_t)\hat{f}_{t-1}(a, X_t)| \Omega_{t-1}\right] - \frac{1}{T}\sum^T_{t=1}\mathbb{E}\left[\sum^K_{a=1}\epol(a, X_t)f^*(a, X_t)| \Omega_{t-1}\right]\\
& = \frac{1}{T}\sum^T_{t=1} \mathbb{E}\left[\sum^K_{a=1}\frac{\epol(a| X_t)\mathbbm{1}[A_t=a]\left(Y_t - \hat{f}_{t-1}(a, X_t)\right) }{\hat{g}_{t-1}(a| X_t)}| \Omega_{t-1}\right]\\
&\ \ \ - \frac{1}{T}\sum^T_{t=1}\mathbb{E}\left[\sum^K_{a=1}\epol(a, X_t)\Big( f^*(a, X_t)- \hat{f}_{t-1}(a, X_t))\Big) | \Omega_{t-1}\right]\\
& = \frac{1}{T}\sum^T_{t=1} \sum^K_{a=1}\mathbb{E}\Bigg[\mathbb{E}\Bigg[\frac{\epol(a| X_t)\pi_{t-1}(a| X_t, \Omega_{t-1})\left(f^*(a, X_t) - \hat{f}_{t-1}(a, X_t)\right) }{\hat{g}_{t-1}(a| X_t)}\\
&\ \ \ \ \ \ \ \ \ \ \ \ \ \ \ \ \ \ \ \ \ \ \ \  - \epol(a, X_t)\Big( f^*(a, X_t)- \hat{f}_{t-1}(a, X_t))\Big) | X_t, \Omega_{t-1}\Bigg]| \Omega_{t-1}\Bigg]\\
& \leq \frac{1}{T}\sum^T_{t=1} \sum^K_{a=1}\Bigg|\mathbb{E}\Bigg[ \frac{\epol(a| X_t)\Big(\pi_{t-1}(a| X_t) - \hat{g}_{t-1}(a| X_t)\Big)\left(f^*(a, X_t) - \hat{f}_{t-1}(a, X_t)\right) }{\hat{g}_{t-1}(a| X_t)}  | \Omega_{t-1}\Bigg]\Bigg|. 
\end{align*}
By using \Holder's inequality $\|fg \|_1 \leq  \|f \|_2  \|g \|_2$, for a constant $C>0$, we have
\begin{align*}
& \leq \frac{C}{T}\sum^T_{t=1} \Bigg\| \pi_{t-1}(a| X_t, \Omega_{t-1}) - \hat{g}_{t-1}(a| X_t)\Bigg\|_{2}\Bigg\|f^*(a, X_t) - \hat{f}_{t-1}(a, X_t) \Bigg\|_{2}\\
& = \frac{C}{T}\sum^T_{t=1}\op(t^{-p})\op(t^{-q})\\
& = \frac{C}{T}\sum^T_{t=1}\op(t^{-1/2}).
\end{align*}

\end{proof}

\section{Adaptive-fitting and batched samples}
\label{appdx:batch}
Section~\ref{appdx:det_af} supplements the description of adaptive-fitting. Next, in Section~\ref{appdx:batch_true}, we introduce the AIPW estimator when the samples are given in batch form, and the true logging policy is given $\pi_t$. This is essentially the same as the generalized method of moment (GMM), which gives an asymptotically normal estimator for martigaale difference sequences (MDS), and we do not use adaptive-fitting. Based on this estimator, in Section~\ref{appdx:batch_est}, we introduce the ADR estimator when the data is given in batch form, but the true logging policy $\pi_t$ is not given.

\subsection{Details of adaptive-fitting}
\label{appdx:det_af}
As Section~\ref{sec:relation_ddm}, let us define the parameter of interest $\theta_0$ that  satisfies $\mathbb{E}[\psi(W_t; \theta_0, \eta_0)] = 0$, where $\{W_t\}^T_{t=1}$ are observations, $\eta_0$ is a nuisance parameter, and $\psi$ is a score function. Let us define two estimators $\hat{\theta}_T$ and $\check{\theta}_T$ as $\frac{1}{T}\sum^T_{t=1}\psi(W_t, \hat{\theta}_T, \eta_{t-1}) = 0$ and $\frac{1}{T}\sum^T_{t=1}\psi(W_t, \check{\theta}_T, \eta_{0}) = 0$. Suppose that $\check{\theta}_T$ is an asymptotically normal estimator of $\theta_0$. Then, if $\check{\theta}_T -\hat{\theta}_T$ converges to $0$ with convergence rate $\mathrm{o}_p(1/\sqrt{T})$, $\hat{\theta}_T$ is also an asymptotically normal estimator. In general, we cannot obtain such a fast convergence rate. However, by using double robustness, we can obtain the convergence rate.  In ADR estimator, this conditions appears as $\|\hat{g}_{t-1}(a| X_t) - \pi_{t-1}(a| X_t, \Omega_{t-1})\|_{2}=\op(t^{-p})$, and $\|\hat{f}_{t-1}(a,X_t)-f^*(a,X_t)\|_2=\op(t^{-q})$, where $p, q > 0$ such that $p+q = 1/2$, and the expectation of the norm is taken over $X_t$. This allows us to obtain $\mathrm{o}_p(1/\sqrt{T})$ of the asympttoic bias. The image of the vanishing asymptotic bias is shown in Figure~\ref{fig:af_bias}.

\begin{figure}[ht]
\begin{center}
 \includegraphics[width=100mm]{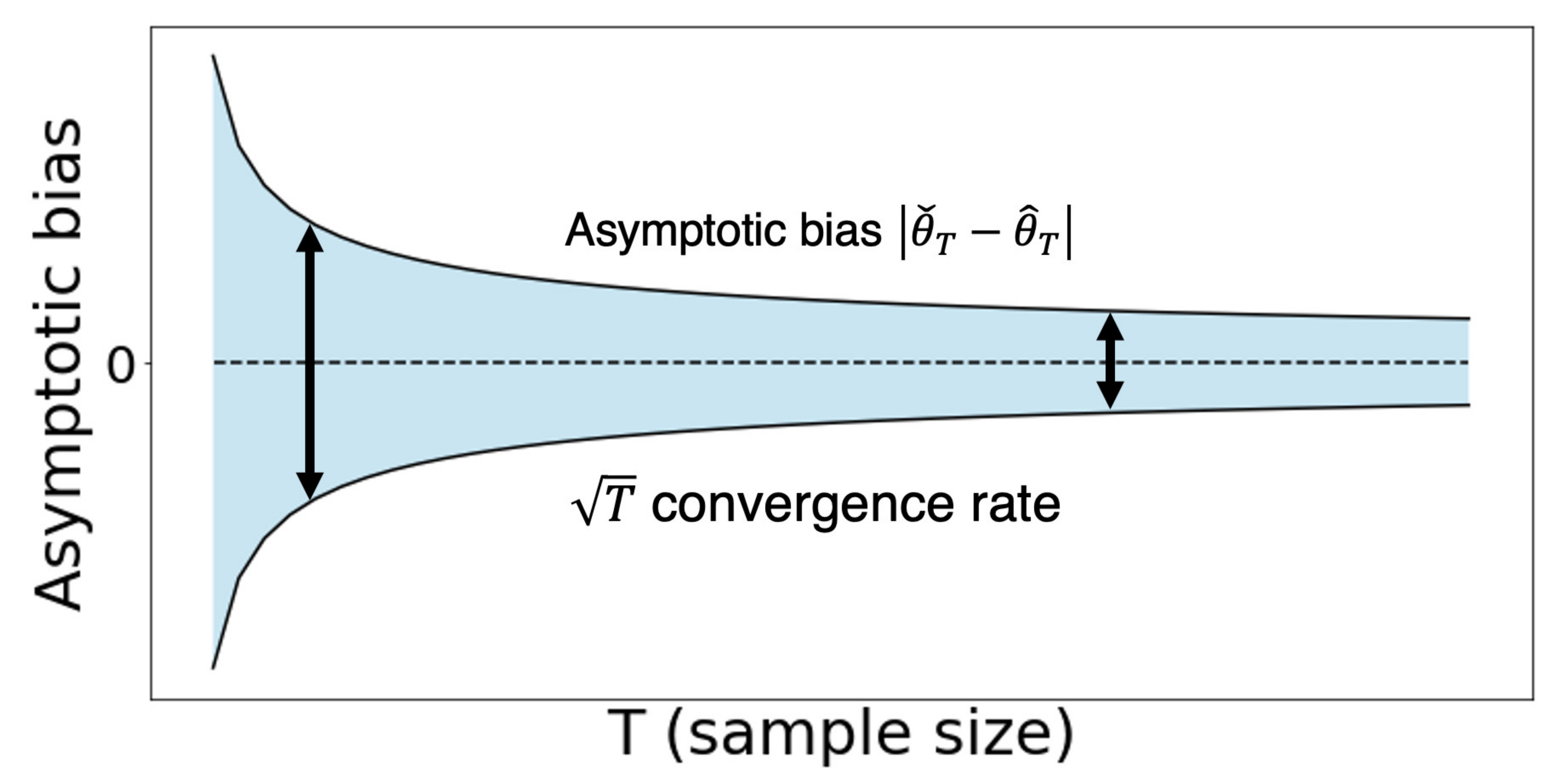}
\end{center}
\caption{Convergence of asymptotic bias $|\check{\theta}_T - \hat{\theta}_T|$.}
\label{fig:af_bias}
\end{figure} 

\subsection{AIPW estimator with batched samples when the true logging policy is known}
\label{appdx:batch_true}
The proposed adaptive-fitting can be applied when batched samples are given. Let $M$ denote the number of batch updates and $\tau\in I = \{1,2,\dots,M\}$ denotes the batch index. For $\tau\in I$, the probability is updated at a period $t_\tau$, where $t_\tau - t_{\tau-1} = Tr_\tau$, using samples $\{(X_t, Y_t, A_t)\}^{t_{\tau}}_{t=t_{\tau-1}}$, where $r_1 + r_2 + \cdots + r_M = 1$ and $t_0 = 0$. Thus, in addition to the DGP \eqref{eq:DGP}, we assume that 
\begin{align*}
\big\{(X_t, A_t, Y_t)\big\}^{t_\tau}_{t=t_{\tau-1}}\iid p(x)\pi_\tau(a| x, \Omega_{t_{\tau-1}})p(y_a| x),
\end{align*}
where $\pi_\tau(a| x, \Omega_{t_{\tau-1}})$ denotes the probability of choosing an action updated based on samples until the period $t_{\tau-1}$.

We consider asymptotic properties based on the assumption of $t_\tau - t_{\tau-1} \to \infty$ as $T\to \infty$ for fixed $\tau$. This strategy is the same as \citet{kelly2020batched}. Because $\big\{(X_t, A_t, Y_t)\big\}^{t_\tau}_{t=t_{\tau-1}}$ is i.i.d., we can use the standard limit theorems for the partial sum of the samples to obtain an asymptotically normal estimator of $\theta_0 = R(\epol)$. However, we also have the motivation to use all samples together to increase the efficiency of the estimator. Therefore, by using the GMM, we propose an estimator of $\theta_0$ considering the sample averages of each block as an empirical moment conditions. Although we cannot use standard CLT, we can can apply the CLT for the MDS by appropriately constructing an estimator. Note that the original GMM is proved when the MDS is applicable, which also includes the method of \citet{kelly2020batched} as a special case.

For an index of batch $\tau\in I$, a function $f\in\mathcal{F}$ such that $f:\mathcal{A}\times \mathcal{X} \to \mathbb{R}$ and an evaluation policy $\epol \in \Pi$, we define a function $h_t$ as 
\begin{align*}
&h_t(x, k, y; \tau, R, f, \pi_\tau, \epol) = \frac{1}{r_\tau}\xi_t(x, k, y; \tau, R, f, \pi_\tau, \epol)\mathbbm{1}\big[t_{\tau - 1} < t \leq t_{\tau} \big],\nonumber
\end{align*}
where $R\in\mathbb{R}$, $\xi_t(x, k, y; \tau, R, f, \pi_\tau, \epol) := \phi_t(x, k, y; \tau, f, \pi_\tau, \epol) - R$ and
\begin{align*}
&\phi_t(x, k, y; \tau, f, \pi_\tau, \epol):=\sum^{K}_{a=1}\epol(a| x)\Bigg\{\frac{\mathbbm{1}[k=a]\big\{y - f(a, x)\big\}}{\pi_{\tau}(a| x, \Omega_{t_{\tau - 1}})}+f(a, x)\Bigg\}.
\end{align*}
Let us note that the sequence $\big\{h_t(X_t, A_t, Y_t; \tau, R(\epol), \hat{f}_{t-1}, \pi_\tau, \epol)\big\}^{T}_{t=1}$ is an MDS: for $h_t(X_t, A_t, Y_t; \tau, R(\epol), \hat{f}_{t-1}, \pi_\tau, \epol)$, by using $\mathbb{E}[\mathbbm{1}[A_t=a]| \mathbb{H}_{t-1}]=\pi_{\tau}(a| X_t, \Omega_{t_{\tau - 1}})$, we have
\begin{align*}
&\mathbb{E}\left[h_t(X_t, A_t, Y_t; \tau, R(\epol), \hat{f}_{t-1}, \pi_\tau, \epol)| \Omega_{t-1}\right]\\
& = \mathbb{E}\left[\frac{\mathbbm{1}\big[t_{\tau - 1} < t \leq t_{\tau} \big]}{r_\tau}\xi_t(x, k, y; \tau, R(\epol), \hat{f}_{t-1}, \pi_\tau, \epol)| \Omega_{t-1}\right]\\
& = \frac{\mathbbm{1}\big[t_{\tau - 1} < t \leq t_{\tau} \big]}{r_\tau}\mathbb{E}\left[\xi_t(x, k, y; \tau, R(\epol), \hat{f}_{t-1}, \pi_\tau,  \epol)| \Omega_{t-1}\right]\\
& = \frac{\mathbbm{1}\big[t_{\tau - 1} < t \leq t_{\tau} \big]}{r_\tau}\times 0 = 0.
\end{align*}
Let us also define
\begin{align*}
\bm{h}_t\left(X_t, A_t, Y_t; R, \hat{f}_{t-1}, \pi_\tau, \epol\right):=
\begin{pmatrix} 
h_t(X_t, A_t, Y_t; 1, R, \hat{f}_{t-1},  \pi_1, \epol) \\
h_t(X_t, A_t, Y_t; 2, R, \hat{f}_{t-1},  \pi_2, \epol) \\
 \vdots \\ 
h_t(X_t, A_t, Y_t; M, R, \hat{f}_{t-1},  \pi_M, \epol)
\end{pmatrix}.
\end{align*}
Then, the sequence $\left\{\bm{h}_t\left(X_t, A_t, Y_t; R(\epol), \hat{f}_{t-1},  \pi_\tau, \epol\right)\right\}^{T}_{t=1}$ is an MDS with respect to $\big\{\Omega_t\big\}^{T-1}_{t=0}$, i.e.,
\begin{align*}
&\mathbb{E}\left[\bm{h}_t\left(X_t, A_t, Y_t; R(\epol), \hat{f}_{t-1}, \pi_\tau,\epol\right)| \Omega_{t-1}\right] = \bm{0}. 
\end{align*}

Using the sequence $\left\{\bm{h}_t\left(X_t, A_t, Y_t; R, \hat{f}_{t-1}, \pi_\tau, \epol\right)\right\}^{T}_{t=1}$, we define an estimator of $R(\epol)$ as
\begin{align}
\label{gmmdm_ope}
\hat{R}^{\mathrm{batch}}_T(\epol) = \argmin_{R\in\mathbb{R}} \left(\hat{\bm{q}}_T(R)\right)^\top \hat{W}_T \left(\hat{\bm{q}}_T(R)\right),
\end{align}
where $\hat{\bm{q}}_T(R) = \frac{1}{T}\sum^{T}_{t=1}\bm{h}_t\left(X_t, A_t, Y_t; R, \hat{f}_{t-1}, \pi_\tau, \epol\right)$ and $\hat{W}_T$ is a data-dependent $(M\times M)$-dimensional positive semi-definite matrix. Let us note that the estimator defined in Eq.~\eqref{gmmdm_ope} is an application of GMM with the moment condition
\begin{align*}
&\bm{q}(R(\epol)) = \mathbb{E}\left[\frac{1}{T}\sum^T_{t=1}\bm{h}_t\left(X_t, A_t, Y_t; R(\epol), \hat{f}_{t-1}, \pi_\tau,\epol\right)\right]=0.
\end{align*}
For the minimization problem defined in Eq.~\eqref{gmmdm_ope}, we can analytically calculate the minimizer as
\begin{align*}
\hat{R}^{\mathrm{batch}}_T(\epol)  = w^\top_T D_T(\epol),
\end{align*}
where $w_T = (w_{T, 1}\ \cdots\ w_{T, M})^\top$ is an $M$-dimensional vector such that $\sum^M_{\tau=1}w_{T, \tau}=1$, and 
\begin{align*}D_T(\epol) = 
\begin{pmatrix}
\frac{1}{t_{1}}\sum^{t_{1}}_{t=1} \phi_t(X_t, A_t, Y_t; 1, f_{t-1},  \pi_1, \epol) \\
\frac{1}{t_{2} - {t_{1}}}\sum^{t_2}_{t=t_1+1} \phi_t(X_t, A_t, Y_t; 2, f_{t-1}, \pi_2, \epol)\\
\vdots\\
\frac{1}{T-t_{M-1}}\sum^{T}_{t=t_{M-1}+1} \phi_t(X_t, A_t, Y_t; M, f_{t-1}, \pi_M, \epol)\\
\end{pmatrix}.
\end{align*}

Here, we show the asymptotic normality of the proposed estimator $\hat \theta^{\mathrm{OPE}}_T$. 

\begin{theorem}[Asymptotic distribution the proposed estimator]
\label{thm:asymp_batch}
Suppose that 
\begin{description}
\item[(i)] $w_T = (w_{T,1}\ \cdots\ w_{T,M})^\top \xrightarrow{\mathrm{p}} w = (w_{1}\ \cdots\ w_{M})^\top$;
\item[(ii)] $w_{T,\tau} > 0$ and $\sum^M_{\tau=1}w_{T,\tau} = 1$.
\end{description}
Then, under Assumptions~\ref{asm:overlap_pol}, \ref{asm:bounded_reward},  \ref{asm:bound_nuisance1}, \ref{asm:pointwise_f}, 
\begin{align*}
\sqrt{T}\big(\hat{R}^{\mathrm{batch}}_T(\epol) - R(\epol)\big)  \xrightarrow{\mathrm{d}}\mathcal{N}\big(0, \sigma^2\big),
\end{align*}
where $\sigma^2 = \sum^M_{\tau=1}w_{\tau}\Psi(\epol, \pi_\tau)$.
\end{theorem}
In the proposed method, we construct a moment condition using martingale difference sequences. On the other hand, for some readers, using martingale difference sequences may look unnecessary because samples are i.i.d in each block between $t_{\tau-1}$ and $t_{\tau}$. Therefore, such readers also might feel that we can use $f_T(a, x)$, which is an estimator of $\mathbb{E}[Y(a)\mid x]$ using samples until $T$-th period, without going through constructing several estimators $\{f_{t_\tau}\}^{M-1}_{\tau=0}$. However, in that case, it is difficult to guarantee the asymptotic normality of the proposed estimator. For example, we can consider Cram\'{e}r-Wold theorem to consider this problem. This motivations shares with the GMM and the method of \citet{kelly2020batched}.

The proof of Theorem~\ref{thm:asymp_batch} is shown as follows.

\paragraph{Choice of weight $w_T$.} We discuss the choice of weight $w_T$. A naive choice is weighting the moment conditions equality; that is, $w_{\tau, T} = \frac{1}{M}$. Next, we consider an efficient weight $w_T$ that minimizes the asymptotic variance of $\hat{R}^{\mathrm{batch}}_T(\epol)$. In the GMM, the $\tau$-th element of the efficient weight is given as $w^*_{\tau}=\frac{1}{\Psi(\epol,\pi_\tau)}/\sum^M_{\tau'=1}\frac{1}{\Psi(\epol,\pi_{\tau'})}$ \citep{GVK126800421}. Here, we use the orthogonality among moment conditions; that is, zero covariance. In this case, the asymptotic variance becomes $1/\sum^M_{\tau'=1}\frac{1}{\sigma^2_{\tau'}}$. Therefore, for gaining efficiency, we use a weight $\hat{w}_{T,\tau} = \frac{1}{\hat{\Psi}_T(\epol,\pi_\tau)}/\sum^M_{\tau'=1}\frac{1}{\hat{\Psi}_T(\epol,\pi_{\tau'})}$, where $\hat{\Psi}_T(\epol,\pi_\tau)$ is an estimator of $\Psi(\epol,\pi_\tau)$. 

\paragraph{Proof of Theorem~\ref{thm:asymp_batch}.}  Instead of $\hat{R}^{\mathrm{batch}}_T(\epol) = w_T D_T(\epol)$, from the original formulation Eq.~\eqref{gmmdm_ope}, we consider an estimator $\hat{R}^{\mathrm{batch}}_T(\epol) = \big(I^\top \hat W_T I\big)^{-1}I^\top \hat W_T D_T(\epol)$, where $W_T$ is a $(M\times M)$-dimensional positive-definite matrix. Let us note that $w_T = \big(I^\top \hat W_T I\big)^{-1}I^\top \hat W_T$. We prove the following theorem, which is a generalized statement of Theorem~\ref{thm:asymp_batch}.

\begin{theorem}[Asymptotic distribution the AIPW estimator under batch update]
\label{thm:main}
Suppose that 
\begin{description}
\item[(i)] $\hat W_T \xrightarrow{\mathrm{p}} W$;
\item[(ii)] $W$ is a positive definite;
\end{description}
Then, under Assumptions~\ref{asm:overlap_pol}, \ref{asm:bounded_reward},  \ref{asm:bound_nuisance1}, \ref{asm:pointwise_f}, 
\begin{align*}
\sqrt{T}\big(\hat{R}^{\mathrm{batch}}_T(\epol) - R(\epol)\big)  \xrightarrow{\mathrm{d}}\mathcal{N}\big(0, \sigma^2\big),
\end{align*}
where $\sigma^2 = \big(I^\top W I \big)^{-1} I^\top W \Sigma W^\top I \big(I^\top W I \big)^{-1}$ and $\Sigma$ is a $(M\times M)$ diagonal matrix such that the $(\tau\times \tau)$-element is $\Psi(\epol,\pi_\tau)$.
\end{theorem}

\begin{proof}
Let us define $\hat{\bm{q}}_T(R(\epol)) = (\hat{q}_{1, T}(R(\epol))\ \hat{q}_{2, T}(R(\epol))\ \cdots\ \hat{q}_{M, T}(R(\epol)))$, where  
\begin{align*}
&\hat{q}_{\tau, T}(R(\epol)) = \\
&\frac{1}{T}\frac{1}{r_1}\sum^{T}_{t=1} \left(\sum^{K}_{a=1}\left\{\frac{\epol(a\mid X_t)\mathbbm{1}[A_t=a]\big\{Y_t - \hat{f}_{t-1}(a, X_t)\big\}}{\pi_{1}(a\mid x, \Omega_{0})} + \epol(a\mid X_t)\hat{f}_{t-1}(a, X_t)\right\} - R(\epol)\right)\\
&\ \ \ \ \ \ \ \ \ \ \ \ \ \ \ \ \ \ \ \ \ \ \ \ \ \ \ \ \ \ \ \ \ \ \ \ \ \ \ \ \ \ \ \ \ \ \ \ \ \ \ \ \ \ \ \ \ \ \ \ \ \ \ \ \times \mathbbm{1}\big[t_0=0 < t \leq t_{1} \big].
\end{align*}
For $\sqrt{T}\big(\hat{R}^{\mathrm{batch}}_T(\epol) - R(\epol)\big)  = \big(I^\top \hat W_T I\big)^{-1}I^\top \hat W_T \sqrt{T} \hat{\bm{q}}_T(R(\epol))$,
we show that 
\begin{align*}
\sqrt{T} \hat{\bm{q}}_T(R(\epol)) \xrightarrow{\mathrm{d}} \mathcal{N}\big(0, \Sigma\big),
\end{align*}
where $\Sigma$ is a diagonal matrix such that the $(\tau,\tau)$-element is
\begin{align*}
&\frac{1}{r_\tau}\mathbb{E}\left[\sum^{K}_{a=1}\frac{\big(\epol(a\mid X)\big)^2\mathrm{Var}(Y(a)\mid X)}{\pi_{\tau}(a\mid X, \Omega_{t_{\tau - 1}})} + \left(\sum^{K}_{a=1}\epol(a\mid X)\mathbb{E}\big[Y(a)\mid X\big] - R(\epol)\right)^2 \right].
\end{align*}
Then, from Slutsky Theorem, we can show that 
\begin{align*}
\big(I^\top \hat W_T I\big)^{-1}I^\top \hat W_T \sqrt{T} \hat{\bm{q}}_T(R(\epol)) \xrightarrow{\mathrm{d}}\mathcal{N}\big(0, \big(I^\top W I \big)^{-1} I^\top W \Sigma W^\top I \big(I^\top W I \big)^{-1}\big).
\end{align*}

To show this result, we use the CLT for MDS by checking the following conditions:
\begin{description}
\item[(a)] $(1/T)\sum^T_{t=1}\Sigma_t\to\Sigma$, where 
$$\Sigma_t = \mathbb{E}\left[\left(\bm{h}_t\left(X_t, A_t, Y_t; R(\epol), f_{t-1}, \pi_\tau,\epol\right)\right)\left(\bm{h}_t\left(X_t, A_t, Y_t; R(\epol), f_{t-1}, \pi_\tau,\epol\right)\right)^{\top}\right];$$
\item[(b)] $\mathbb{E}\big[\tilde{h}_t(i, R(\epol), f_{t-1}, \epol)\tilde{h}_t(j, R(\epol), f_{t-1}, \epol)\tilde{h}_t(a, R(\epol), f_{t-1}, \epol)\tilde{h}_t(l, R(\epol), f_{t-1}, \epol)\big] < \infty$ for $i,j,k,l\in I$, where $\tilde{h}_t(a, R(\epol), f_{t-1}, \epol)=h^\mathrm{OPE}_t(X_t, A_t, Y_t; k, R(\epol), f_{k}, \epol)$ for $k\in I$;
\item[(c)] $\frac{1}{T}\sum^{T}_{t=1}\left(\bm{h}_t\left(X_t, A_t, Y_t; R(\epol), f_{t-1}, \pi_\tau,\epol\right)\right)\left(\bm{h}_t\left(X_t, A_t, Y_t; R(\epol), f_{t-1}, \pi_\tau,\epol\right)\right)^{\top}\xrightarrow{p}\Sigma$,
\end{description}
Note that the GMM shows the asymptotic normality for more general cases.

\subsection*{Step~1: Condition~(a)}
From 
\begin{align*}
&\Sigma_t = \mathbb{E}\left[\left(\bm{h}_t\left(X_t, A_t, Y_t; R(\epol), f_{t-1}, \pi_\tau,\epol\right)\right)\left(\bm{h}_t\left(X_t, A_t, Y_t; R(\epol), f_{t-1}, \pi_\tau,\epol\right)\right)^{\top}\right],
\end{align*}
the matrix $(1/T)\sum^T_{t=1}\Omega_t$ becomes a diagonal matrix such that the $(\tau, \tau)$-element is 
\begin{align*}
&\frac{1}{r^2_\tau T}\sum^T_{t=1}\mathbb{E}\Bigg[\left(\sum^{K}_{a=1}\left\{\frac{\epol(a\mid X_t)\mathbbm{1}[A_t=a]\big\{Y_t - f_{t- 1}(a, X_t)\big\}}{\pi_{\tau}(a\mid X_t, \Omega_{t_{\tau - 1}})} + \epol(a\mid X_t)f_{t-1}(a, X_t)\right\} - R(\epol)\right)^2\\
&\ \ \ \ \ \ \ \ \ \ \ \ \ \ \ \ \ \ \ \ \ \ \ \ \ \ \ \ \ \ \ \ \ \ \ \ \ \ \ \ \ \ \ \ \ \ \ \ \ \ \ \ \ \ \ \ \ \ \ \ \ \ \ \ \ \ \ \ \ \ \ \ \ \ \ \ \ \ \ \ \times \mathbbm{1}\big[t_{\tau - 1} < t \leq t_{\tau} \big]\Bigg].
\end{align*}

For $\tau \in I$ and $t$ such that $t_{\tau - 1} < t \leq t_{\tau}$,
\begin{align*}
&\mathbb{E}\left[\left(\sum^{K}_{a=1}\left\{\frac{\epol(a\mid X_t)\mathbbm{1}[A_t=a]\big\{Y_t - f_{t-1}(a, X_t)\big\}}{\pi_{\tau}(a\mid X_t, \Omega_{t_{\tau - 1}})} + \epol(a\mid X_t)f_{t-1}(a, X_t)\right\} - R(\epol)\right)^2\right]\\
&- \mathbb{E}\left[\left(\sum^{K}_{a=1}\left\{\frac{\epol(a\mid X_t)\mathbbm{1}[A_t=a]\big\{Y_t - \mathbb{E}\big[Y_t(a)\mid X_t\big]\big\}}{\pi_{\tau}(a\mid X_t, \Omega_{t_{\tau - 1}})} + \epol(a\mid X_t)\mathbb{E}\big[Y(a)\mid X_t\big]\right\} - R(\epol)\right)^2\right]\\
&\leq \mathbb{E}\Bigg[\Bigg|\left(\sum^{K}_{a=1}\left\{\frac{\epol(a\mid X_t)\mathbbm{1}[A_t=a]\big\{Y_t - f_{t-1}(a, X_t)\big\}}{\pi_{\tau}(a\mid X_t, \Omega_{t_{\tau - 1}})} + \epol(a\mid X_t)f_{t-1}(a, X_t)\right\} - R(\epol)\right)^2\\
&\ \ \ \ \ \ \ \ - \left(\sum^{K}_{a=1}\left\{\frac{\epol(a\mid X_t)\mathbbm{1}[A_t=a]\big\{Y_t - \mathbb{E}\big[Y_t(a)\mid X_t\big]\big\}}{\pi_{\tau}(a\mid X_t, \Omega_{t_{\tau - 1}})} + \epol(a\mid X_t)\mathbb{E}\big[Y(a)\mid X_t\big]\right\} - R(\epol)\right)^2\Bigg|\Bigg]
\end{align*}
Because $\alpha^2 - \beta^2 = (\alpha + \beta)(\alpha - \beta)$, there exists a constant $\gamma_0 > 0$ such that
\begin{align*}
&\leq \gamma_0\mathbb{E}\Bigg[\Bigg|\sum^{K}_{a=1}\Bigg\{\frac{\epol(a\mid X_t)\mathbbm{1}[A_t=a]\big\{Y_t - f_{t-1}(a, X_t)\big\}}{\pi_{\tau}(a\mid X_t, \Omega_{t_{\tau - 1}})} + \epol(a\mid X_t)f_{t-1}(a, X_t)\\
&\ \ \ \ \ \ \ \ \ \ \ \ \ \ - \frac{\epol(a\mid X_t)\mathbbm{1}[A_t=a]\big\{Y_t - \mathbb{E}\big[Y_t(a)\mid X_t\big]\big\}}{\pi_{\tau}(a\mid X_t, \Omega_{t_{\tau - 1}})} - \epol(a\mid X_t)\mathbb{E}\big[Y(a)\mid X_t\big]\Bigg\}\Bigg|\Bigg]
\end{align*}
Then, there exist constants $\gamma_1 > 0$ such that
\begin{align*}
\leq &\gamma_1 \mathbb{E}\left[\sum^{K}_{a=1}\Big|f_{t-1}(a, X_t) - \mathbb{E}\big[Y(a)\mid X_t\big]\Big|\right].
\end{align*}
Here, from the assumption that $f_{t-1}(a, x)- \mathbb{E}\big[Y(a)\mid X\big]\xrightarrow{\mathrm{p}}0$ for $\tau=2,3,\dots,M$, and $f_{t_{\tau-1}}(a, x)$ is bounded for $\tau \in I$, we can use $L^r$ convergence theorem. First, to use $L^r$ convergence theorem, we use boundedness of $f_{t_m}$ to derive the uniform integrability of $f_{t_m}$ for $m=0,1,\dots,\tau-1$. Then, from $L^r$ convergence theorem, we have $\mathbb{E}\big[|f_{t_m}(a, X) - \mathbb{E}[Y(a)\mid X]|\big]\to 0$ as $t_m\to\infty$. Using this results, we can show that, as $t_{\tau-1}\to\infty$  (this also means $T\to\infty$),
\begin{align*}
&\gamma_1 \sum^{K}_{a=1}\mathbb{E}\left[\Big|f_{t-1}(a, X_t) - \mathbb{E}\big[Y(a)\mid X_t\big]\Big|\right] \to 0.
\end{align*}
Therefore, as $t_{\tau-1}\to\infty$ ($T\to\infty$),
\begin{align*}
&\mathbb{E}\left[\left(\sum^{K}_{a=1}\left\{\frac{\epol(a\mid X_t)\mathbbm{1}[A_t=a]\big\{Y_t - f_{t-1}(a, X_t)\big\}}{\pi_{\tau}(a\mid X_t, \Omega_{t_{\tau - 1}})} + \epol(a\mid X_t)f_{t-1}(a, X_t)\right\} - R(\epol)\right)^2\right]\\
&\to \mathbb{E}\left[\left(\sum^{K}_{a=1}\left\{\frac{\epol(a\mid X_t)\mathbbm{1}[A_t=a]\big\{Y_t - \mathbb{E}\big[Y_t(a)\mid X_t\big]\big\}}{\pi_{\tau}(a\mid X_t, \Omega_{t_{\tau - 1}})} + \epol(a\mid X_t)\mathbb{E}\big[Y(a)\mid X_t\big]\right\} - R(\epol)\right)^2\right].
\end{align*} 
Then, by using $\mathbbm{1}[A_t=a]\mathbbm{1}[A_t=l] = 0$, $\mathbb{E}\left[\frac{\mathbbm{1}[A_t=a]Y^2_t}{\big(\pi_{\tau}(a\mid X_t, \Omega_{t_{\tau - 1}})\big)^2}\right] = \mathbb{E}\left[\frac{\mathbb{E}\big[Y^2_t(a)\mid X_t\big]}{\pi_{\tau}(a\mid X_t, \Omega_{t_{\tau - 1}})}\right]$, and $\frac{1}{r_\tau T}\sum^T_{t=1}\mathbbm{1}\big[t_{\tau - 1} < t \leq t_{\tau} \big]=1$,
\begin{align*}
&\mathbb{E}\left[\left(\sum^{K}_{a=1}\left\{\frac{\epol(a\mid X_t)\mathbbm{1}[A_t=a]\big\{Y_t - \mathbb{E}\big[Y_t(a)\mid X_t\big]\big\}}{\pi_{\tau}(a\mid X_t, \Omega_{t_{\tau - 1}})} + \epol(a\mid X_t)\mathbb{E}\big[Y(a)\mid X_t\big]\right\} - R(\epol)\right)^2\right]\\
&= \mathbb{E}\left[\sum^{K}_{a=1}\left\{\frac{\big(\epol(a\mid X_t)\big)^2\mathrm{Var}\big(Y_t(a) \mid X_t\big)}{\pi_{\tau}(a\mid X_t, \Omega_{t_{\tau - 1}})} + \Big(\epol(a\mid X_t)\mathbb{E}\big[Y_t(a)\mid X_t\big] - R(\epol)\Big)^2\right\}\right].
\end{align*}
In addition, the variance does not depend on $t$. We represent the independence by omitting the subscript $t$, i.e.,
\begin{align*}
&\mathbb{E}\left[\sum^{K}_{a=1}\left\{\frac{\big(\epol(a\mid X_t)\big)^2\mathrm{Var}\big(Y_t(a) \mid X_t\big)}{\pi_{\tau}(a\mid X_t, \Omega_{t_{\tau - 1}})} + \Big(\epol(a\mid X_t)\mathbb{E}\big[Y_t(a)\mid X_t\big] - R(\epol)\Big)^2\right\}\right]\\
&=\mathbb{E}\left[\sum^{K}_{a=1}\left\{\frac{\big(\epol(a\mid X)\big)^2\mathrm{Var}\big(Y(a) \mid X\big)}{\pi_{\tau}(a\mid X, \Omega_{t_{\tau - 1}})} + \Big(\epol(a\mid X)\mathbb{E}\big[Y_t(a)\mid X\big] - R(\epol)\Big)^2\right\}\right].
\end{align*}
Therefore, we have
\begin{align*}
&\frac{1}{r^2_\tau T}\sum^T_{t=1}\mathbb{E}\Bigg[\left(\sum^{K}_{a=1}\left\{\frac{\epol(a\mid X_t)\mathbbm{1}[A_t=a]\big\{Y_t - f_{t-1}(a, X_t)\big\}}{\pi_{\tau}(a\mid X_t, \Omega_{t_{\tau - 1}})} + \epol(a\mid X_t)f_{t-1}(a, X_t)\right\} - R(\epol)\right)^2\\
&\ \ \ \ \ \ \ \ \ \ \ \ \ \ \ \ \ \ \ \ \ \ \ \ \ \ \ \ \ \ \ \ \ \ \ \ \ \ \ \ \ \ \ \ \ \ \ \ \ \ \ \ \ \ \ \ \ \ \ \ \ \ \ \ \ \ \ \ \ \ \ \ \ \ \ \ \ \ \ \ \times \mathbbm{1}\big[t_{\tau - 1} < t \leq t_{\tau} \big]\Bigg]\\
&\to \frac{1}{r_\tau}\mathbb{E}\left[\sum^{K}_{a=1}\left\{\frac{\big(\epol(a\mid X)\big)^2\mathrm{Var}\big(Y(a) \mid X\big)}{\pi_{\tau}(a\mid X, \Omega_{t_{\tau - 1}})} + \Big(\epol(a\mid X)\mathbb{E}\big[Y(a)\mid X\big] - R(\epol)\Big)^2\right\}\right].
\end{align*}

Thus, the matrix $(1/T)\sum^T_{t=1}\Sigma_t$ converges to a diagonal matrix $\Sigma$ as $T\to\infty$, where the $(\tau,\tau)$-element of $\Sigma$ is
\begin{align*}
\frac{1}{r_\tau}\mathbb{E}\left[\sum^{K}_{a=1}\left\{\frac{\big(\epol(a\mid X)\big)^2\mathrm{Var}\big(Y(a) \mid X\big)}{\pi_{\tau}(a\mid X, \Omega_{t_{\tau - 1}})} + \Big(\epol(a\mid X)\mathbb{E}\big[Y(a)\mid X\big] - R(\epol)\Big)^2\right\}\right].
\end{align*}

\subsection*{Step~2: Condition~(b)}
Because we assume that all variables are bounded, this condition holds.

\subsection*{Step~3: Condition~(c)}
Here, we check that $(1/T)\sum^{T}_{t=1}\left(\bm{h}_t\left(X_t, A_t, Y_t; R(\epol), f_{t-1}, \pi_\tau,\epol\right)\right)\left(\bm{h}_t\left(X_t, A_t, Y_t; R(\epol), f_{t-1}, \pi_\tau,\epol\right)\right)^{\top}\xrightarrow{p}\Sigma$. The $(\tau,\tau)$-element of the matrix is 
\begin{align*}
&\frac{1}{T}\sum^{T}_{t=1}\frac{1}{r^2_\tau}\left(\sum^{K}_{a=1}\left\{\frac{\epol(a\mid X_t)\mathbbm{1}[A_t=a]\big\{Y_t - f_{t-1}(a, X_t)\big\}}{\pi_{\tau}(a\mid X, \Omega_{t_{\tau - 1}})} + \epol(a\mid X_t)f_{t-1}(a, X_t)\right\} - \theta\right)^2\\
&\ \ \ \ \ \ \ \ \  \ \ \ \ \ \ \ \ \ \ \ \ \ \ \ \ \ \ \ \ \ \ \ \ \ \ \ \ \ \ \  \ \ \ \ \ \ \ \ \ \ \ \ \ \ \ \ \ \ \ \ \ \  \times\mathbbm{1}\big[t_{\tau - 1} < t \leq t_{\tau} \big]\\
&=\frac{1}{T}\sum^{T}_{t=1}\frac{1}{r^2_\tau}\left(\sum^{K}_{a=1}\left\{\frac{\epol(a\mid X_t)\mathbbm{1}[A_t=a]\big\{Y_t - f_{t-1}(a, X_t)\big\}}{\pi_{\tau}(a\mid X, \Omega_{t_{\tau - 1}})} + \epol(a\mid X_t)f_{t-1}(a, X_t)\right\} - \theta\right)^2\\
&\ \ \ \ \ \ \ \ \  \ \ \ \ \ \ \ \ \ \ \ \ \ \ \ \ \ \ \ \ \ \ \ \ \ \ \ \ \ \ \  \ \ \ \ \ \ \ \ \ \ \ \ \ \ \ \ \ \ \ \ \ \  \times\mathbbm{1}\big[t_{\tau - 1} < t \leq t_{\tau} \big]\\
& - \frac{1}{T}\sum^{T}_{t=1}\frac{1}{r^2_\tau}\left(\sum^{K}_{a=1}\left\{\frac{\epol(a\mid X_t)\mathbbm{1}[A_t=a]\big\{Y_t - \mathbb{E}\big[Y(a)\mid X_t\big]\big\}}{\pi_{\tau}(a\mid X, \Omega_{t_{\tau - 1}})} + \epol(a\mid X_t)\mathbb{E}\big[Y(a)\mid X_t\big]\right\} - \theta\right)^2\\
&\ \ \ \ \ \ \ \ \  \ \ \ \ \ \ \ \ \ \ \ \ \ \ \ \ \ \ \ \ \ \ \ \ \ \ \ \ \ \ \  \ \ \ \ \ \ \ \ \ \ \ \ \ \ \ \ \ \ \ \ \ \  \times \mathbbm{1}\big[t_{\tau - 1} < t \leq t_{\tau} \big]\\
& + \frac{1}{T}\sum^{T}_{t=1}\frac{1}{r^2_\tau}\left(\sum^{K}_{a=1}\left\{\frac{\epol(a\mid X_t)\mathbbm{1}[A_t=a]\big\{Y_t - \mathbb{E}\big[Y(a)\mid X_t\big]\big\}}{\pi_{\tau}(a\mid X, \Omega_{t_{\tau - 1}})} + \epol(a\mid X_t)\mathbb{E}\big[Y(a)\mid X_t\big]\right\} - \theta\right)^2\\
&\ \ \ \ \ \ \ \ \  \ \ \ \ \ \ \ \ \ \ \ \ \ \ \ \ \ \ \ \ \ \ \ \ \ \ \ \ \ \ \  \ \ \ \ \ \ \ \ \ \ \ \ \ \ \ \ \ \ \ \ \ \  \times\mathbbm{1}\big[t_{\tau - 1} < t \leq t_{\tau} \big].
\end{align*}
The part 
\begin{align*}
&\frac{1}{T}\sum^{T}_{t=1}\frac{1}{r^2_\tau}\left(\sum^{K}_{a=1}\left\{\frac{\epol(a\mid X_t)\mathbbm{1}[A_t=a]\big\{Y_t - f_{t-1}(a, X_t)\big\}}{\pi_{\tau}(a\mid X, \Omega_{t_{\tau - 1}})} + \epol(a\mid X_t)f_{t-1}(a, X_t)\right\} - \theta\right)^2\\
&\ \ \ \ \ \ \ \ \  \ \ \ \ \ \ \ \ \ \ \ \ \ \ \ \ \ \ \ \ \ \ \ \ \ \ \ \ \ \ \  \ \ \ \ \ \ \ \ \ \ \ \ \ \ \ \ \ \ \ \ \ \  \times\mathbbm{1}\big[t_{\tau - 1} < t \leq t_{\tau} \big]\\
& - \frac{1}{T}\sum^{T}_{t=1}\frac{1}{r^2_\tau}\left(\sum^{K}_{a=1}\left\{\frac{\epol(a\mid X_t)\mathbbm{1}[A_t=a]\big\{Y_t - \mathbb{E}\big[Y(a) \mid X_t\big]\big\}}{\pi_{\tau}(a\mid X, \Omega_{t_{\tau - 1}})} + \epol(a\mid X_t)\mathbb{E}\big[Y(a)\mid X_t\big]\right\} - \theta\right)^2\\
&\ \ \ \ \ \ \ \ \  \ \ \ \ \ \ \ \ \ \ \ \ \ \ \ \ \ \ \ \ \ \ \ \ \ \ \ \ \ \ \  \ \ \ \ \ \ \ \ \ \ \ \ \ \ \ \ \ \ \ \ \ \  \times\mathbbm{1}\big[t_{\tau - 1} < t \leq t_{\tau} \big]
\end{align*}
converges in probability to $0$ because $f_{t-1}(a, X_t) \xrightarrow{\mathrm{p}} \mathbb{E}\big[Y(a)\mid X_t\big]$. The term 
\begin{align*}
&\frac{1}{T}\sum^{T}_{t=1}\frac{1}{r^2_\tau}\left(\sum^{K}_{a=1}\left\{\frac{\epol(a\mid X_t)\mathbbm{1}[A_t=a]\big\{Y_t - \mathbb{E}\big[Y(a) \mid X_t\big]\big\}}{\pi_{\tau}(a\mid X, \Omega_{t_{\tau - 1}})} + \epol(a\mid X_t)\mathbb{E}\big[Y(a) \mid X_t\big]\right\} - R(\epol)\right)^2\\
&\ \ \ \ \ \ \ \ \  \ \ \ \ \ \ \ \ \ \ \ \ \ \ \ \ \ \ \ \ \ \ \ \ \ \ \ \ \ \ \  \ \ \ \ \ \ \ \ \ \ \ \ \ \ \ \ \ \ \ \ \ \  \times\mathbbm{1}\big[t_{\tau - 1} < t \leq t_{\tau} \big].
\end{align*}
converges in probability to 
\begin{align*}
\frac{1}{r_\tau}\mathbb{E}\left[\sum^{K}_{a=1}\left\{\frac{\big(\epol(a\mid X)\big)^2\mathrm{Var}\big(Y(a) \mid X\big)}{\pi_{\tau}(a\mid X, \Omega_{t_{\tau - 1}})} + \Big(\epol(a\mid X)\mathbb{E}\big[Y(a)\mid X\big] - R(\epol)\Big)^2\right\}\right].
\end{align*}
from the weak law of large numbers for i.i.d. samples as $t_{\tau-1}-t_{\tau} \to \infty$ because the samples are i.i.d. between $t_{\tau-1}$ and $t_{\tau}$. 
\end{proof}

\subsection{OPE estimator when the true logging policy is unknown}
\label{appdx:batch_est}
Then, we consider estimating the policy value without using the true logging policy $\pi_t$. We use adaptive-fitting for obtaining an asymptotically normal estimator. As the ADR estimator under Assumption~\ref{asm:stationarity}, we estimate $\pi_t$ and $f^*$ only using $\Omega_{t-1}$ and denote them as $g_{t-1}$ and $\hat{f}_{t-1}$, respectively. Then, we define an estimator of $R(\epol)$ as
\begin{align*}
\tilde{R}^{\mathrm{batch}}_T(\epol) = \argmin_{R\in\mathbb{R}} \left(\tilde{\bm{q}}_T(R)\right)^\top \hat{W}_T \left(\tilde{\bm{q}}_T(R)\right),
\end{align*}
where $\tilde{\bm{q}}_T(R) = \frac{1}{T}\sum^{T}_{t=1}\bm{h}_t\left(X_t, A_t, Y_t; R, \hat{f}_{t-1}, \hat{g}_{t-1} \epol\right)$ and $\hat{W}_T$ is a data-dependent $(M\times M)$-dimensional positive semi-definite matrix. This estimator is the ADR estimator under batch update. As well as the proof of Theorem~\ref{thm:asymp_dist_adr}, it is hold that 
$$ \left| \tilde{R}^{\mathrm{batch}}_T(\epol) - \hat{R}^{\mathrm{batch}}_T(\epol)  \right| = \mathrm{o}_p(1/\sqrt{T}) $$ if $\|\hat{g}_{t-1}(a| X_t) - \pi_{t-1}(a| X_t, \Omega_{t-1})\|_{2}=\op(t^{-p})$, and $\|\hat{f}_{t-1}(a,X_t)-f^*(a,X_t)\|_2=\op(t^{-q})$, where $p, q > 0$ such that $p+q = 1/2$. Therefore, we can obtain the following theorem. 

\begin{theorem}[Asymptotic distribution the ADR estimator under batch update]
Suppose that 
\begin{description}
\item[(i)] $\hat W_T \xrightarrow{\mathrm{p}} W$;
\item[(ii)] $W$ is a positive definite.
\end{description}
Then, under Assumptions~\ref{asm:overlap_pol}, \ref{asm:bounded_reward}, \ref{asm:conv_rate1}--\ref{asm:bound_nuisance2}, 
\begin{align*}
\sqrt{T}\big(\tilde{R}^{\mathrm{batch}}_T(\epol) - R(\epol)\big)  \xrightarrow{\mathrm{d}}\mathcal{N}\big(0, \sigma^2\big),
\end{align*}
where $\sigma^2 = \big(I^\top W I \big)^{-1} I^\top W \Sigma W^\top I \big(I^\top W I \big)^{-1}$ and $\Sigma$ is a $(M\times M)$ diagonal matrix such that the $(\tau\times \tau)$-element is $\Psi(\epol,\pi_\tau)$.
\end{theorem}

\section{Details of experiments}
\label{appdx:det_exp}

The description of the dataset is shown in Table~\ref{Dataset}. We use LinUCB and LinTS policies. We add uniform sampling to make overlap between policies. We can relax this requirement by considering different DGPs, such as batched sampling. However, for brevity, we adopt this setting. Additional results are shown as follows. 

For numerical experiments in Section~\ref{sec:num_exp}, we show the result with sample sizes $T= 100,1,000$ in Table~\ref{tbl:exp_table_bias2}. In addition, we show the error distribution with sample size $T= 100$ in Figure~\ref{fig:bias2}; we show the error distribution with sample size $T=500$ in Figure~\ref{fig:bias3}; we show the error distribution with sample size $T=500$ in Figure~\ref{fig:bias4}; we show the error distribution with sample size $T=1,000$ in Figure~\ref{fig:bias5}. 

For experiments with dependent samples in Section~\ref{sec:exp}, we show the additional results with different settings in Tables~\ref{tbl:appdx:exp_table1}--\ref{tbl:appdx:exp_table6}. In Table~\ref{tbl:appdx:exp_table1}, we show the results using the benchmark datasets with $800$ samples generated from the LinUCB algorithm.  In Table~\ref{tbl:appdx:exp_table2}, we show the results using the benchmark datasets with $1,000$ samples generated from the LinUCB algorithm. In Table~\ref{tbl:appdx:exp_table3}, we show the results using the the benchmark datasets with $1,200$ samples generated from the LinUCB algorithm. In Table~\ref{tbl:appdx:exp_table4}, we show the results using the the benchmark datasets with $800$ samples generated from the LinTS algorithm. In Table~\ref{tbl:appdx:exp_table5}, we show the results using the the benchmark datasets with $1,000$ samples generated from the LinTS algorithm. In Table~\ref{tbl:appdx:exp_table6}, we show the results using the the benchmark datasets with $1,200$ samples generated from the LinTS algorithm.

Next, we compare the estimators using the benchmark datasets generated from the logistic regression as well as the evaluation weight; that is, the samples are i.i.d.  For $\alpha\in\{0.7, 0.4, 0.1\}$ and the sample sizes $800$, $1,000$, and $1,200$, we calculate the RMSEs and the SDs over $10$ trials. We show additional results with different settings in Tables~\ref{tbl:appdx:exp_table7}--\ref{tbl:appdx:exp_table9}. In Table~\ref{tbl:appdx:exp_table7}, we show the results using the benchmark datasets with $800$ samples. In Table~\ref{tbl:appdx:exp_table8}, we show the results using the benchmark datasets with $1,000$ samples. In Table~\ref{tbl:appdx:exp_table9}, we show the results using the benchmark datasets with $1,200$ samples. In these experiments, AIPW estimators show better performances than the ADR estimator. We conjecture that this is because the logging policy is not unstable unlike the case with dependent samples; that is, the paradox is specific to the case where samples are dependent. Note this case (i.i.d. samples) is not in the scope of the proposed method. In such cases, it is common to use other methods, such as the cross-fitting proposed by \citet{ChernozhukovVictor2018Dmlf}, instead of the AIPW and ADR estimators discussed in this paper. We show this result to clarify the cause of the paradox. Since the ADR and AIPW estimators have the same asymptotic variance, and since the AIPW estimator uses more information, it is a natural result that the AIPW estimator performs better. However, when the samples are dependent, the ADR estimator paradoxically outperforms the AIPW estimator because of the unstable behavior of the logging policy $\pi_t$, as pointed out by \citet{hadad2019}.

As a surprising discovery, our proposed ADR estimator shows better results than the AIPW estimator, although the AIPW estimator uses more information (true logging policy $\pi_t$) than the ADR estimator and their asymptotic properties are the same. As discussed above, we consider that this result is due to the logging policy's unstable behavior. Even when knowing the true logging policy $\pi_t$, we can stabilize estimation by reestimating the logging policy from $(A_t, X_t)$. This paradox is similar to the well-known property that the IPW estimator using an estimated propensity score shows a smaller asymptotic variance than the IPW using the true propensity score \citep{hirano2003,Henmi2004paradox,Henmi2007imp}. However, we consider that they are different phenomena. In previous studies, the paradox is mainly explained by differences in the asymptotic variance between the IPW estimators with the true and an estimated propensity score. On the other hand, for our case, the AIPW and ADR estimators have the same asymptotic variance, unlike IPW-type estimators; therefore, we cannot elucidate the paradox by traditional explanation.

\begin{table}[ht]
\caption{Specification of datasets}
\label{Dataset}
\begin{center}
\scalebox{1}[1]{\begin{tabular}{cccc}
\hline
Dataset&the number of samples &Dimension &the number of classes\\
\hline
 {\tt mnist} & 60,000 &  780 & 10 \\
 {\tt satimage} & 4,435  &  35 & 6 \\
 {\tt sensorless} & 58,509 &  48 & 11 \\
 {\tt connect-4} & 67,557 &  126 & 3\\
\end{tabular}}
\end{center}
\end{table}

\begin{table*}[t]
\vspace{-0.1cm}
\caption{The results of Section~\ref{sec:a2ipw} with sample sizes $T=100, 1,000$. We show the RMSE, SD, and coverage ratio of the confidence interval (CR). We highlight in red bold two estimators with the lowest RMSE.
%and highlight in under line the estimator with the lowest RMSE among estimators with asymptotic normality. 
Estimators with asymptotic normality are marked with $\dagger$, and estimators that do not require the true logging policy are marked with $*$.} 
\vspace{-0.3cm}
\label{tbl:exp_table_bias2}
\begin{center}
\scalebox{0.61}[0.61]{
\begin{tabular}{|l||rrr||rrr|rrr|rrr||rrr|rrr|}
\hline
\multicolumn{19}{|c|}{LinUCB policy}\\
\hline
{$T$} &    \multicolumn{3}{c||}{ADR $\dagger*$} &   \multicolumn{3}{c|}{IPW $\dagger$} &   \multicolumn{3}{c|}{AIPW $\dagger$} &   
\multicolumn{3}{c||}{AW-AIPW $\dagger$} &   \multicolumn{3}{c|}{DM $*$} &    \multicolumn{3}{c|}{EIPW $*$}\\
\hline
 &      RMSE &      SD &      CR &      RMSE &      SD &      CR &      RMSE &      SD &      CR &      RMSE &     SD &     CR &  RMSE &      SD &      CR &  RMSE &      SD &      CR\\
\hline
100 &  \textcolor{red}{\textbf{0.106}} &  0.011 &  0.86 &  0.158 &  0.036 &  0.8 &  0.139 &  0.024 &  0.87 &  0.161 &  0.035 &  0.34 &  0.116 &  0.016 &  0.29 &  0.144 &  0.031 &  0.73 \\
1,000 &  \textcolor{red}{\textbf{0.030}} &  0.001 &  0.93 &  0.083 &  0.035 &  0.88 &  0.061 &  0.011 &  0.87 &  0.189 &  0.016 &  0.0 &  0.033 &  0.001 &  0.22 &  0.133 &  0.012 &  0.07 \\
\hline
\end{tabular}
} 
\end{center}
\vspace{-0.3cm}
\begin{center}
\scalebox{0.61}[0.61]{
\begin{tabular}{|l||rrr||rrr|rrr|rrr||rrr|rrr|}
\hline
\multicolumn{19}{|c|}{LinTS policy}\\
\hline
{$T$} &    \multicolumn{3}{c||}{ADR $\dagger*$} &   \multicolumn{3}{c|}{IPW $\dagger$} &   \multicolumn{3}{c|}{AIPW $\dagger$} &   
\multicolumn{3}{c||}{AW-AIPW $\dagger$} &   \multicolumn{3}{c|}{DM $*$} &    \multicolumn{3}{c|}{EIPW $*$}\\
\hline
 &      RMSE &      SD &      CR &      RMSE &      SD &      CR &      RMSE &      SD &      CR &      RMSE &     SD &     CR &  RMSE &      SD &      CR &  RMSE &      SD &      CR\\
\hline
100 &  \textcolor{red}{\textbf{0.087}} &  0.008 &  0.92 &  0.181 &  0.062 &  0.88 &  0.150 &  0.032 &  0.92 &  0.124 &  0.02 &  0.42 &  0.100 &  0.011 &  0.25 &  0.140 &  0.026 &  0.82 \\
1,000 &  \textcolor{red}{\textbf{0.026}} &  0.001 &  0.98 &  0.065 &  0.008 &  0.89 &  0.053 &  0.005 &  0.94 &  0.155 &  0.011 &  0.0 &  0.033 &  0.001 &  0.15 &  0.114 &  0.009 &  0.09 \\
\hline
\end{tabular}
} 
\end{center}
\vspace{-0.5cm}
\end{table*}

\begin{figure}[t]
\begin{center}
 \includegraphics[width=120mm]{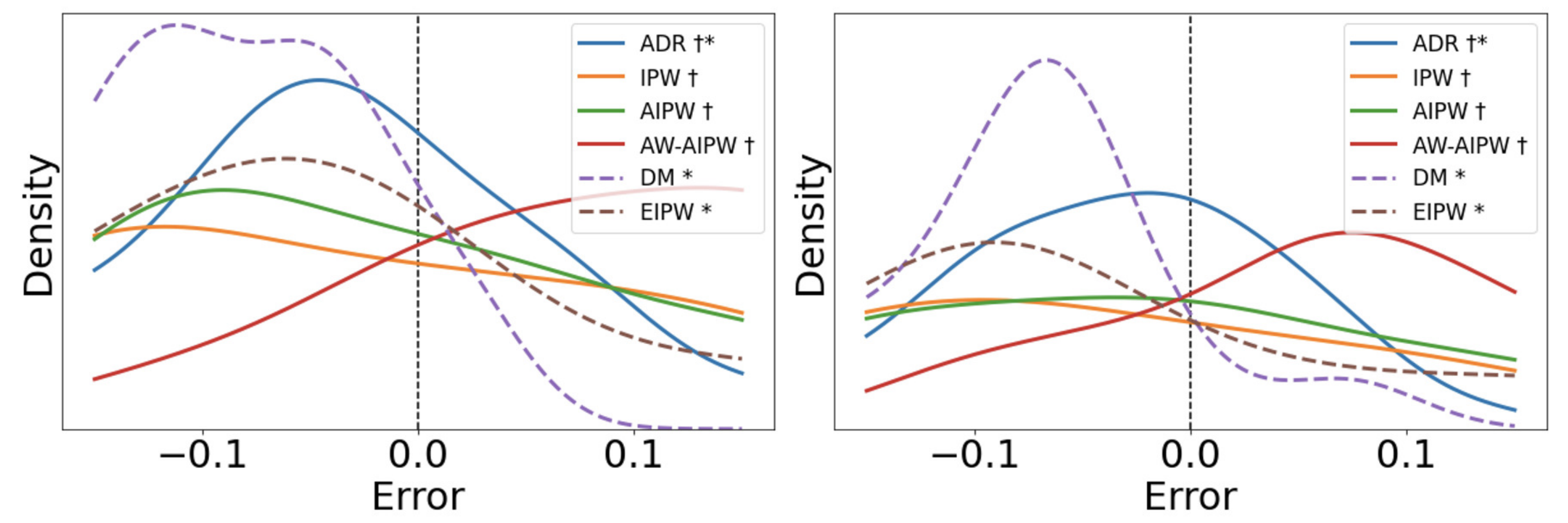}
\end{center}
\vspace{-0.5cm}
\caption{This figure illustrates the error distributions of estimators for OPVE from dependent samples generated with the LinUcB(left) and LinTS(right) with the sample size $100$. We smoothed the error distributions using kernel density estimation. Estimators with asymptotic normality are marked with $\dagger$, and estimators that do not require a true logging policy are marked with $*$.}
\label{fig:bias2}
\begin{center}
 \includegraphics[width=120mm]{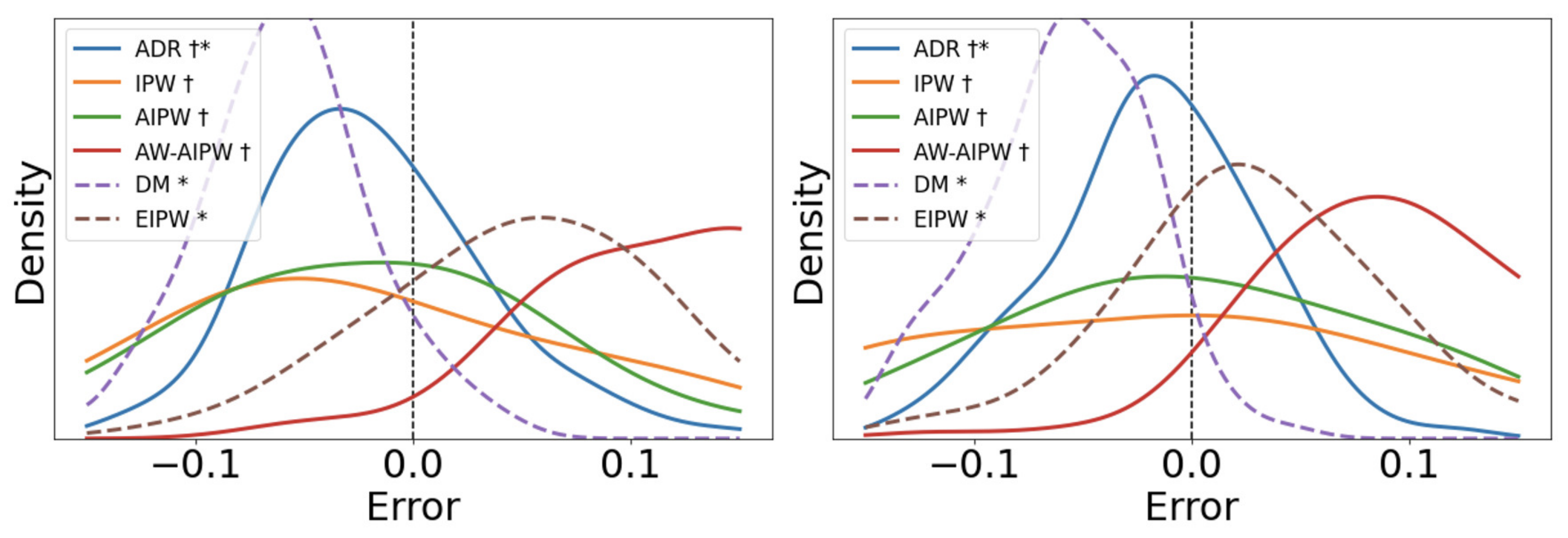}
\end{center}
\vspace{-0.5cm}
\caption{This figure illustrates the error distributions of estimators for OPVE from dependent samples generated with the LinUcB(left) and LinTS(right) with the sample size $250$. We smoothed the error distributions using kernel density estimation. Estimators with asymptotic normality are marked with $\dagger$, and estimators that do not require a true logging policy are marked with $*$.}
\label{fig:bias3}
\begin{center}
 \includegraphics[width=120mm]{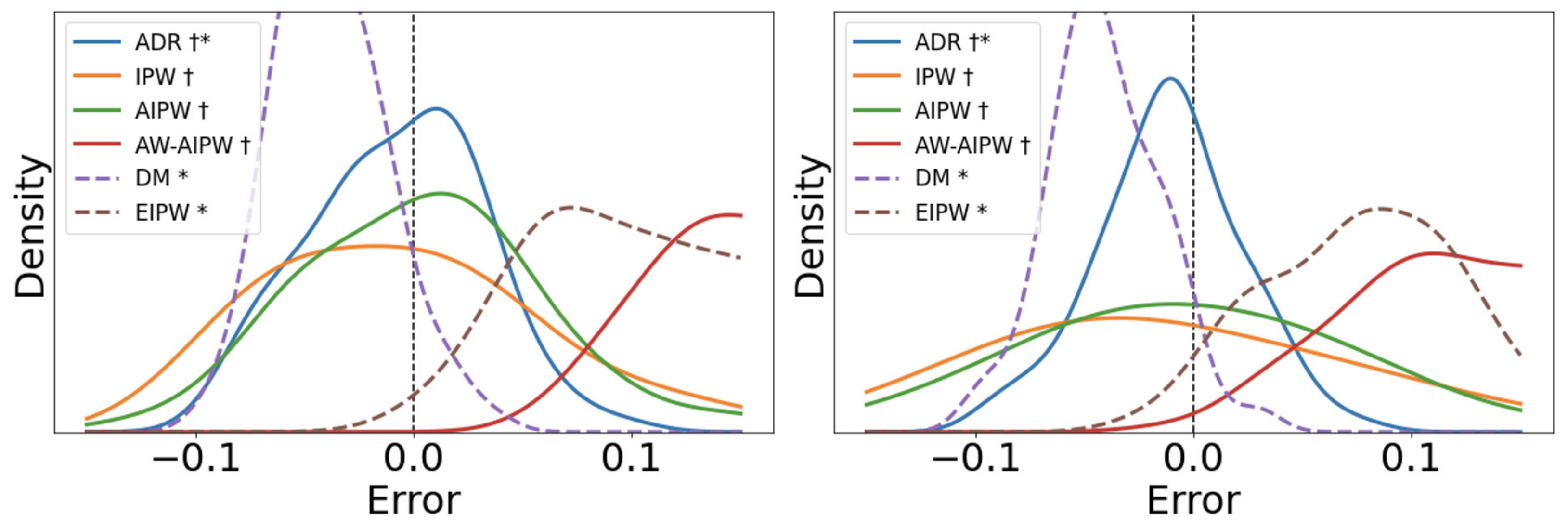}
\end{center}
\vspace{-0.5cm}
\caption{This figure illustrates the error distributions of estimators for OPVE from dependent samples generated with the LinUcB(left) and LinTS(right) with the sample size $500$. We smoothed the error distributions using kernel density estimation. Estimators with asymptotic normality are marked with $\dagger$, and estimators that do not require a true logging policy are marked with $*$.}
\label{fig:bias4}
\begin{center}
 \includegraphics[width=120mm]{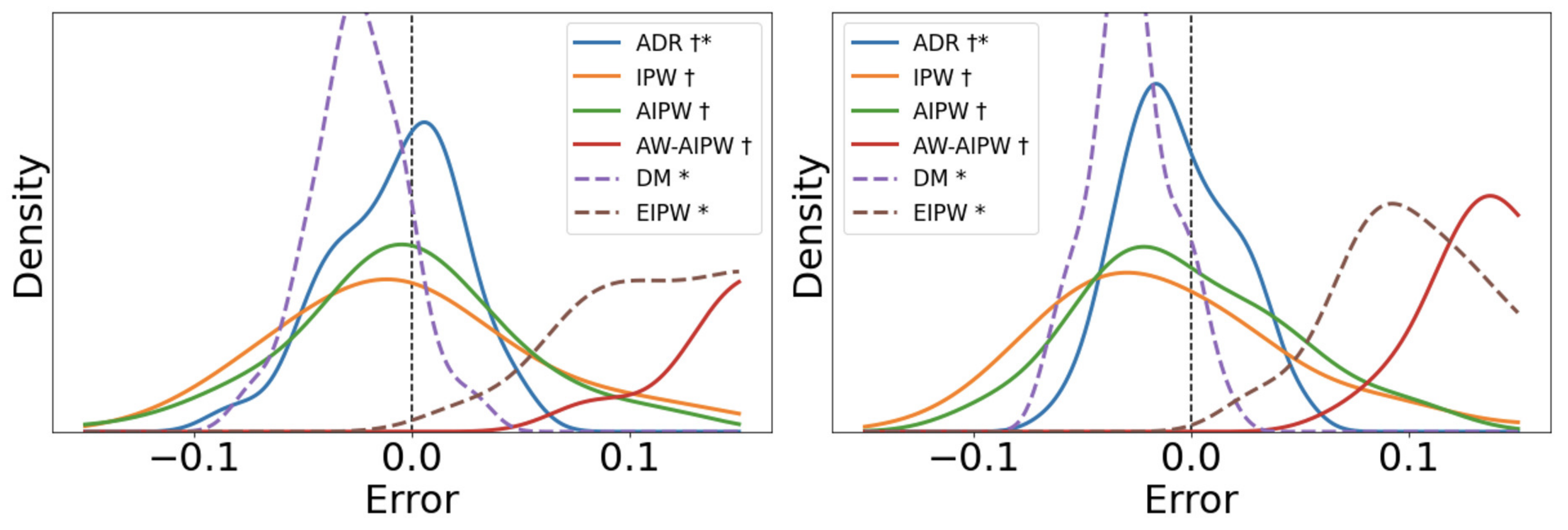}
\end{center}
\vspace{-0.5cm}
\caption{This figure illustrates the error distributions of estimators for OPVE from dependent samples generated with the LinUcB(left) and LinTS(right) with the sample size $1,000$. We smoothed the error distributions using kernel density estimation. Estimators with asymptotic normality are marked with $\dagger$, and estimators that do not require a true logging policy are marked with $*$.}
\label{fig:bias5}
\vspace{-0.5cm}
\end{figure} 

\begin{table*}[t]
\vspace{-0.2cm}
\caption{The results of benchmark datasets with the LinUCB policy and $T=800$. We highlight in red bold the estimator with the lowest RMSE and highlight in under line the estimator with the lowest RMSE among estimators that do not use the true logging policy. Estimators with asymptotic normality are marked with $\dagger$, and estimators that do not require the true logging policy are marked with $*$.} 
\label{tbl:appdx:exp_table1}
\begin{center}
\scalebox{0.80}[0.80]{
\begin{tabular}{|l||rr||rr|rr||rr|rr|}
\hline
{\tt mnist}$\ \ \ \ \ \ \ \ \ \ $ &    \multicolumn{2}{|c||}{ADR\ $\dagger*$} &   \multicolumn{2}{c|}{IPW $\dagger$} &   \multicolumn{2}{c||}{AIPW $\dagger$} &    \multicolumn{2}{c|}{DM $*$} &    \multicolumn{2}{c|}{EIPW $*$} \\
\hline
$\alpha$ &      RMSE &      SD &      RMSE &      SD &      RMSE &      SD &      RMSE &      SD &      RMSE &      SD \\
\hline
0.7 &  \underline{\textcolor{red}{\textbf{0.047}}} &  0.003 &  0.099 &  0.011 &  0.132 &  0.032 &  0.279 &  0.025 &  0.130 &  0.019 \\
0.4 &  \underline{\textcolor{red}{\textbf{0.034}}} &  0.002 &  0.067 &  0.005 &  0.115 &  0.014 &  0.281 &  0.015 &  0.095 &  0.009 \\
0.1 &  \underline{\textcolor{red}{\textbf{0.053}}} &  0.006 &  0.079 &  0.006 &  0.107 &  0.013 &  0.287 &  0.028 &  0.119 &  0.018 \\
\hline
\end{tabular}
}
\end{center}
\vspace{-0.3cm}
\begin{center}
\scalebox{0.80}[0.80]{
\begin{tabular}{|l||rr||rr|rr||rr|rr|}
\hline
{\tt satimage}$\ \ \ \ $ &    \multicolumn{2}{|c||}{ADR\ $\dagger*$} &   \multicolumn{2}{c|}{IPW $\dagger$} &   \multicolumn{2}{c||}{AIPW $\dagger$} &    \multicolumn{2}{c|}{DM $*$} &    \multicolumn{2}{c|}{EIPW $*$} \\
\hline
$\alpha$ &      RMSE &      SD &      RMSE &      SD &      RMSE &      SD &      RMSE &      SD &      RMSE &      SD \\
\hline
0.7 &  \underline{\textcolor{red}{\textbf{0.015}}} &  0.000 &  0.097 &  0.013 &  0.074 &  0.007 &  0.035 &  0.001 &  0.065 &  0.003 \\
0.4 &  \underline{\textcolor{red}{\textbf{0.020}}} &  0.000 &  0.078 &  0.006 &  0.088 &  0.010 &  0.041 &  0.001 &  0.085 &  0.014 \\
0.1 &  \underline{\textcolor{red}{\textbf{0.026}}} &  0.001 &  0.080 &  0.006 &  0.092 &  0.011 &  0.059 &  0.003 &  0.097 &  0.012 \\
\hline
\end{tabular}
}
\end{center}
\vspace{-0.3cm}
\begin{center}
\scalebox{0.80}[0.80]{
\begin{tabular}{|l||rr||rr|rr||rr|rr|}
\hline
{\tt sensorless} &    \multicolumn{2}{|c||}{ADR\ $\dagger*$} &   \multicolumn{2}{c|}{IPW $\dagger$} &   \multicolumn{2}{c||}{AIPW $\dagger$} &    \multicolumn{2}{c|}{DM $*$} &    \multicolumn{2}{c|}{EIPW $*$} \\
\hline
$\alpha$ &      RMSE &      SD &      RMSE &      SD &      RMSE &      SD &      RMSE &      SD &      RMSE &      SD \\
\hline
0.7 &  \underline{\textcolor{red}{\textbf{0.050}}} &  0.002 &  0.253 &  0.049 &  0.101 &  0.016 &  0.130 &  0.008 &  0.121 &  0.016 \\
0.4 &  \underline{\textcolor{red}{\textbf{0.065}}} &  0.006 &  0.270 &  0.039 &  0.104 &  0.019 &  0.150 &  0.013 &  0.096 &  0.012 \\
0.1 &  \underline{\textcolor{red}{\textbf{0.042}}} &  0.001 &  0.250 &  0.043 &  0.075 &  0.006 &  0.127 &  0.008 &  0.102 &  0.010 \\
\hline
\end{tabular}
}
\end{center}
\vspace{-0.3cm}
\begin{center}
\scalebox{0.80}[0.80]{
\begin{tabular}{|l||rr||rr|rr||rr|rr|}
\hline
{\tt connect-4}$\ \ $ &    \multicolumn{2}{|c||}{ADR\ $\dagger*$} &   \multicolumn{2}{c|}{IPW $\dagger$} &   \multicolumn{2}{c||}{AIPW $\dagger$} &    \multicolumn{2}{c|}{DM $*$} &    \multicolumn{2}{c|}{EIPW $*$} \\
\hline
$\alpha$ &      RMSE &      SD &      RMSE &      SD &      RMSE &      SD &      RMSE &      SD &      RMSE &      SD \\
\hline
0.7 &  \underline{\textcolor{red}{\textbf{0.024}}} &  0.001 &  0.121 &  0.019 &  0.042 &  0.002 &  0.045 &  0.001 &  0.140 &  0.007 \\
0.4 &  \underline{\textcolor{red}{\textbf{0.023}}} &  0.001 &  0.125 &  0.014 &  0.033 &  0.001 &  0.044 &  0.002 &  0.090 &  0.007 \\
0.1 &  \underline{\textcolor{red}{\textbf{0.028}}} &  0.001 &  0.107 &  0.012 &  0.042 &  0.002 &  0.068 &  0.004 &  0.031 &  0.001 \\
\hline
\end{tabular}
}
\end{center}
\vspace{-0.65cm}
\end{table*}

\begin{table*}[t]
\vspace{-0.2cm}
\caption{The results of benchmark datasets with the LinUCB policy and $T=1,000$. We highlight in red bold the estimator with the lowest RMSE and highlight in under line the estimator with the lowest RMSE among estimators that do not use the true logging policy. Estimators with asymptotic normality are marked with $\dagger$, and estimators that do not require the true logging policy are marked with $*$.} 
\label{tbl:appdx:exp_table2}
\begin{center}
\scalebox{0.80}[0.80]{
\begin{tabular}{|l||rr||rr|rr||rr|rr|}
\hline
{\tt mnist}$\ \ \ \ \ \ \ \ \ \ $ &    \multicolumn{2}{|c||}{ADR\ $\dagger*$} &   \multicolumn{2}{c|}{IPW $\dagger$} &   \multicolumn{2}{c||}{AIPW $\dagger$} &    \multicolumn{2}{c|}{DM $*$} &    \multicolumn{2}{c|}{EIPW $*$} \\
\hline
$\alpha$ &      RMSE &      SD &      RMSE &      SD &      RMSE &      SD &      RMSE &      SD &      RMSE &      SD \\
\hline
0.7 &  \underline{\textcolor{red}{\textbf{0.046}}} &  0.002 &  0.100 &  0.011 &  0.162 &  0.027 &  0.232 &  0.013 &  0.148 &  0.014 \\
0.4 &  \underline{\textcolor{red}{\textbf{0.028}}} &  0.001 &  0.068 &  0.005 &  0.112 &  0.009 &  0.249 &  0.010 &  0.080 &  0.004 \\
0.1 &  \underline{0.086} &  0.006 &  \textcolor{red}{\textbf{0.078}} &  0.006 &  0.085 &  0.008 &  0.299 &  0.024 &  0.091 &  0.008 \\
\hline
\end{tabular}
}
\end{center}
\vspace{-0.3cm}
\begin{center}
\scalebox{0.80}[0.80]{
\begin{tabular}{|l||rr||rr|rr||rr|rr|}
\hline
{\tt satimage}$\ \ \ \ $ &    \multicolumn{2}{|c||}{ADR\ $\dagger*$} &   \multicolumn{2}{c|}{IPW $\dagger$} &   \multicolumn{2}{c||}{AIPW $\dagger$} &    \multicolumn{2}{c|}{DM $*$} &    \multicolumn{2}{c|}{EIPW $*$} \\
\hline
$\alpha$ &      RMSE &      SD &      RMSE &      SD &      RMSE &      SD &      RMSE &      SD &      RMSE &      SD \\
\hline
0.7 &  \underline{\textcolor{red}{\textbf{0.013}}} &  0.000 &  0.098 &  0.011 &  0.060 &  0.004 &  0.037 &  0.001 &  0.056 &  0.002 \\
0.4 &  \underline{\textcolor{red}{\textbf{0.022}}} &  0.000 &  0.078 &  0.008 &  0.019 &  0.000 &  0.043 &  0.001 &  0.060 &  0.002 \\
0.1 &  \underline{\textcolor{red}{\textbf{0.029}}} &  0.001 &  0.078 &  0.005 &  0.061 &  0.008 &  0.041 &  0.002 &  0.041 &  0.002 \\
\hline
\end{tabular}
}
\end{center}
\vspace{-0.3cm}
\begin{center}
\scalebox{0.80}[0.80]{
\begin{tabular}{|l||rr||rr|rr||rr|rr|}
\hline
{\tt sensorless} &    \multicolumn{2}{|c||}{ADR\ $\dagger*$} &   \multicolumn{2}{c|}{IPW $\dagger$} &   \multicolumn{2}{c||}{AIPW $\dagger$} &    \multicolumn{2}{c|}{DM $*$} &    \multicolumn{2}{c|}{EIPW $*$} \\
\hline
$\alpha$ &      RMSE &      SD &      RMSE &      SD &      RMSE &      SD &      RMSE &      SD &      RMSE &      SD \\
\hline
0.7 &  \underline{\textcolor{red}{\textbf{0.043}}} &  0.001 &  0.270 &  0.049 &  0.096 &  0.014 &  0.109 &  0.014 &  0.068 &  0.007 \\
0.4 &  \underline{\textcolor{red}{\textbf{0.053}}} &  0.004 &  0.287 &  0.035 &  0.070 &  0.006 &  0.112 &  0.008 &  0.104 &  0.013 \\
0.1 &  \underline{\textcolor{red}{\textbf{0.048}}} &  0.002 &  0.260 &  0.031 &  0.072 &  0.006 &  0.154 &  0.015 &  0.088 &  0.012 \\
\hline
\end{tabular}
}
\end{center}
\vspace{-0.3cm}
\begin{center}
\scalebox{0.80}[0.80]{
\begin{tabular}{|l||rr||rr|rr||rr|rr|}
\hline
{\tt connect-4}$\ \ $ &    \multicolumn{2}{|c||}{ADR\ $\dagger*$} &   \multicolumn{2}{c|}{IPW $\dagger$} &   \multicolumn{2}{c||}{AIPW $\dagger$} &    \multicolumn{2}{c|}{DM $*$} &    \multicolumn{2}{c|}{EIPW $*$} \\
\hline
$\alpha$ &      RMSE &      SD &      RMSE &      SD &      RMSE &      SD &      RMSE &      SD &      RMSE &      SD \\
\hline
0.7 &  \underline{\textcolor{red}{\textbf{0.030}}} &  0.001 &  0.135 &  0.024 &  0.042 &  0.002 &  0.032 &  0.001 &  0.173 &  0.006 \\
0.4 &  \underline{\textcolor{red}{\textbf{0.015}}} &  0.000 &  0.135 &  0.017 &  0.037 &  0.001 &  0.038 &  0.001 &  0.085 &  0.002 \\
0.1 &  \underline{\textcolor{red}{\textbf{0.012}}} &  0.000 &  0.117 &  0.013 &  0.030 &  0.001 &  0.051 &  0.002 &  0.023 &  0.001 \\
\hline
\end{tabular}
}
\end{center}
\vspace{-0.65cm}
\end{table*}

\begin{table*}[t]
\vspace{-0.2cm}
\caption{The results of benchmark datasets with the LinUCB policy and $T=1,200$. We highlight in red bold the estimator with the lowest RMSE and highlight in under line the estimator with the lowest RMSE among estimators that do not use the true logging policy. Estimators with asymptotic normality are marked with $\dagger$, and estimators that do not require the true logging policy are marked with $*$.} 
\label{tbl:appdx:exp_table3}
\begin{center}
\scalebox{0.80}[0.80]{
\begin{tabular}{|l||rr||rr|rr||rr|rr|}
\hline
{\tt mnist}$\ \ \ \ \ \ \ \ \ \ $ &    \multicolumn{2}{|c||}{ADR\ $\dagger*$} &   \multicolumn{2}{c|}{IPW $\dagger$} &   \multicolumn{2}{c||}{AIPW $\dagger$} &    \multicolumn{2}{c|}{DM $*$} &    \multicolumn{2}{c|}{EIPW $*$} \\
\hline
$\alpha$ &      RMSE &      SD &      RMSE &      SD &      RMSE &      SD &      RMSE &      SD &      RMSE &      SD \\
\hline
0.7 &  \underline{\textcolor{red}{\textbf{0.063}}} &  0.003 &  0.095 &  0.010 &  0.114 &  0.018 &  0.204 &  0.007 &  0.218 &  0.021 \\
0.4 &  \underline{\textcolor{red}{\textbf{0.031}}} &  0.002 &  0.061 &  0.004 &  0.100 &  0.011 &  0.227 &  0.010 &  0.136 &  0.015 \\
0.1 &  0.073 &  0.005 &  0.076 &  0.006 &  0.093 &  0.013 &  0.245 &  0.013 & \underline{\textcolor{red}{\textbf{0.059}}} &  0.004 \\
\hline
\end{tabular}
}
\end{center}
\vspace{-0.3cm}
\begin{center}
\scalebox{0.80}[0.80]{
\begin{tabular}{|l||rr||rr|rr||rr|rr|}
\hline
{\tt satimage}$\ \ \ \ $ &    \multicolumn{2}{|c||}{ADR\ $\dagger*$} &   \multicolumn{2}{c|}{IPW $\dagger$} &   \multicolumn{2}{c||}{AIPW $\dagger$} &    \multicolumn{2}{c|}{DM $*$} &    \multicolumn{2}{c|}{EIPW $*$} \\
\hline
$\alpha$ &      RMSE &      SD &      RMSE &      SD &      RMSE &      SD &      RMSE &      SD &      RMSE &      SD \\
\hline
0.7 &  \underline{\textcolor{red}{\textbf{0.015}}} &  0.000 &  0.091 &  0.010 &  0.039 &  0.002 &  0.035 &  0.001 &  0.039 &  0.001 \\
0.4 &  \underline{\textcolor{red}{\textbf{0.016}}} &  0.000 &  0.075 &  0.006 &  0.041 &  0.002 &  0.043 &  0.001 &  0.048 &  0.002 \\
0.1 &  \underline{\textcolor{red}{\textbf{0.019}}} &  0.000 &  0.075 &  0.005 &  0.033 &  0.001 &  0.045 &  0.001 &  0.073 &  0.005 \\
\hline
\end{tabular}
}
\end{center}
\vspace{-0.3cm}
\begin{center}
\scalebox{0.80}[0.80]{
\begin{tabular}{|l||rr||rr|rr||rr|rr|}
\hline
{\tt sensorless} &    \multicolumn{2}{|c||}{ADR\ $\dagger*$} &   \multicolumn{2}{c|}{IPW $\dagger$} &   \multicolumn{2}{c||}{AIPW $\dagger$} &    \multicolumn{2}{c|}{DM $*$} &    \multicolumn{2}{c|}{EIPW $*$} \\
\hline
$\alpha$ &      RMSE &      SD &      RMSE &      SD &      RMSE &      SD &      RMSE &      SD &      RMSE &      SD \\
\hline
0.7 &  \underline{\textcolor{red}{\textbf{0.049}}} &  0.003 &  0.272 &  0.059 &  0.102 &  0.018 &  0.092 &  0.008 &  0.078 &  0.010 \\
0.4 &  \underline{\textcolor{red}{\textbf{0.047}}} &  0.004 &  0.286 &  0.037 &  0.071 &  0.007 &  0.096 &  0.006 &  0.082 &  0.009 \\
0.1 &  \underline{\textcolor{red}{\textbf{0.056}}} &  0.004 &  0.268 &  0.047 &  0.061 &  0.006 &  0.116 &  0.008 &  0.067 &  0.005 \\
\hline
\end{tabular}
}
\end{center}
\vspace{-0.3cm}
\begin{center}
\scalebox{0.80}[0.80]{
\begin{tabular}{|l||rr||rr|rr||rr|rr|}
\hline
{\tt connect-4}$\ \ $ &    \multicolumn{2}{|c||}{ADR\ $\dagger*$} &   \multicolumn{2}{c|}{IPW $\dagger$} &   \multicolumn{2}{c||}{AIPW $\dagger$} &    \multicolumn{2}{c|}{DM $*$} &    \multicolumn{2}{c|}{EIPW $*$} \\
\hline
$\alpha$ &      RMSE &      SD &      RMSE &      SD &      RMSE &      SD &      RMSE &      SD &      RMSE &      SD \\
\hline
0.7 &  \underline{\textcolor{red}{\textbf{0.033}}} &  0.001 &  0.138 &  0.027 &  0.046 &  0.002 &  0.037 &  0.002 &  0.184 &  0.013 \\
0.4 &  \underline{\textcolor{red}{\textbf{0.020}}} &  0.000 &  0.130 &  0.014 &  0.052 &  0.003 &  0.038 &  0.001 &  0.114 &  0.004 \\
0.1 &  \underline{\textcolor{red}{\textbf{0.020}}} &  0.000 &  0.118 &  0.013 &  0.037 &  0.002 &  0.034 &  0.001 &  0.048 &  0.003 \\
\hline
\end{tabular}
}
\end{center}
\vspace{-0.65cm}
\end{table*}

\begin{table*}[t]
\vspace{-0.2cm}
\caption{The results of benchmark datasets with the LinTS policy and $T=800$. We highlight in red bold the estimator with the lowest RMSE and highlight in under line the estimator with the lowest RMSE among estimators that do not use the true logging policy. Estimators with asymptotic normality are marked with $\dagger$, and estimators that do not require the true logging policy are marked with $*$.} 
\label{tbl:appdx:exp_table4}
\begin{center}
\scalebox{0.80}[0.80]{
\begin{tabular}{|l||rr||rr|rr||rr|rr|}
\hline
{\tt mnist}$\ \ \ \ \ \ \ \ \ \ $ &    \multicolumn{2}{|c||}{ADR\ $\dagger*$} &   \multicolumn{2}{c|}{IPW $\dagger$} &   \multicolumn{2}{c||}{AIPW $\dagger$} &    \multicolumn{2}{c|}{DM $*$} &    \multicolumn{2}{c|}{EIPW $*$} \\
\hline
$\alpha$ &      RMSE &      SD &      RMSE &      SD &      RMSE &      SD &      RMSE &      SD &      RMSE &      SD \\
\hline
0.7 &  \underline{\textcolor{red}{\textbf{0.064}}} &  0.006 &  0.098 &  0.011 &  0.150 &  0.019 &  0.289 &  0.026 &  0.172 &  0.019 \\
0.4 &  \underline{0.084} &  0.005 &  \textcolor{red}{\textbf{0.060}} &  0.003 &  0.089 &  0.007 &  0.311 &  0.027 &  0.116 &  0.016 \\
0.1 &  \underline{0.117} &  0.011 &  0.078 &  0.006 &  \textcolor{red}{\textbf{0.061}} &  0.004 &  0.331 &  0.030 &  0.065 &  0.004 \\
\hline
\end{tabular}
}
\end{center}
\vspace{-0.3cm}
\begin{center}
\scalebox{0.80}[0.80]{
\begin{tabular}{|l||rr||rr|rr||rr|rr|}
\hline
{\tt satimage}$\ \ \ \ $ &    \multicolumn{2}{|c||}{ADR\ $\dagger*$} &   \multicolumn{2}{c|}{IPW $\dagger$} &   \multicolumn{2}{c||}{AIPW $\dagger$} &    \multicolumn{2}{c|}{DM $*$} &    \multicolumn{2}{c|}{EIPW $*$} \\
\hline
$\alpha$ &      RMSE &      SD &      RMSE &      SD &      RMSE &      SD &      RMSE &      SD &      RMSE &      SD \\
\hline
0.7 &  \underline{\textcolor{red}{\textbf{0.032}}} &  0.001 &  0.102 &  0.011 &  0.094 &  0.007 &  0.043 &  0.002 &  0.095 &  0.016 \\
0.4 &  \underline{\textcolor{red}{\textbf{0.033}}} &  0.001 &  0.081 &  0.009 &  0.073 &  0.005 &  0.048 &  0.002 &  0.098 &  0.015 \\
0.1 &  \underline{\textcolor{red}{\textbf{0.028}}} &  0.001 &  0.075 &  0.004 &  0.074 &  0.005 &  0.043 &  0.001 &  0.048 &  0.003 \\
\hline
\end{tabular}
}
\end{center}
\vspace{-0.3cm}
\begin{center}
\scalebox{0.80}[0.80]{
\begin{tabular}{|l||rr||rr|rr||rr|rr|}
\hline
{\tt sensorless} &    \multicolumn{2}{|c||}{ADR\ $\dagger*$} &   \multicolumn{2}{c|}{IPW $\dagger$} &   \multicolumn{2}{c||}{AIPW $\dagger$} &    \multicolumn{2}{c|}{DM $*$} &    \multicolumn{2}{c|}{EIPW $*$} \\
\hline
$\alpha$ &      RMSE &      SD &      RMSE &      SD &      RMSE &      SD &      RMSE &      SD &      RMSE &      SD \\
\hline
0.7 &  \underline{\textcolor{red}{\textbf{0.055}}} &  0.003 &  0.252 &  0.053 &  0.094 &  0.011 &  0.158 &  0.015 &  0.112 &  0.014 \\
0.4 &  0.084 &  0.006 &  0.267 &  0.033 &  0.078 &  0.007 &  0.177 &  0.024 &  \underline{\textcolor{red}{\textbf{0.059}}} &  0.004 \\
0.1 &  0.086 &  0.009 &  0.247 &  0.041 &  0.104 &  0.017 &  0.164 &  0.009 &  \underline{\textcolor{red}{\textbf{0.079}}} &  0.006 \\
\hline
\end{tabular}
}
\end{center}
\vspace{-0.3cm}
\begin{center}
\scalebox{0.80}[0.80]{
\begin{tabular}{|l||rr||rr|rr||rr|rr|}
\hline
{\tt connect-4}$\ \ $ &    \multicolumn{2}{|c||}{ADR\ $\dagger*$} &   \multicolumn{2}{c|}{IPW $\dagger$} &   \multicolumn{2}{c||}{AIPW $\dagger$} &    \multicolumn{2}{c|}{DM $*$} &    \multicolumn{2}{c|}{EIPW $*$} \\
\hline
$\alpha$ &      RMSE &      SD &      RMSE &      SD &      RMSE &      SD &      RMSE &      SD &      RMSE &      SD \\
\hline
0.7 &  \underline{\textcolor{red}{\textbf{0.021}}} &  0.000 &  0.129 &  0.019 &  0.057 &  0.003 &  0.041 &  0.001 &  0.115 &  0.004 \\
0.4 &  \underline{\textcolor{red}{\textbf{0.025}}} &  0.001 &  0.130 &  0.018 &  0.055 &  0.003 &  0.050 &  0.003 &  0.064 &  0.004 \\
0.1 &  \underline{\textcolor{red}{\textbf{0.018}}} &  0.000 &  0.107 &  0.011 &  0.054 &  0.004 &  0.054 &  0.002 &  0.029 &  0.001 \\
\hline
\end{tabular}
}
\end{center}
\vspace{-0.65cm}
\end{table*}

\begin{table*}[t]
\vspace{-0.2cm}
\caption{The results of benchmark datasets with the LinTS policy and $T=1,000$. We highlight in red bold the estimator with the lowest RMSE and highlight in under line the estimator with the lowest RMSE among estimators that do not use the true logging policy. Estimators with asymptotic normality are marked with $\dagger$, and estimators that do not require the true logging policy are marked with $*$.} 
\label{tbl:appdx:exp_table5}
\begin{center}
\scalebox{0.80}[0.80]{
\begin{tabular}{|l||rr||rr|rr||rr|rr|}
\hline
{\tt mnist}$\ \ \ \ \ \ \ \ \ \ $ &    \multicolumn{2}{|c||}{ADR\ $\dagger*$} &   \multicolumn{2}{c|}{IPW $\dagger$} &   \multicolumn{2}{c||}{AIPW $\dagger$} &    \multicolumn{2}{c|}{DM $*$} &    \multicolumn{2}{c|}{EIPW $*$} \\
\hline
$\alpha$ &      RMSE &      SD &      RMSE &      SD &      RMSE &      SD &      RMSE &      SD &      RMSE &      SD \\
\hline
0.7 &  \underline{\textcolor{red}{\textbf{0.053}}} &  0.006 &  0.097 &  0.010 &  0.140 &  0.023 &  0.248 &  0.008 &  0.184 &  0.020 \\
0.4 &  \underline{0.077} &  0.008 &  \textcolor{red}{\textbf{0.070}} &  0.005 &  0.098 &  0.011 &  0.276 &  0.028 &  0.099 &  0.011 \\
0.1 & 0.099 &  0.011 &  \textcolor{red}{\textbf{0.076}} &  0.006 &  0.091 &  0.008 &  0.296 &  0.026 &   \underline{0.088} &  0.007 \\
\hline
\end{tabular}
}
\end{center}
\vspace{-0.3cm}
\begin{center}
\scalebox{0.80}[0.80]{
\begin{tabular}{|l||rr||rr|rr||rr|rr|}
\hline
{\tt satimage}$\ \ \ \ $ &    \multicolumn{2}{|c||}{ADR\ $\dagger*$} &   \multicolumn{2}{c|}{IPW $\dagger$} &   \multicolumn{2}{c||}{AIPW $\dagger$} &    \multicolumn{2}{c|}{DM $*$} &    \multicolumn{2}{c|}{EIPW $*$} \\
\hline
$\alpha$ &      RMSE &      SD &      RMSE &      SD &      RMSE &      SD &      RMSE &      SD &      RMSE &      SD \\
\hline
0.7 &  \underline{\textcolor{red}{\textbf{0.016}}} &  0.000 &  0.096 &  0.011 &  0.050 &  0.004 &  0.044 &  0.001 &  0.056 &  0.003 \\
0.4 &  \underline{\textcolor{red}{\textbf{0.021}}} &  0.001 &  0.077 &  0.007 &  0.041 &  0.002 &  0.048 &  0.001 &  0.063 &  0.009 \\
0.1 &  \underline{\textcolor{red}{\textbf{0.026}}} &  0.001 &  0.079 &  0.005 &  0.087 &  0.011 &  0.048 &  0.002 &  0.091 &  0.020 \\
\hline
\end{tabular}
}
\end{center}
\vspace{-0.3cm}
\begin{center}
\scalebox{0.80}[0.80]{
\begin{tabular}{|l||rr||rr|rr||rr|rr|}
\hline
{\tt sensorless} &    \multicolumn{2}{|c||}{ADR\ $\dagger*$} &   \multicolumn{2}{c|}{IPW $\dagger$} &   \multicolumn{2}{c||}{AIPW $\dagger$} &    \multicolumn{2}{c|}{DM $*$} &    \multicolumn{2}{c|}{EIPW $*$} \\
\hline
$\alpha$ &      RMSE &      SD &      RMSE &      SD &      RMSE &      SD &      RMSE &      SD &      RMSE &      SD \\
\hline
0.7 &  \underline{\textcolor{red}{\textbf{0.040}}} &  0.003 &  0.273 &  0.057 &  0.083 &  0.009 &  0.114 &  0.008 &  0.087 &  0.008 \\
0.4 &  \underline{\textcolor{red}{\textbf{0.034}}} &  0.001 &  0.291 &  0.035 &  0.091 &  0.010 &  0.099 &  0.006 &  0.070 &  0.007 \\
0.1 &  \underline{\textcolor{red}{\textbf{0.060}}} &  0.002 &  0.268 &  0.042 &  0.075 &  0.008 &  0.146 &  0.011 &  0.069 &  0.006 \\
\hline
\end{tabular}
}
\end{center}
\vspace{-0.3cm}
\begin{center}
\scalebox{0.80}[0.80]{
\begin{tabular}{|l||rr||rr|rr||rr|rr|}
\hline
{\tt connect-4}$\ \ $ &    \multicolumn{2}{|c||}{ADR\ $\dagger*$} &   \multicolumn{2}{c|}{IPW $\dagger$} &   \multicolumn{2}{c||}{AIPW $\dagger$} &    \multicolumn{2}{c|}{DM $*$} &    \multicolumn{2}{c|}{EIPW $*$} \\
\hline
$\alpha$ &      RMSE &      SD &      RMSE &      SD &      RMSE &      SD &      RMSE &      SD &      RMSE &      SD \\
\hline
0.7 &  \underline{\textcolor{red}{\textbf{0.014}}} &  0.000 &  0.132 &  0.020 &  0.059 &  0.003 &  0.045 &  0.001 &  0.110 &  0.006 \\
0.4 &  \underline{\textcolor{red}{\textbf{0.019}}} &  0.000 &  0.136 &  0.017 &  0.028 &  0.001 &  0.046 &  0.002 &  0.066 &  0.004 \\
0.1 &  \underline{\textcolor{red}{\textbf{0.016}}} &  0.000 &  0.114 &  0.012 &  0.023 &  0.001 &  0.050 &  0.002 &  0.019 &  0.000 \\
\hline
\end{tabular}
}
\end{center}
\vspace{-0.65cm}
\end{table*}

\begin{table*}[t]
\vspace{-0.2cm}
\caption{The results of benchmark datasets with the LinTS policy and $T=1,200$. We highlight in red bold the estimator with the lowest RMSE and highlight in under line the estimator with the lowest RMSE among estimators that do not use the true logging policy. Estimators with asymptotic normality are marked with $\dagger$, and estimators that do not require the true logging policy are marked with $*$.} 
\label{tbl:appdx:exp_table6}
\begin{center}
\scalebox{0.80}[0.80]{
\begin{tabular}{|l||rr||rr|rr||rr|rr|}
\hline
{\tt mnist}$\ \ \ \ \ \ \ \ \ \ $ &    \multicolumn{2}{|c||}{ADR\ $\dagger*$} &   \multicolumn{2}{c|}{IPW $\dagger$} &   \multicolumn{2}{c||}{AIPW $\dagger$} &    \multicolumn{2}{c|}{DM $*$} &    \multicolumn{2}{c|}{EIPW $*$} \\
\hline
$\alpha$ &      RMSE &      SD &      RMSE &      SD &      RMSE &      SD &      RMSE &      SD &      RMSE &      SD \\
\hline
0.7 &  \underline{\textcolor{red}{\textbf{0.038}}} &  0.003 &  0.095 &  0.010 &  0.095 &  0.010 &  0.234 &  0.013 &  0.203 &  0.028 \\
0.4 &  \underline{\textcolor{red}{\textbf{0.048}}} &  0.003 &  0.066 &  0.004 &  0.099 &  0.012 &  0.245 &  0.013 &  0.117 &  0.011 \\
0.1 &  0.058 &  0.002 &  0.073 &  0.006 &  \underline{\textcolor{red}{\textbf{0.052}}} &  0.003 &  0.257 &  0.021 &  0.069 &  0.006 \\
\hline
\end{tabular}
}
\end{center}
\vspace{-0.3cm}
\begin{center}
\scalebox{0.80}[0.80]{
\begin{tabular}{|l||rr||rr|rr||rr|rr|}
\hline
{\tt satimage}$\ \ \ \ $ &    \multicolumn{2}{|c||}{ADR\ $\dagger*$} &   \multicolumn{2}{c|}{IPW $\dagger$} &   \multicolumn{2}{c||}{AIPW $\dagger$} &    \multicolumn{2}{c|}{DM $*$} &    \multicolumn{2}{c|}{EIPW $*$} \\
\hline
$\alpha$ &      RMSE &      SD &      RMSE &      SD &      RMSE &      SD &      RMSE &      SD &      RMSE &      SD \\
\hline
0.7 &  \underline{\textcolor{red}{\textbf{0.012}}} &  0.000 &  0.096 &  0.010 &  0.058 &  0.004 &  0.037 &  0.001 &  0.059 &  0.003 \\
0.4 &  \underline{\textcolor{red}{\textbf{0.020}}} &  0.001 &  0.079 &  0.007 &  0.036 &  0.001 &  0.043 &  0.002 &  0.064 &  0.010 \\
0.1 &  \underline{\textcolor{red}{\textbf{0.027}}} &  0.001 &  0.079 &  0.005 &  0.045 &  0.002 &  0.050 &  0.001 &  0.050 &  0.004 \\
\hline
\end{tabular}
}
\end{center}
\vspace{-0.3cm}
\begin{center}
\scalebox{0.80}[0.80]{
\begin{tabular}{|l||rr||rr|rr||rr|rr|}
\hline
{\tt sensorless} &    \multicolumn{2}{|c||}{ADR\ $\dagger*$} &   \multicolumn{2}{c|}{IPW $\dagger$} &   \multicolumn{2}{c||}{AIPW $\dagger$} &    \multicolumn{2}{c|}{DM $*$} &    \multicolumn{2}{c|}{EIPW $*$} \\
\hline
$\alpha$ &      RMSE &      SD &      RMSE &      SD &      RMSE &      SD &      RMSE &      SD &      RMSE &      SD \\
\hline
0.7 &  \underline{\textcolor{red}{\textbf{0.027}}} &  0.001 &  0.257 &  0.043 &  0.063 &  0.004 &  0.102 &  0.005 &  0.053 &  0.004 \\
0.4 &  \underline{\textcolor{red}{\textbf{0.038}}} &  0.002 &  0.277 &  0.036 &  0.059 &  0.003 &  0.106 &  0.006 &  0.087 &  0.010 \\
0.1 &  0.046 &  0.002 &  0.252 &  0.035 &  \textcolor{red}{\textbf{0.034}} &  0.002 &  0.138 &  0.008 &  \underline{0.045} &  0.003 \\
\hline
\end{tabular}
}
\end{center}
\vspace{-0.3cm}
\begin{center}
\scalebox{0.80}[0.80]{
\begin{tabular}{|l||rr||rr|rr||rr|rr|}
\hline
{\tt connect-4}$\ \ $ &    \multicolumn{2}{|c||}{ADR\ $\dagger*$} &   \multicolumn{2}{c|}{IPW $\dagger$} &   \multicolumn{2}{c||}{AIPW $\dagger$} &    \multicolumn{2}{c|}{DM $*$} &    \multicolumn{2}{c|}{EIPW $*$} \\
\hline
$\alpha$ &      RMSE &      SD &      RMSE &      SD &      RMSE &      SD &      RMSE &      SD &      RMSE &      SD \\
\hline
0.7 &  \underline{\textcolor{red}{\textbf{0.014}}} &  0.000 &  0.376 &  0.244 &  0.044 &  0.002 &  0.029 &  0.001 &  0.107 &  0.007 \\
0.4 &  \underline{\textcolor{red}{\textbf{0.021}}} &  0.000 &  0.395 &  0.283 &  0.025 &  0.001 &  0.033 &  0.001 &  0.067 &  0.004 \\
0.1 & \underline{\textcolor{red}{\textbf{0.215}}} &  0.138 &  0.385 &  0.273 &  \textcolor{red}{\textbf{0.215}} &  0.138 &  0.217 &  0.137 & \underline{ \textcolor{red}{\textbf{0.215}}} &  0.138 \\
\hline
\end{tabular}
}
\end{center}
\vspace{-0.65cm}
\end{table*}

\begin{table*}[t]
\vspace{-0.2cm}
\caption{The results of benchmark datasets with the i.i.d samples and $T=800$. We highlight in red bold the estimator with the lowest RMSE and highlight in under line the estimator with the lowest RMSE among estimators that do not use the true logging policy. Estimators with asymptotic normality are marked with $\dagger$, and estimators that do not require the true logging policy are marked with $*$.} 
\label{tbl:appdx:exp_table7}
\begin{center}
\scalebox{0.80}[0.80]{
\begin{tabular}{|l||rr||rr|rr||rr|rr|}
\hline
{\tt mnist}$\ \ \ \ \ \ \ \ \ \ $ &    \multicolumn{2}{|c||}{ADR\ $\dagger*$} &   \multicolumn{2}{c|}{IPW $\dagger$} &   \multicolumn{2}{c||}{AIPW $\dagger$} &    \multicolumn{2}{c|}{DM $\dagger *$} &    \multicolumn{2}{c|}{EIPW $\dagger *$} \\
\hline
$\alpha$ &      RMSE &      SD &      RMSE &      SD &      RMSE &      SD &      RMSE &      SD &      RMSE &      SD \\
\hline
0.7 &  0.195 &  0.090 &  \textcolor{red}{\textbf{0.078}} &  0.061 &  \textcolor{red}{\textbf{0.078}} &  0.061 &  \underline{\textcolor{red}{\textbf{0.078}}} &  0.061 &  2.085 &  9.013 \\
0.4 &  0.192 &  0.088 &  \textcolor{red}{\textbf{0.078}} &  0.061 &  0.079 &  0.061 &  \underline{0.080} &  0.061 &  1.187 &  2.913 \\
0.1 &  \underline{0.105} &  0.061 &  \textcolor{red}{\textbf{0.080}} &  0.061 &  0.083 &  0.061 &  0.108 &  0.061 &  0.298 &  0.189 \\
\hline
\end{tabular}
}
\end{center}
\vspace{-0.3cm}
\begin{center}
\scalebox{0.80}[0.80]{
\begin{tabular}{|l||rr||rr|rr||rr|rr|}
\hline
{\tt satimage}$\ \ \ \ $ &    \multicolumn{2}{|c||}{ADR\ $\dagger*$} &   \multicolumn{2}{c|}{IPW $\dagger$} &   \multicolumn{2}{c||}{AIPW $\dagger$} &    \multicolumn{2}{c|}{DM $\dagger *$} &    \multicolumn{2}{c|}{EIPW $\dagger *$} \\
\hline
$\alpha$ &      RMSE &      SD &      RMSE &      SD &      RMSE &      SD &      RMSE &      SD &      RMSE &      SD \\
\hline
0.7 &  0.025 &  0.001 &  \textcolor{red}{\textbf{0.003}} &  0.00 &  0.008 &  0.000 &  \underline{0.006} &  0.00 &  0.076 &  0.014 \\
0.4 &  0.019 &  0.001 &  \textcolor{red}{\textbf{0.006}} &  0.00 &  0.017 &  0.001 &  \underline{0.012} &  0.00 &  0.074 &  0.012 \\
0.1 &  \underline{\textcolor{red}{\textbf{0.079}}} &  0.060 &  \textcolor{red}{\textbf{0.079}} &  0.06 &  0.083 &  0.060 &  0.081 &  0.06 &  0.098 &  0.061 \\
\hline
\end{tabular}
}
\end{center}
\vspace{-0.3cm}
\begin{center}
\scalebox{0.80}[0.80]{
\begin{tabular}{|l||rr||rr|rr||rr|rr|}
\hline
{\tt sensorless} &    \multicolumn{2}{|c||}{ADR\ $\dagger*$} &   \multicolumn{2}{c|}{IPW $\dagger$} &   \multicolumn{2}{c||}{AIPW $\dagger$} &    \multicolumn{2}{c|}{DM $\dagger *$} &    \multicolumn{2}{c|}{EIPW $\dagger *$} \\
\hline
$\alpha$ &      RMSE &      SD &      RMSE &      SD &      RMSE &      SD &      RMSE &      SD &      RMSE &      SD \\
\hline
0.7 &  0.032 &  0.003 &  \textcolor{red}{\textbf{0.004}} &  0.000 &  0.005 &  0.000 &  \underline{\textcolor{red}{\textbf{0.004}}} &  0.000 &  0.193 &  0.099 \\
0.4 &  0.031 &  0.003 &  \textcolor{red}{\textbf{0.007}} &  0.000 &  0.008 &  0.000 &  \underline{0.015} &  0.001 &  0.154 &  0.064 \\
0.1 &  \underline{0.052} &  0.021 &  \textcolor{red}{\textbf{0.046}} &  0.021 &  0.048 &  0.021 &  0.057 &  0.021 &  0.121 &  0.050 \\
\hline
\end{tabular}
}
\end{center}
\vspace{-0.3cm}
\begin{center}
\scalebox{0.80}[0.80]{
\begin{tabular}{|l||rr||rr|rr||rr|rr|}
\hline
{\tt connect-4}$\ \ $ &    \multicolumn{2}{|c||}{ADR\ $\dagger*$} &   \multicolumn{2}{c|}{IPW $\dagger$} &   \multicolumn{2}{c||}{AIPW $\dagger$} &    \multicolumn{2}{c|}{DM $\dagger *$} &    \multicolumn{2}{c|}{EIPW $\dagger *$} \\
\hline
$\alpha$ &      RMSE &      SD &      RMSE &      SD &      RMSE &      SD &      RMSE &      SD &      RMSE &      SD \\
\hline
0.7 &  0.080 &  0.047 &  \textcolor{red}{\textbf{0.069}} &  0.047 &  \textcolor{red}{\textbf{0.069}} &  0.047 & \underline{ \textcolor{red}{\textbf{0.069}}} &  0.047 &  0.234 &  0.086 \\
0.4 &  0.074 &  0.047 &  \textcolor{red}{\textbf{0.069}} &  0.047 &  0.071 &  0.047 &  \underline{0.071} &  0.047 &  0.141 &  0.052 \\
0.1 &  \underline{\textcolor{red}{\textbf{0.070}}} &  0.047 &  \textcolor{red}{\textbf{0.070}} &  0.047 &  0.073 &  0.047 &  0.074 &  0.047 &  0.076 &  0.047 \\
\hline
\end{tabular}
}
\end{center}
\vspace{-0.65cm}
\end{table*}

\begin{table*}[t]
\vspace{-0.2cm}
\caption{The results of benchmark datasets with the i.i.d samples and $T=1,000$. We highlight in red bold the estimator with the lowest RMSE and highlight in under line the estimator with the lowest RMSE among estimators that do not use the true logging policy. Estimators with asymptotic normality are marked with $\dagger$, and estimators that do not require the true logging policy are marked with $*$.} 
\label{tbl:appdx:exp_table8}
\begin{center}
\scalebox{0.80}[0.80]{
\begin{tabular}{|l||rr||rr|rr||rr|rr|}
\hline
{\tt mnist}$\ \ \ \ \ \ \ \ \ \ $ &    \multicolumn{2}{|c||}{ADR\ $\dagger*$} &   \multicolumn{2}{c|}{IPW $\dagger$} &   \multicolumn{2}{c||}{AIPW $\dagger$} &    \multicolumn{2}{c|}{DM $\dagger *$} &    \multicolumn{2}{c|}{EIPW $\dagger *$} \\
\hline
$\alpha$ &      RMSE &      SD &      RMSE &      SD &      RMSE &      SD &      RMSE &      SD &      RMSE &      SD \\
\hline
0.7 &  0.367 &  0.062 &  0.006 &  0.000 &  0.011 &  0.000 & \underline{ \textcolor{red}{\textbf{0.005}}} &  0.000 &  4.834 &  2.294 \\
0.4 &  0.373 &  0.037 &  \textcolor{red}{\textbf{0.016}} &  0.000 &  0.031 &  0.001 &  \underline{0.038} &  0.001 &  2.672 &  0.905 \\
0.1 &  0.165 &  0.012 &  \textcolor{red}{\textbf{0.029}} &  0.001 &  0.049 &  0.002 &  \underline{0.155} &  0.005 &  0.671 &  0.098 \\
\hline
\end{tabular}
}
\end{center}
\vspace{-0.3cm}
\begin{center}
\scalebox{0.80}[0.80]{
\begin{tabular}{|l||rr||rr|rr||rr|rr|}
\hline
{\tt satimage}$\ \ \ \ $ &    \multicolumn{2}{|c||}{ADR\ $\dagger*$} &   \multicolumn{2}{c|}{IPW $\dagger$} &   \multicolumn{2}{c||}{AIPW $\dagger$} &    \multicolumn{2}{c|}{DM $\dagger *$} &    \multicolumn{2}{c|}{EIPW $\dagger *$} \\
\hline
$\alpha$ &      RMSE &      SD &      RMSE &      SD &      RMSE &      SD &      RMSE &      SD &      RMSE &      SD \\
\hline
0.7 &  0.043 &  0.002 &  0.006 &  0.000 &  0.013 &  0.000 & \underline{ \textcolor{red}{\textbf{0.010}}} &  0.000 &  0.136 &  0.010 \\
0.4 &  0.038 &  0.001 &  \textcolor{red}{\textbf{0.013}} &  0.000 &  0.035 &  0.001 &  \underline{0.018} &  0.000 &  0.148 &  0.012 \\
0.1 &  \underline{0.028} &  0.001 &  \textcolor{red}{\textbf{0.026}} &  0.001 &  0.034 &  0.001 &  0.044 &  0.002 &  0.113 &  0.019 \\
\hline
\end{tabular}
}
\end{center}
\vspace{-0.3cm}
\begin{center}
\scalebox{0.80}[0.80]{
\begin{tabular}{|l||rr||rr|rr||rr|rr|}
\hline
{\tt sensorless} &    \multicolumn{2}{|c||}{ADR\ $\dagger*$} &   \multicolumn{2}{c|}{IPW $\dagger$} &   \multicolumn{2}{c||}{AIPW $\dagger$} &    \multicolumn{2}{c|}{DM $\dagger *$} &    \multicolumn{2}{c|}{EIPW $\dagger *$} \\
\hline
$\alpha$ &      RMSE &      SD &      RMSE &      SD &      RMSE &      SD &      RMSE &      SD &      RMSE &      SD \\
\hline
0.7 &  0.073 &  0.005 &  \textcolor{red}{\textbf{0.005}} &  0.000 &  0.012 &  0.000 &  \underline{0.009} &  0.000 &  0.454 &  0.071 \\
0.4 &  0.077 &  0.005 &  \textcolor{red}{\textbf{0.009}} &  0.000 &  0.031 &  0.001 &  \underline{0.023} &  0.001 &  0.390 &  0.045 \\
0.1 &  \underline{0.066} &  0.003 &  \textcolor{red}{\textbf{0.042}} &  0.003 &  0.047 &  0.004 &  0.083 &  0.007 &  0.257 &  0.066 \\
\hline
\end{tabular}
}
\end{center}
\vspace{-0.3cm}
\begin{center}
\scalebox{0.80}[0.80]{
\begin{tabular}{|l||rr||rr|rr||rr|rr|}
\hline
{\tt connect-4}$\ \ $ &    \multicolumn{2}{|c||}{ADR\ $\dagger*$} &   \multicolumn{2}{c|}{IPW $\dagger$} &   \multicolumn{2}{c||}{AIPW $\dagger$} &    \multicolumn{2}{c|}{DM $\dagger *$} &    \multicolumn{2}{c|}{EIPW $\dagger *$} \\
\hline
$\alpha$ &      RMSE &      SD &      RMSE &      SD &      RMSE &      SD &      RMSE &      SD &      RMSE &      SD \\
\hline
0.7 &  0.057 &  0.002 &  \textcolor{red}{\textbf{0.009}} &  0.000 &  0.014 &  0.000 &  \underline{0.013} &  0.000 &  0.396 &  0.020 \\
0.4 &  0.040 &  0.001 &  \textcolor{red}{\textbf{0.008}} &  0.000 &  0.021 &  0.001 &  \underline{0.019} &  0.000 &  0.218 &  0.010 \\
0.1 &  \underline{0.027} &  0.001 &  \textcolor{red}{\textbf{0.024}} &  0.001 &  0.047 &  0.003 &  0.042 &  0.002 &  0.056 &  0.003 \\
\hline
\end{tabular}
}
\end{center}
\vspace{-0.65cm}
\end{table*}

\begin{table*}[t]
\vspace{-0.2cm}
\caption{The results of benchmark datasets with the i.i.d samples and $T=1,200$. We highlight in red bold the estimator with the lowest RMSE and highlight in under line the estimator with the lowest RMSE among estimators that do not use the true logging policy. Estimators with asymptotic normality are marked with $\dagger$, and estimators that do not require the true logging policy are marked with $*$.} 
\label{tbl:appdx:exp_table9}
\begin{center}
\scalebox{0.80}[0.80]{
\begin{tabular}{|l||rr||rr|rr||rr|rr|}
\hline
{\tt mnist}$\ \ \ \ \ \ \ \ \ \ $ &    \multicolumn{2}{|c||}{ADR\ $\dagger*$} &   \multicolumn{2}{c|}{IPW $\dagger$} &   \multicolumn{2}{c||}{AIPW $\dagger$} &    \multicolumn{2}{c|}{DM $\dagger *$} &    \multicolumn{2}{c|}{EIPW $\dagger *$} \\
\hline
$\alpha$ &      RMSE &      SD &      RMSE &      SD &      RMSE &      SD &      RMSE &      SD &      RMSE &      SD \\
\hline
0.7 &  0.310 &  0.036 &  0.006 &  0.000 &  0.013 &  0.000 & \underline{ \textcolor{red}{\textbf{0.005}}} &  0.000 &  4.709 &  1.037 \\
0.4 &  0.346 &  0.033 &  0.012 &  0.000 &  \textcolor{red}{\textbf{0.020}} &  0.001 &  \underline{0.040} &  0.001 &  2.735 &  0.461 \\
0.1 &  \underline{0.143} &  0.013 &  \textcolor{red}{\textbf{0.036}} &  0.001 &  0.076 &  0.007 &  0.161 &  0.006 &  0.608 &  0.137 \\
\hline
\end{tabular}
}
\end{center}
\vspace{-0.3cm}
\begin{center}
\scalebox{0.80}[0.80]{
\begin{tabular}{|l||rr||rr|rr||rr|rr|}
\hline
{\tt satimage}$\ \ \ \ $ &    \multicolumn{2}{|c||}{ADR\ $\dagger*$} &   \multicolumn{2}{c|}{IPW $\dagger$} &   \multicolumn{2}{c||}{AIPW $\dagger$} &    \multicolumn{2}{c|}{DM $\dagger *$} &    \multicolumn{2}{c|}{EIPW $\dagger *$} \\
\hline
$\alpha$ &      RMSE &      SD &      RMSE &      SD &      RMSE &      SD &      RMSE &      SD &      RMSE &      SD \\
\hline
0.7 &  0.038 &  0.001 &  \textcolor{red}{\textbf{0.005}} &  0.0 &  0.015 &  0.000 &  \underline{0.012} &  0.000 &  0.114 &  0.004 \\
0.4 &  0.029 &  0.001 &  \textcolor{red}{\textbf{0.011}} &  0.0 &  0.019 &  0.000 &  \underline{0.016} &  0.000 &  0.159 &  0.011 \\
0.1 &  \underline{0.021} &  0.001 &  \textcolor{red}{\textbf{0.014}} &  0.0 &  0.048 &  0.002 &  0.044 &  0.001 &  0.092 &  0.007 \\
\hline
\end{tabular}
}
\end{center}
\vspace{-0.3cm}
\begin{center}
\scalebox{0.80}[0.80]{
\begin{tabular}{|l||rr||rr|rr||rr|rr|}
\hline
{\tt sensorless} &    \multicolumn{2}{|c||}{ADR\ $\dagger*$} &   \multicolumn{2}{c|}{IPW $\dagger$} &   \multicolumn{2}{c||}{AIPW $\dagger$} &    \multicolumn{2}{c|}{DM $\dagger *$} &    \multicolumn{2}{c|}{EIPW $\dagger *$} \\
\hline
$\alpha$ &      RMSE &      SD &      RMSE &      SD &      RMSE &      SD &      RMSE &      SD &      RMSE &      SD \\
\hline
0.7 &  0.096 &  0.007 &  \textcolor{red}{\textbf{0.007}} &  0.000 &  0.012 &  0.000 &  \underline{0.008} &  0.000 &  0.481 &  0.083 \\
0.4 &  0.097 &  0.007 &  \textcolor{red}{\textbf{0.014}} &  0.000 &  0.027 &  0.001 &  \underline{0.025} &  0.000 &  0.407 &  0.058 \\
0.1 &  \underline{0.054} &  0.002 &  \textcolor{red}{\textbf{0.023}} &  0.001 &  0.038 &  0.002 &  0.065 &  0.002 &  0.208 &  0.025 \\
\hline
\end{tabular}
}
\end{center}
\vspace{-0.3cm}
\begin{center}
\scalebox{0.80}[0.80]{
\begin{tabular}{|l||rr||rr|rr||rr|rr|}
\hline
{\tt connect-4}$\ \ $ &    \multicolumn{2}{|c||}{ADR\ $\dagger*$} &   \multicolumn{2}{c|}{IPW $\dagger$} &   \multicolumn{2}{c||}{AIPW $\dagger$} &    \multicolumn{2}{c|}{DM $\dagger *$} &    \multicolumn{2}{c|}{EIPW $\dagger *$} \\
\hline
$\alpha$ &      RMSE &      SD &      RMSE &      SD &      RMSE &      SD &      RMSE &      SD &      RMSE &      SD \\
\hline
0.7 &  0.060 &  0.002 &  \textcolor{red}{\textbf{0.005}} &  0.0 &  0.012 &  0.000 &  \underline{0.009} &  0.000 &  0.378 &  0.013 \\
0.4 &  0.043 &  0.002 &  \textcolor{red}{\textbf{0.010}} &  0.0 &  0.014 &  0.000 & \underline{0.023} &  0.000 &  0.211 &  0.012 \\
0.1 &  \underline{\textcolor{red}{\textbf{0.017}}} &  0.000 &  \textcolor{red}{\textbf{0.017}} &  0.0 &  0.037 &  0.002 &  0.035 &  0.001 &  0.057 &  0.003 \\
\hline
\end{tabular}
}
\end{center}
\vspace{-0.65cm}
\end{table*}

\end{document}